\documentclass[letterpaper]{article} 
\usepackage{aaai2026}  
\usepackage{times}  
\usepackage{helvet}  
\usepackage{courier}  
\usepackage[hyphens]{url}  
\usepackage{graphicx} 
\usepackage{natbib}  
\usepackage{caption} 
\usepackage{algorithm}
\usepackage{algorithmic}

\usepackage{newfloat}
\usepackage{listings}
\DeclareCaptionStyle{ruled}{labelfont=normalfont,labelsep=colon,strut=off} 
\floatstyle{ruled}
\newfloat{listing}{tb}{lst}{}
\floatname{listing}{Listing}
\usepackage{algorithm}
\usepackage{algorithmic}
\usepackage[utf8]{inputenc} 
\usepackage[T1]{fontenc}    
\usepackage{hyperref}       
\usepackage{url}            
\usepackage{booktabs}       
\usepackage{amsfonts}       
\usepackage{amsmath,mathtools,amsthm}	
\usepackage{amssymb,dsfont}	
\usepackage{nicefrac}       
\usepackage{microtype}      
\usepackage{xcolor}         
\usepackage{tikz-cd}        

\usepackage{xcolor}
\usepackage{srcltx}
\usepackage{etoolbox}
\usepackage{nameref}
\usepackage{comment}
\usepackage{mathrsfs}
\usepackage{enumerate}
\usepackage[shortlabels]{enumitem}
\usepackage{amsfonts} 
\usepackage{nicefrac} 
\usepackage{mathtools}
\usepackage{dsfont,mathrsfs}
\usepackage{cleveref}
\usepackage{xargs}
\usepackage{multirow}
\usepackage{bm}
\usepackage{subfigure}

\usepackage{autonum}
\usepackage{upgreek}
\usepackage{appendix}

\title{Gaussian Approximation for Two-Timescale Linear Stochastic Approximation}
\author{
    Bogdan Butyrin\textsuperscript{\rm 1, *}, 
    Artemy Rubtsov\textsuperscript{\rm 1, *}, 
    Alexey Naumov\textsuperscript{\rm 1},
    Vladimir Ulyanov\textsuperscript{\rm 1, 2},
    Sergey Samsonov\textsuperscript{\rm 1}
}
\affiliations{
    \textsuperscript{\rm 1} HSE University,  \textsuperscript{\rm 2} Lomonosov Moscow State University \\
    \textsuperscript{*}Equal contribution: {bbutyrin@hse.ru}, {asrubtsov@hse.ru} \,
}

\newtheorem{assum}{A\hspace{-2pt}}

\newtheorem{theorem}{Theorem}
\crefname{theorem}{theorem}{Theorems}
\Crefname{theorem}{Theorem}{Theorems}

\newtheorem{lemma}{Lemma}
\crefname{lemma}{lemma}{lemmas}
\Crefname{lemma}{Lemma}{Lemmas}

\newtheorem{remark}{Remark}
\crefname{remark}{remark}{remarks}
\Crefname{remark}{Remark}{Remarks}

\newtheorem{corollary}{Corollary}
\crefname{corollary}{corollary}{corollaries}
\Crefname{corollary}{Corollary}{Corollaries}

\newtheorem{proposition}{Proposition}
\crefname{proposition}{proposition}{propositions}
\Crefname{proposition}{Proposition}{Propositions}

\crefname{definition}{definition}{definitions}
\Crefname{Definition}{Definition}{Definitions}

\crefname{example}{example}{examples}
\Crefname{Example}{Example}{Examples}

\crefname{figure}{figure}{figures}
\Crefname{Figure}{Figure}{Figures}

\crefname{table}{table}{tables}
\Crefname{Table}{Table}{Tables}

\crefname{assum}{A\hspace{-2pt}}{A\hspace{-2pt}}
\crefname{assumb}{B\hspace{-2pt}}{B\hspace{-2pt}}
\crefname{assumUGE}{UGE\hspace{-1pt}}{UGE\hspace{-1pt}}
\crefname{assumID}{IND\hspace{-1pt}}{IND\hspace{-1pt}}
\crefname{assumUE}{UE\hspace{-1pt}}{UE\hspace{-1pt}}
\crefname{assumSUP}{M\hspace{-1pt}}{M\hspace{-1pt}}

\newlist{renumerate}{enumerate}{3}
\setlist[renumerate]{wide, labelwidth=!, labelindent=0pt,label=(\roman*)}

\newlist{aenumerate}{enumerate}{3}
\setlist[aenumerate]{wide, labelwidth=!, labelindent=0pt,label=(\arabic*)}

\newlist{aaenumerate}{enumerate}{3}
\setlist[aaenumerate]{wide, labelwidth=!, labelindent=0pt,label=(\alph*)}

\newlist{aenumerateSpace}{enumerate}{3}
\setlist[aenumerateSpace]{wide, labelwidth=!,label=(\arabic*)}

\newlist{benumerate}{enumerate}{3}
\setlist[benumerate]{wide, labelwidth=!, labelindent=0pt,label=$\bullet$}

\newcommand{\PE}{\mathbb{E}}
\newcommand{\var}{\operatorname{Var}}
\newcommand{\PP}{\mathbb{P}}
\newcommandx{\genericb}[1][1=]{b_{#1}}
\newcommandx{\Constros}[1][1=]{\operatorname{C}_{\operatorname{Ros},#1}}
\newcommandx{\Constburk}[1][1=]{\operatorname{C}_{\operatorname{Burk}}}
\newcommandx{\driftW}[1][1=]{W_{#1}}

\newcommandx{\metricd}[1][1=]{\mathsf{d}_{#1}}

\newcommandx\invmeasure[1][1=]{\Pi_{#1}}
\newcommandx{\PPjoint}[1][1=]{\PP^{\MKjoint[#1]}}
\newcommandx{\PEjoint}[1][1=]{\PE^{\MKjoint[#1]}}
\newcommandx{\PEMID}[1][1=\alpha]{\PE^{\MK[#1]}}
\newcommandx{\PPMID}[1][1=\alpha]{\PP^{\MK[#1]}}

\newcommandx{\MKjoint}[1][1=]{\bar{\operatorname{P}}_{#1}}
\newcommandx\costw[1][1=]{\mathsf{c}_{#1}}

\newcommandx\Intergrdist[1][1=]{\mathbb{M}_{1}(#1)}

\newcommandx{\mmarkov}[1][1=0]{m^{(\Markov)}_{#1}}

\def\Conv{\operatorname{Conv}}

\def\H{\mathcal{H}}

\def\Xset{\mathsf{X}}
\def\Xsigma{\mathcal{X}}

\def\Zset{\mathsf{Z}}
\def\Zsigma{\mathcal{Z}}

\def\rset{\mathbb{R}}

\def\nset{\ensuremath{\mathbb{N}}}
\def\nsets{\ensuremath{\mathbb{N}^*}}

\renewcommand{\S}{\mathcal{S}}
\newcommand{\A}{\mathcal{A}}

\def\PMDP{\MKQ}
\def\wstar{w^{\star}}
\def\what{\hat{w}}

\newcommand{\bConst}[1]{\operatorname{C}_{{\bf #1}}}

\newcommandx\sequence[4][2=,3=,4=]
{\ifthenelse{\equal{#3}{}}{\ensuremath{\{ #1_{#2 #4}\}}}{\ensuremath{\{ #1_{#2 #4} \}_{#2 \in #3}}}}

\newcommandx\sequenceD[2][2=]
{\ifthenelse{\equal{#2}{}}{\ensuremath{\{ #1\}}}{\ensuremath{\{ #1\!~:\!~#2  \}}}}
\newcommandx\sequenceDouble[4][3=,4=]
{\ifthenelse{\equal{#3}{}}{\ensuremath{\{ (#1_{#3},#2_{#3})\}}}{\ensuremath{\{ (#1_{#3},#2_{#3})\}_{#3 \in #4}}}}

\newcommandx{\sequencen}[2][2=n\in\nset]{\ensuremath{\{ #1, \eqsp #2 \}}}
\newcommandx\sequencens[2][2=n]
{\ensuremath{\{ #1_{#2} \!~:\!~#2\in\nsets\}}}
\newcommandx\sequencet[4]
{\ensuremath{\{ #1{#2_{#3}} \, : \, \eqsp #3 \in #4 \}}}
\def\PE{\mathbb{E}}
\def\P{\mathbb{P}}
\def\ProdB{\Gamma}

\newcommandx{\PVar}[1][1=]{\ensuremath{\operatorname{Var}_{#1}}}
\newcommandx\conststab[1][1=p]{\varkappa_{#1}}

\def\Mart{\operatorname{M}}

\def\tvdist{\mathsf{d}_{\operatorname{tv}}}
\newcommandx{\MK}[1][1=\alpha]{\mathrm{P}_{#1}}
\newcommandx\MKK[1][1=\alpha]{\mathrm{K}_{#1}}
\def\MKQ{\mathrm{P}}

\def\F{\mathcal{F}}

\newcommandx{\PEtilde}[1][1=]{\PE^{\mathrm{K}_{#1}}}
\newcommandx{\PPtilde}[1][1=]{\PP^{\mathrm{K}_{#1}}}

\def\Sigmabf{\boldsymbol{\Sigma}}
\def\lineG{\Sigmabf}

\newcommandx{\norm}[2][2=]{\Vert#1 \Vert_{{#2}}}
\newcommandx{\normLigne}[2][2=]{\Vert#1 \Vert_{{#2}}}
\newcommandx{\normLine}[2][2=]{\Vert#1 \Vert_{{#2}}}
\newcommandx{\normop}[2][2=]{\Vert{#1}\Vert_{{#2}}}
\newcommandx{\normopLigne}[2][2=]{\Vert{#1}\Vert_{{#2}}}
\newcommandx{\normopLine}[2][2=]{\Vert{#1}\Vert_{{#2}}}
\newcommandx{\osc}[2][1=]{\mathrm{osc}_{#1}(#2)}

\newcommand{\pois}[2]{\boldsymbol{g}^{#1}_{#2}}
\newcommand{\poisA}[2]{\boldsymbol{g}^{\funcAw_{#1}}_{#2}}
\newcommand{\poiseps}[2]{\boldsymbol{g}^{\boldsymbol{\eps}_{#1}}_{#2}}

\newcommandx{\cmark}[2]{{#1}^{\text{mark}}_{#2}}

\def\dth{{d_{\theta}}}
\def\dw{{d_w}}
\def\pth{{\myqcond_\Delta}}
\def\pw{{\myqcond_{22}}}

\def\CovV{\Sigma_{V}}
\def\CovW{\Sigma_{W}}

\def\Covfast{\Sigma_{\theta}}
\def\Covslow{\Sigma_{w}}

\def\CovVW{\Sigma_{VW}}

\newcommandx{\normlip}[2][2=\operatorname{Lip}]{\Vert#1 \Vert_{{#2}}}
\newcommand{\lip}{\operatorname{L}}
\newcommandx{\lipspace}[1]{\lip_{#1}}

\newcommandx{\CPP}[3][1=]
{\ifthenelse{\equal{#1}{}}{{\mathbb P}\left(\left. #2 \, \right| #3 \right)}{{\mathbb P}_{#1}\left(\left. #2 \, \right | #3 \right)}}
\newcommandx{\CPPtilde}[3][1=]
{\ifthenelse{\equal{#1}{}}{{\tilde{\mathbb P}}\left(\left. #2 \, \right| #3 \right)}{{\tilde{\mathbb P}}_{#1}\left(\left. #2 \, \right | #3 \right)}}

\def\rhs{right-hand side}

\def\iid{i.i.d.}

\newcommandx{\as}[1][1=\PP]{\ensuremath{#1\, -\mathrm{a.s.}}}

\newcommand{\eqsp}{\;}

\newcommand{\Id}{\mathrm{I}}

\def\ttheta{\tilde{\theta}}
\def\tw{\tilde{w}}

\newcommand{\ConstC}{\operatorname{C}}

\newcommandx{\boundmetric}[1][1=]{\kappa_{\MKK[#1]}}

\newcommandx{\Nnorm}[2][1=V]{[ #2]_{#1}}
\newcommandx{\lipnorm}[2][1=g]{[ #1]_{#2}}

\newcommandx{\CPE}[3][1=]{{\mathbb E}^{#3}_{#1}\left[#2\right]}
\newcommandx{\CPEext}[3][1=]{\tilde{\mathbb E}^{#3}_{#1}\left[#2\right]}
\newcommandx{\CPEtilde}[3][1=]{{\tilde{\mathbb E}}^{#3}_{#1}\left[#2\right]}
\newcommandx{\CPEs}[3][1=]{{\mathbb E}^{#3}_{#1}[#2]}

\def\thetalim{\theta^\star}

\def\trace{\operatorname{Tr}}

\newcommand{\rme}{\mathrm{e}}
\newcommand{\rmd}{\mathrm{d}}
\def\funcAw{\mathbf{A}}

\def\funcAwtilde{\tilde{\mathbf{A}}}

\def\funcbw{\mathbf{b}}

\newcommandx{\zmfuncA}[2][1=]{\tilde{\funcAw}^{#1}(#2)}
\newcommandx{\zmfuncAw}[1][1=]{\tilde{\funcAw}_{#1}}
\newcommandx{\zmfuncb}[2][1=]{\tilde{\funcbw}^{#1}(#2)}
\def\funnoisew{\boldsymbol{\varepsilon}}

\newcommandx{\funcct}[2][1=]{\funcctilde^{#1}(#2)}
\def\rstep{r_{\operatorname{step}}}

\def\myqcond{\kappa}
\def\State{Z}
\def\state{z}

\def\taumix{t_{\operatorname{mix}}}

\newcommand{\1}{\boldsymbol{1}}

\newcommandx{\CovC}[1][1=u]{\operatorname{C}_{#1}}

 \usepackage[textwidth=2.25cm,textsize=footnotesize]{todonotes}

\newcommand{\som}[1]{\todo[color=green!20]{{\bf SS:} #1}}

\newcommand\borel[1]{\mathcal{B}(#1)}

\DeclareMathAlphabet{\mathpzc}{OT1}{pzc}{m}{it}

\def\lyapW{\mathpzc{W}}

\newcommandx{\noiset}[1]{\eps^{(\theta)}_{#1}}
\newcommandx{\noisew}[1]{\eps^{(w)}_{#1}}

\newcommandx{\hnoiset}[1]{\hat{\eps}^{(\theta)}_{#1}}
\newcommandx{\hnoisew}[1]{\hat{\eps}^{(w)}_{#1}}

\newcommandx{\covnoiset}[1]{\tilde{\eps}^{(\theta)}_{#1}}
\newcommandx{\covnoisew}[1]{\tilde{\eps}^{(w)}_{#1}}

\newcommandx{\bias}[1][1=\alpha]{\operatorname{B}_{#1}}

\newcommandx\probaMarkovTilde[2][2=]
{\ifthenelse{\equal{#2}{}}{{\widetilde{\mathbb{P}}_{#1}}}{\widetilde{\mathbb{P}}_{#1}\left[ #2\right]}}

\newcommand{\expep}[2]{\PE^{1/#2} \bigl[ \norm{{#1}}^{#2} \bigr]}

\newcommand{\indi}[1]{\1_{#1}}

\def\thetas{\thetalim}

\def\funcctilde{\tilde{c}_u}

\newcommandx{\driftb}[1][1=p]{\bar{b}_{#1}}

\def\eps{\varepsilon}

\newcommandx{\boldb}[1][1={q}]{\mathsf{b}_{#1}}

\newcommandx{\ConstGW}[1][1={n,\lyapW}]{\operatorname{G}_{#1}}

\newcommandx{\ConstMW}[1][1={n,\lyapW}]{\operatorname{M}_{#1}}

\newtheorem{assumTD}{\textbf{TD}\hspace{-1pt}}
\Crefname{assumTD}{\textbf{TD}\hspace{-1pt}}{\textbf{TD}\hspace{-1pt}}
\crefname{assumTD}{\textbf{TD}}{\textbf{TD}}

\Crefname{assumprime}{\textbf{A'}\hspace{-1pt}}{\textbf{A'}\hspace{-1pt}}
\crefname{assumprime}{\textbf{A'}}{\textbf{A'}}

\Crefname{assumBprime}{B'\hspace{-1pt}}{\textbf{A'}\hspace{-1pt}}
\crefname{assumBprime}{B'}{\textbf{B'}}

\newtheorem{assumB}{\textbf{B}\hspace{-1pt}}
\Crefname{assumB}{B\hspace{-1pt}}{\textbf{B}\hspace{-1pt}}
\crefname{assumB}{B}{\textbf{B}}

\Crefname{assumptionC}{\textbf{C}\hspace{-1pt}}{\textbf{C}\hspace{-1pt}}
\crefname{assumptionC}{\textbf{C}}{\textbf{C}}

\Crefname{assumptionM}{\textbf{UGE}\hspace{-1pt}}{\textbf{UGE}\hspace{-1pt}}
\crefname{assumptionM}{\textbf{UGE}}{\textbf{UGE}}

\def\distance{\mathsf{d}}

\newcommandx{\vartconstwas}[1][1=V]{c_{#1}}

\newcommandx{\deltawas}[1][1=*]{\delta_{#1}}

\newcommandx{\wasser}[4][1=\distance,4=]{\mathbf{W}_{#1}^{#4}\left(#2,#3\right)}
\newcommandx{\covcoeff}[2]{\rho_{#1}^{(#2)}}

\newcommand{\dobrush}{\mathsf{\Delta}}
\newcommandx{\dobru}[3][1=,3=]{\dobrush_{#1}^{#3}( #2)}  

\def\invariantQ{\pi}

\def\Markov{\mathrm{M}}

\def\btheta{\bar{\theta}}

\newcommandx{\dlim}[1]{\ensuremath{\stackrel{#1}{\Longrightarrow}}}

\def\kolmogorov{\rho^{\Conv}}

\def\boundconstmart{\varkappa}

\def\mart{\mathcal{M}}

\newcommand{\dkolmogorov}[1]{\mathsf{d}_{K}\bigl(#1\bigr)}

\begin{document}

\maketitle

\begin{abstract}
In this paper, we establish non-asymptotic bounds for accuracy of normal approximation for linear two-timescale stochastic approximation (TTSA) algorithms driven by martingale difference or Markov noise. Focusing on both the last iterate and Polyak–Ruppert averaging regimes, we derive bounds for normal approximation in terms of the convex distance between probability distributions. Our analysis reveals a non-trivial interaction between the fast and slow timescales: the normal approximation rate for the last iterate improves as the timescale separation increases, while it decreases in the Polyak–Ruppert averaged setting. We also provide the high-order moment bounds for the error of linear TTSA algorithm, which may be of independent interest. Finally, we demonstrate that our theoretical results are directly applicable to reinforcement learning algorithms such as GTD and TDC.
\end{abstract}


%

\section{Introduction}
\label{sec:intro}

Stochastic approximation (SA) methods play an important role in the field of machine learning, especially due to  their role in solving reinforcement learning (RL) problems \cite{sutton:book:2018}. Recent studies cover both asymptotic \cite{nemirovskij1983problem,polyak1992acceleration} and non-asymptotic \cite{moulines2011non} properties of SA estimates. In particular, two-timescale stochastic approximation (TTSA) algorithms \cite{borkar1997stochastic} refer to the class of methods that update two interdependent variables with separate step size sequences, one typically decreasing faster than the other. This class of methods is especially important in RL, where policy evaluation in the off-policy setting requires TTSA methods such as the Gradient Temporal Difference (GTD) method \cite{Sutton_2008}.
\par 
An important question for SA algorithms is related to the accuracy of Gaussian approximation (GAR) of the constructed estimates. Classical results on GAR for SA algorithms, such as \cite{polyak1992acceleration, konda:tsitsiklis:2004}, are asymptotic and do not provide convergence rates. At the same time, the latter results play an important role in statistical inference for optimization \cite{fan2019exact}, as they pave the way for non-asymptotic analysis of various procedures for constructing confidence intervals. We focus on the linear two-timescale SA problem, that is, we aim to find a solution $(\theta^\star, w^\star)$ that solves the system of linear equations:
\begin{align} 
\label{eq:linear_sys}
A_{11} \theta + A_{12} w = b_1\eqsp, \quad A_{21} \theta + A_{22} w = b_2\eqsp,
\end{align}
assuming that the solution $(\theta^\star, w^\star)$ is unique and is given by
\begin{equation} 
\label{eq:opt_sol}
\theta^\star = \Delta^{-1} (b_1 - A_{12} A_{22}^{-1} b_2),\quad w^\star = A_{22}^{-1} (b_2 - A_{21} \theta^\star)\eqsp,
\end{equation}
with $\Delta := A_{11} - A_{12} A_{22}^{-1} A_{21}$. We consider the setting, where the underlying matrices $A_{ij}$ and vectors $b_i$, $i,j \in \{1,2\}$, are not accessible. Instead, following \cite{borkar1997stochastic}, we assume that the learner has access to a sequence of random variables $\{X_k\}_{k \in \nset}$ taking values in a measurable space $(\Xset,\Xsigma)$, and vector/matrix-valued functions $\funcbw_i(x)$, $\funcAw_{ij}(x)$, $i,j \in \{1,2\}$, which serves as stochastic estimates of $b_i$ and $A_{ij}$, respectively. The corresponding recurrence runs as 
\begin{equation}
\label{eq:tts-gen1}
\begin{split}
\theta_{k+1} &= \theta_k + \beta_k \{ \funcbw^{k+1}_{1} - \funcAw^{k+1}_{11} \theta_k - \funcAw^{k+1}_{12} w_k \}\eqsp, \\
w_{k+1} &= w_k + \gamma_k \{ \funcbw^{k+1}_{2} - \funcAw^{k+1}_{21} \theta_k - \funcAw^{k+1}_{22}w_k \}\eqsp,
\end{split}
\end{equation}
where $\theta_k \in \rset^{\dth}$, $w_k \in \rset^\dw$, and $\funcbw_{i}^{k}$, $\funcAw_{ij}^{k}$ are shorthand notations for $\funcbw_{i}(X_{k})$ and $\funcAw_{ij}(X_k)$, respectively. The scalars $\gamma_k, \beta_k > 0$ in \eqref{eq:tts-gen1} are step sizes, and the underlying SA scheme is said to have two timescales as the step sizes satisfy $\lim_{k \rightarrow \infty} \beta_k / \gamma_k < 1$ such that $w_k$ is updated at a faster timescale. In our paper we consider $\beta_k = c_{0,\beta} (k+k_0)^{-b}$ and $\gamma_k = c_{0,\gamma} (k+k_0)^{-a}$ with exponents $a$ and $b$ satisfying $1/2 < a < b < 1$. When $\{X_k\}_{k \in \nset}$ are i.i.d., and under appropriate technical assumptions on the parameters of \eqref{eq:tts-gen1}, it is known (see e.g. \cite{konda:tsitsiklis:2004}), that the asymptotic normality of the "slow" timescale $\theta_k$ holds:
\begin{equation}
\label{eq:last_iterate_clt}
\beta_k^{-1/2}(\theta_k - \thetas) \to \mathcal{N}(0,\Covfast)\eqsp,
\end{equation}
with some covariance $\Covfast$. The authors in \cite{mokkadem2006convergence} generalized this result for the averaged iterates of non-linear SA:
\begin{equation}
\label{eq:pr_averaged_iterates}
\textstyle
\bar{\theta}_n := n^{-1}\sum_{k=1}^{n}\theta_k\eqsp, \quad \bar{w}_n := n^{-1}\sum_{k=1}^{n}w_k\eqsp.
\end{equation}
The latter estimates correspond to the Polyak-Ruppert averaging procedure introduced in \cite{ruppert1988efficient,polyak1992acceleration}, a popular technique for stabilization of the SA algorithms. The authors of the recent paper \cite{srikant20252ts} obtained the non-asymptotic convergence rates for the averaged iterates $\bar{\theta}_n$ and $\bar{w}_n$ in Wasserstein distance of order $1$, using the vector-valued versions of the Berry-Essen theorem for martingale-difference sequences due to \cite{srikant2024rates}. In this paper, we not only generalize these results for the setting of Markov noise, but also establish the corresponding convergence rates for the last iterate $\theta_k$. The main contributions of this paper are the following:
\begin{itemize}[noitemsep,topsep=0pt]
    \item We derive non-asymptotic bounds for the accuracy of normal approximation for the Polyak–Ruppert-averaged TTSA $\sqrt{n}(\bar{\theta}_n - \thetas)$ and last iterate $\beta_n^{-1/2}(\theta_n - \thetas)$ in terms of convex distance under martingale-difference noise assumptions. Our results indicate that the normal approximation for the last iterate improves as the timescale separation increases and achieves a convergence rate of order up to $n^{-1/4}$, up to $\log{n}$ factors. We show that the Polyak–Ruppert averaged TTSA iterates achieve the same rate of normal approximation, but require that the timescales $\beta_k$ and $\gamma_k$ coincide up to a constant factor. While our analysis for the Polyak–Ruppert averaged TTSA generalizes recent results due to \cite{srikant20252ts}, we provide, to the best of our knowledge, the first fully non-asymptotic analysis of the normal approximation rates for the last iterate of TTSA. 
    \item We generalize the obtained results for normal approximation for the averaged TTSA and the last iterate to the setting of Markov noise. Our results show a convergence rate of order up to $n^{-1/6}$, up to logarithmic factors, with the same conclusion regarding timescale separation as in the martingale noise case. This is the first result on the normal approximation rate for TTSA with Markov noise. 
\end{itemize}
\textbf{Notations.} For a matrix $A \in \rset^{d \times d}$ we denote by $\norm{A}$ its operator norm. For symmetric positive-definite matrix $Q = Q^\top \succ 0\eqsp, \eqsp Q \in \rset^{d \times d}$ and $x \in \rset^{d}$ we define the corresponding norm $\|x\|_Q = \sqrt{x^\top Q x}$, and define the respective matrix $Q$-norm of the matrix $B \in \rset^{d \times d}$ by $\normop{B}[Q] = \sup_{x \neq 0} \norm{Bx}[Q]/\norm{x}[Q]$. For sequences $a_n$ and $b_n$, we write $a_n \lesssim_{\log_n} b_n$ if there exist $c, \alpha > 0$ (not depending upon $n$), such that $a_n \leq c (1+\log n)^{\alpha} b_n$. In the present text, the following abbreviations are used: "w.r.t." stands for "with respect to", "\iid\ " - for "independent and identically distributed", "GAR" - for "Gaussian Approximation".

\paragraph{Related works}
Classical results in the stochastic approximation \cite{borkar:sa:2008} study the asymptotic properties of the single timescale SA algorithms, with the properties of averaged estimated studied in \cite{polyak1992acceleration}. Two-timescale SA schemes were studied in \cite{borkar1997stochastic,tadic2004as,tadic:asantdlearn:conststep:2006} in terms of almost sure convergence. Asymptotic convergence rates of linear two-timescale SA were studied in \cite{konda:tsitsiklis:2004}, where the authors showed that asymptotically $\PE[\norm{\theta_k - \thetas}^2] = \mathcal{O}(\beta_k)$ and $\PE[\norm{w_k - \wstar}^2] = \mathcal{O}(\gamma_k)$. 
\par 
Non-asymptotic error bounds for TTSA were first developed in \cite{dalal2018finite,dalal2020tale} under the martingale noise assumptions and additional projections used in the update scheme \eqref{eq:tts-gen1}. These results were further improved in \cite{kaledin2020finite} for linear TTSA problems. \cite{haque2023tight} refined the results of \cite{kaledin2020finite} obtaining the MSE bounds $\PE[\norm{\theta_k - \thetas}^2]$ and $\PE[\norm{w_k - \wstar}^2]$ with the leading terms given by $\beta_k \trace{\Covfast}$ and $\gamma_k \trace{\Covslow}$, where the covariances $\Covfast$ and $\Covslow$ aligns with the  CLT in \eqref{eq:last_iterate_clt}. \cite{kwon2024two} considered the version of \eqref{eq:tts-gen1} with constant step sizes and studied convergence to equilibrium for the corresponding Markov chain. Non-linear TTSA has been considered in \cite{doan2024fast} under strong monotonicity assumptions, focusing on obtaining the MSE rate of order $\mathcal{O}(1/k)$ for $k$-th iterate. 
\par 
Central limit theorem for TTSA iterates has been established in \cite{mokkadem2006convergence}, where the asymptotic version of the CLT was proved both for the last iterates $(\theta_k, w_k)$ and their Polyak--Ruppert averaged counterparts $(\bar{\theta}_n, \bar{w}_n)$. \cite{hu2024clt} established an asymptotic CLT for general TTSA under Markov noise and controlled Markov chain dynamics, without quantifying the convergence rate. \cite{srikant20252ts} studied the CLT for averaged iterates $(\bar{\theta}_n, \bar{w}_n)$ and provided a non-asymptotic CLT with the convergence rate studied in terms of Wasserstein distance of order $1$.

\section{Gaussian Approximation for SA algorithms}
We outline a general scheme for proving the normal approximation. We consider vector-valued nonlinear statistics $T(X_1,\ldots,X_n) \in \rset^{d}$, which can be represented in the form
\begin{equation}
\label{eq:linear_non_linear_t_decomposition}
\textstyle
T = W + D\eqsp,
\end{equation}
where $W$ is a linear statistic of the random variables $X_1,\ldots,X_n$, and $D$ is a small perturbation. This approach is well studied when $X_1,\ldots,X_n$ are i.i.d. random variables \cite{chen2007,shao2022berry} or form a martingale-difference sequence \cite{shorack2017probability}. The case of Markov random variables can be reduced to the setting of martingale-difference sequences through the Poisson equation \cite[Chapter~21]{douc:moulines:priouret:soulier:2018}. We consider the decomposition \eqref{eq:linear_non_linear_t_decomposition} and assume, without loss of generality, that $\PE[W W^\top] = \Id_{d}$. To measure the approximation quality, a common approach is to use the supremum of the difference between measures taken over some subclass $\H \subseteq \Conv(\rset^{d})$ of the collection of convex sets $\Conv(\rset^{d})$. Specifically, for probability measures $\mu, \nu$ on $\rset^{d}$, we write
\begin{equation}
\label{eq:convex_distance_def}
\textstyle 
\metricd[\H](\mu, \nu) = \sup_{B \in \H}\left|\mu(B) - \nu(B)\right|\eqsp.
\end{equation}
Examples of $\H$ include the class of all convex sets, half-spaces, rectangles, ellipsoids, etc. The choice of different collections of sets $\H$ may be motivated by the needs of a particular application and may introduce differences in the dependence of the results on the problem dimension $d$. Indeed, even this dimensional dependence for linear statistics $W$ can vary; see \cite{bentkus2004} and \cite{kojevnikov2022berry} for the respective results for i.i.d.\ sequences and martingale differences. In this paper, we focus on the convex distance $\kolmogorov$, defined as  
\[
\textstyle 
\kolmogorov(\mu, \nu) = \sup_{B \in \Conv(\rset^{d})} |\mu(B) - \nu(B)| \eqsp,
\]  
and rely on the following proposition to reduce the problem of Gaussian approximation for the nonlinear statistic $W+D$ to that for the linear statistic $W$:
\begin{proposition}[Proposition 2 in \cite{sheshukova2025gaussian}]
\label{prop:nonlinearapprox}
Let \(\nu\) be a standard Gaussian measure in \(\mathbb{R}^d\). Then for any random vectors \(W, D\) taking values in \(\mathbb{R}^d\), and any \(p \geq 1\),
\begin{multline}
\textstyle
\kolmogorov(W+D, \nu) \leq \kolmogorov(W, \nu) \\
\textstyle
+ 2c_d^{p/(p+1)} \mathbb{E}^{1/(p+1)}\left[\norm{D}^p\right],
\end{multline}
where $c_d$ is the isoperimetric constant of class $\Conv(\rset^{d})$.
\end{proposition}
Similar results can be derived for other classes of sets $\H$, with the constant $c_d$ depending on the isoperimetric properties of the specific class $\H$; see, e.g., \cite{10.1109/FOCS.2008.64}. \Cref{prop:nonlinearapprox} shows that the estimation of $\kolmogorov(W+D, \mathcal{N}(0,\Id))$ can be reduced to:
\begin{enumerate}
\item Estimating $\kolmogorov(W,\mathcal N(0, \Id))$;
\item Estimating moments $\PE[\norm{D}^p]$ for some $p \geq 1$.
\end{enumerate}
To bound $\kolmogorov(W,\mathcal N(0, \Id))$, one can apply a Berry–Esseen bound for the appropriate linear statistic, e.g., \cite{shao2022berry} for i.i.d. random variables or \cite{srikant2024rates,samsonov2025statistical,wu2025uncertainty} for the martingale-difference setting. The most involved part of the proof is the proper estimation of $\PE[\norm{D}^p]$.

\section{GAR for TTSA with Martingale noise}
\label{sec:TTSA_martingale}

\paragraph{Assumptions and definitions.} We investigate the linear TTSA algorithm given by the equivalent form of \eqref{eq:tts-gen1}:
\begin{align}
    &\theta_{k+1} = \theta_k + \beta_k (b_1 - A_{11}\theta_k - A_{12}w_k + V_{k+1}), \label{eq:2ts1} \\
    &w_{k+1} = w_k + \gamma_k (b_2 - A_{21}\theta_k - A_{22} w_k + W_{k+1})\eqsp. \label{eq:2ts2}
\end{align}
In this recurrence, the noise terms $V_{k+1}, W_{k+1}$ are given by:
\begin{equation}
\label{eq:noise_term_new}
\begin{split} 
V_{k+1} &=  \funnoisew_{V}^{k+1} - \funcAwtilde_{11}^{k+1} (\theta_k - \thetas) - \funcAwtilde_{12}^{k+1} (w_k - \wstar),  \\
W_{k+1} &= \funnoisew_{W}^{k+1} - \funcAwtilde_{21}^{k+1} (\theta_k - \thetas) - \funcAwtilde_{22}^{k+1} (w_k - \wstar),
\end{split}
\end{equation}
where we used the notation $\funcAwtilde_{ij}^{k+1} := \funcAw_{ij}^{k+1} - A_{ij}$ for $i,j \in \{1,2\}$, and the random vectors $\funnoisew_{V}^{k+1}, \funnoisew_{W}^{k+1}$ are given by 
\begin{equation}
\label{eq:noise_term_eps_V_W}
\begin{split}
\funnoisew_{V}^{k+1} &= \funcbw_1^{k+1} - \funcAw_{11}^{k+1} \thetas - \funcAw_{12}^{k+1}\wstar \eqsp, \\
\funnoisew_{W}^{k+1} &= \funcbw_2^{k+1} - \funcAw_{21}^{k+1} \thetas - \funcAw_{22}^{k+1} \wstar \eqsp.
\end{split}
\end{equation}
We consider a setting where the random elements $V_{k+1}$ and $W_{k+1}$ form a martingale-difference w.r.t. filtration $\F_k = \sigma(X_1,\ldots,X_k)$, $\F_0$ is trivial. We first consider the martingale noise setting. This setting covers the i.i.d. setting from \cite{konda:tsitsiklis:2004} and also serves as a basis for subsequent analysis of the Markov noise setting. 
\begin{assum}
\label{assum:zero-mean} 
The noise terms are zero-mean given $\F_k$, i.e., $\CPE{ V_{k+1} }{\F_k} = 0$, and $\CPE{W_{k+1}}{\F_k} = 0$.
\end{assum}
Next, for a given $p \geq 2$, we impose the following moment bound on $V_{k+1}$, $W_{k+1}$:
\begin{assum}[$p$]
\label{assum:bound-conditional-moments-p}
There exist constants $m_W, m_V > 0$ such that for any $k \in \nset$:
\end{assum}
\resizebox{0.98\linewidth}{!}{$
\begin{array}{l}
\PE^{1/p}[\|V_{k+1}\|^p] \leq m_V(1 + \PE^{1/p}[\|\theta_k - \theta^*\|^p] + \mathbb{E}^{1/p}[\|w_k - w^*\|^p]) \\
\PE^{1/p}[\|W_{k+1}\|^p] \leq m_W(1 + \PE^{1/p}[\|\theta_k - \theta^*\|^p] + \PE^{1/p}[\|w_k - w^*\|^p])
\end{array}
$}

The assumption \Cref{assum:bound-conditional-moments-p}($p$) appears in a similar form with $p=2$ in  \cite[Assumption A4]{kaledin2020finite}. Since our results require to control high-order moments of the TTSA iterates $\theta_k$ and $w_k$, it is natural to require that $p$-th moment of $V_{k+1}$ and $W_{k+1}$ are finite. Next, we present an assumption on the quadratic characteristic of $V_k$ and $W_k$:
\begin{assum}
\label{assum:quadratic_characteristic}
Noise variables $\funnoisew_{V}^{k+1}$ and $\funnoisew_{W}^{k+1}$ defined in \eqref{eq:noise_term_eps_V_W} have zero conditional expectation given $\F_k$, that is, $\CPE{\funnoisew_{V}^{k+1}}{\F_k} = 0$ and $\CPE{\funnoisew_{W}^{k+1}}{\F_k} = 0$. Moreover, there exist matrices $\CovV$, $\CovW$, $\CovVW$ such that for any $k > 0$:
\begin{align}
&\CPE{\funnoisew_{V}^{k+1} \{\funnoisew^{k+1}_{V}\}^{\top}}{\F_k} = \CovV \eqsp, \CPE{\funnoisew_{W}^{k+1} \{\funnoisew_{W}^{k+1}\}^{\top}}{\F_k} = \CovW\eqsp, \\
&\CPE{\funnoisew_{V}^{k+1} \{\funnoisew_{W}^{k+1}\}^{\top}}{\F_k} =\CovVW \eqsp. 
\end{align}
\end{assum}
This assumption relaxes the one stated in \cite{srikant20252ts}, where the authors required the quadratic characteristic of the entire vectors $V_{k+1}$ and $W_{k+1}$ to be constant. However, this assumption is unlikely to hold due to the structure of these vectors outlined in \eqref{eq:noise_term_new}. We also impose the following conditions on the problem matrices:
\begin{assum}
\label{assum:hurwitz}
Matrices $-A_{22}$ and $ - \Delta =-\left(A_{11}-A_{12}A_{22}^{-1}A_{21}\right)$ are \textit{Hurwitz}. 
\end{assum}
\Cref{assum:hurwitz} is common for the analysis of both the linear two-timescale SA, see \cite{konda:tsitsiklis:2004}, and single-timescale SA, see \cite{durmus2022finite,mou2020linear}. \Cref{assum:hurwitz} implies, due to the Lyapunov lemma (stated in the supplement paper for completeness), that there exist matrices $Q_{22}^\top = Q_{22} \succ 0$, $Q_\Delta^\top = Q_\Delta \succ 0$, such that  
\begin{equation}
\label{eq:contraction_p}
\begin{split}
&\normop{\Id - \gamma_k A_{22}}[Q_{22}] \leq 1 - a_{22} \gamma_k, \quad a_{22} := \tfrac{1}{4 \| Q_{22}\|}\eqsp, \\
&\normop{\Id - \beta_k \Delta}[Q_\Delta] \leq 1 - a_\Delta \beta_k, \quad  a_\Delta := \tfrac{1}{4 \| Q_\Delta \|}\eqsp, 
\end{split}
\end{equation}
provided that the step sizes $\gamma_k$ and $\beta_k$ are small enough. Precisely, for $p \geq 2$, we impose the following assumption \Cref{assum:stepsize}($p$) on the step sizes:
\begin{assum}[$p$]
\label{assum:stepsize}
Step sizes $(\gamma_k)_{k \geq 1}$, $(\beta_k)_{k \geq 1}$ are non-increasing sequences of the form 
\[
\beta_k = c_{0,\beta} (k+k_0)^{-b}, \quad \gamma_k = c_{0,\gamma} (k+k_0)^{-a}\eqsp, 
\]  
where \(1/2 < a < b < 1\), fraction $c_{0,\beta}/c_{0,\gamma}$ is small enough, and constant $k_0$ satisfies the bound $k_0 \geq C_{\Cref{assum:stepsize}} p^{4/b}$, where the constant $C_{\Cref{assum:stepsize}}$ does not depend upon $p$.
\end{assum} 
In the subsequent main results, we set the parameter $p$ of order \(\log(n)\). Hence, the parameter $k_0$ will depend on the total number of iterations to be performed. The same effect appears in the single-timescale SA algorithms \cite{durmus2022finite,wu2024statistical}. This effect is unavoidable at least in the setting of the constant step size algorithms, see \cite[Theorem~1]{durmus2021tight}. 
\begin{assum}
\label{assum:aij_bound}
    There exist constants $\bConst{A}, \bConst{b} > 0$ such that
    \begin{align}
        & \textstyle \sup_{x \in \Xset} \norm{\funcAw_{ij}(x)} \vee \norm{\funcAw_{ij}(x) - A_{ij}} \leq \bConst{A}\eqsp, \eqsp  \forall \eqsp i, j \in \{1, 2\} \eqsp, \\
        & \textstyle \sup_{x \in \Xset} \norm{\funcbw_{i}(x)} \vee \norm{\funcbw_{i}(x) - b_i}  \leq \bConst{b} \eqsp, \eqsp  \forall \eqsp i \in \{1, 2\} \eqsp.
    \end{align}
\end{assum}
We expect that \Cref{assum:aij_bound} can be replaced with an appropriate moment condition, at least in a setting where the noise variables $V_{k}$ and $W_{k}$ form a martingale difference. At the same time, our further generalizations to the Markov noise setting inherently rely on the boundedness of $\funcAw_{ij}(x)$ and $\funcbw_{i}(x)$.

\vspace{-2mm}
\subsection{Moment bounds for Martingale TTSA}
\label{sec:moments_martingale_TTSA}
Given the assumptions \Cref{assum:zero-mean} - \Cref{assum:aij_bound}, we present the classical reformulation of the two-timescale SA scheme \eqref{eq:2ts1}-\eqref{eq:2ts2}, which is due to \cite{konda:tsitsiklis:2004}, see also \cite{kaledin2020finite}. We define recursively the following sequence of matrices $\{L_k\}_{k \in \nset}$, with $L_0 = 0$, and 
\begin{equation}
\label{eq:L_k_matr_def}
\begin{split}
L_{k+1} &:= \big( L_k - \gamma_k A_{22} L_k + \beta_k A_{22}^{-1} A_{21} U_{k} \big) \\ 
& \times \bigl( \Id - \beta_k U_k \bigr)^{-1} \eqsp, \eqsp U_k := \Delta - A_{12}L_k\eqsp.
\end{split}
\end{equation}
and define $\lip_{\infty} = a_{\Delta}\lambda_{\sf max}(Q_{\Delta}) / (\lambda_{\sf min}(Q_{22})2\|A_{12}\|)$. As shown in \cite[Lemma~18]{kaledin2020finite}, under \Cref{assum:stepsize} above recursion on $L_k$ is well-defined, and every $L_k$ satisfies the relation $\norm{L_k} \leq \lip_{\infty}$. In addition, define the matrices:
\begin{align}
B_{11}^k &:= \Delta - A_{12} L_k\eqsp, \quad D_k := L_{k+1} + A_{22}^{-1} A_{21} \eqsp, \\
B_{22}^k &:= (\beta_k/\gamma_k) \big( L_{k+1} + A_{22}^{-1} A_{21} \big) A_{12} + A_{22}\eqsp.
\end{align}

In a similar vein as performing Gaussian elimination, we obtain a simplified two-timescale SA recursions:
\begin{proposition}[Observation 1 in \cite{kaledin2020finite}]
\label{prop:kaledine_decoupling}
Consider the following change of variables:
\begin{equation}
\label{eq:tilde_params}
\ttheta_k := \theta_k - \thetas,\quad \tw_k = w_k - \wstar + D_{k-1} \ttheta_k.
\end{equation}
Then the two-timescale SA \eqref{eq:2ts1}-\eqref{eq:2ts2} is equivalent to:
\begin{equation}
\label{eq:2ts2-1}
\begin{split}
\ttheta_{k+1} &= ( \Id - \beta_k B_{11}^k ) \ttheta_k - \beta_k A_{12} \Tilde{w}_k - \beta_k V_{k+1}\eqsp, \\
\tilde{w}_{k+1} &= (\Id - \gamma_k B_{22}^k ) \Tilde{w}_k - \beta_k D_k V_{k+1} - \gamma_k W_{k+1}\eqsp.
\end{split}
\end{equation}
\end{proposition}
Our further analysis, both for martingale and Markov noise, will essentially rely on the decoupled TTSA updates \eqref{eq:2ts2-1}. We refer to this dynamics as to the "decoupled" one, since the update of the scale $\tilde{w}_{k+1}$ no longer depends directly on $\ttheta_{k}$, only through the noise variables $V_{k+1}$ and $W_{k+1}$. Now we aim to upper bound the quantities
\begin{equation}
\label{eq:tilde_moments}
M_{k,p}^{\tw} := \PE^{1/p}[\|\tw_k\|^p], \quad M_{k,p}^{\ttheta} := \PE^{1/p}[\|\ttheta_k\|^p]\eqsp.
\end{equation}

Similarly to \eqref{eq:contraction_p}, we show in the supplement paper, that
\begin{equation}
\label{eq:contraction_B11}
\begin{split}
\normop{\Id - \beta_k B_{11}^k}[Q_\Delta] 
&\leq 1 - (1/2) \beta_k a_\Delta \eqsp, \\ 
\normop{\Id - \gamma_k B_{22}^k}[Q_{22}] 
&\leq 1 - (1/2) \gamma_k a_{22} \eqsp.
\end{split}
\end{equation}
The result \eqref{eq:contraction_B11} together with the structure of the updates \eqref{eq:2ts2-1} enables us to expand the recurrence and to show that the error component, associated with the initial error $\theta_0 - \thetas$ and $w_0 - \wstar$ decay at the exponential rate. Precisely, the following bound holds:
\begin{proposition}
\label{prop:moments_bound}
Let $p \geq 2$ and assume \Cref{assum:zero-mean},\Cref{assum:bound-conditional-moments-p}($p$), \Cref{assum:quadratic_characteristic}, \Cref{assum:hurwitz}, \Cref{assum:stepsize}($p$), and \Cref{assum:aij_bound}. Then for any \(k\in\mathbb{N}\) it holds
\begin{align}
\label{eq:mth_moment_bound}
\textstyle M_{k+1,p}^{\ttheta} & \textstyle \lesssim   \prod_{j=0}^{k}\big(1 - \beta_ja_{\Delta}/8\big) + p^2 \beta_{k}^{1/2}\eqsp, \\
\label{eq:mtw_bound}
\textstyle M_{k+1,p}^{\tw} & \textstyle \lesssim   \prod_{j=0}^{k}\big(1-\gamma_ja_{22}/8\big) + p^3 \gamma_{k}^{1/2}\eqsp,
    \end{align}
where $\lesssim$ stands for inequality up to constants not depending upon $k$ and $p$. 
\end{proposition}
\paragraph{Discussion.} \Cref{prop:moments_bound} provides, to best of our knowledge, the first high-order moment bounds in the linear TTSA with martingale noise. The scaling of the r.h.s. with $\beta_k^{1/2}$ for $M_{k+1,p}^{\ttheta}$ and $\gamma_k^{1/2}$ for $M_{k+1, p}^{\tw}$ coincides with the one previously obtained for the particular case $p = 2$ in \cite{kaledin2020finite}. Similar asymptotic results were previously obtained in \cite{konda:tsitsiklis:2004}. We expect that the dependence of the r.h.s. of \eqref{eq:mth_moment_bound} and \eqref{eq:mtw_bound} upon $p$ can be improved based on applying the Pinelis version of Rosenthal inequality \cite[Theorem 4.1]{pinelis_1994} instead of Burkholder's inequality \cite[Theorem 8.6]{osekowski}, that was used in the current proof, yet we expect that this approach introduces additional technical difficulties. 

\vspace{-2mm}
\subsection{GAR for Polyak-Ruppert averaged TTSA}
\label{sec:TTSA_averaged}
Based on the results of the previous section, we can now quantify the Gaussian approximation rates for $\sqrt{n}(\btheta_n-\thetas)$ for the Polyak-Ruppert averaged estimator $\btheta_n$ from \eqref{eq:pr_averaged_iterates}. Now we present the key decomposition:
\begin{equation}
\label{eq:pr_eq_3}
\begin{split}
\Delta(\theta_k-\thetas) &= \frac{\theta_k - \theta_{k+1}}{\beta_k} - \frac{A_{12}A_{22}^{-1}(w_k - w_{k+1})}{\gamma_k} \\ 
&\textstyle \qquad +(V_{k+1} - A_{12}A_{22}^{-1}W_{k+1})\eqsp.
\end{split}
\end{equation}
The proof of the above identity is given in the supplement paper. Taking sum in \eqref{eq:pr_eq_3} for $k=1$ to $n$, and using the definition of $V_{k+1},W_{k+1}$ in \eqref{eq:noise_term_new}, we get:
\begin{equation}
\label{eq:pr_theta_decomposition}
\textstyle
\sqrt{n}\Delta(\btheta_n - \thetas) = \frac{1}{\sqrt{n}} \sum_{k=1}^n \psi_{k+1} + R_n^{\operatorname{pr}}\eqsp, 
\end{equation}
where we set $\psi_{k+1} = \funnoisew_V^{k+1} - A_{12} A_{22}^{-1} \funnoisew_W^{k+1}$, and $R_n^{\operatorname{pr}}$ is a residual term defined in the supplement paper. Assumption \Cref{assum:quadratic_characteristic} implies that the variance $\var[\funnoisew_V^{k+1} - A_{12} A_{22}^{-1} \funnoisew_W^{k+1}]$ is constant for any $k$, so we can define
\begin{align}
\label{def:sigma_star_main}
\lineG_{\eps} := \var[\funnoisew_V^{1} - A_{12} A_{22}^{-1} \funnoisew_W^{1}] \in \rset^{\dth \times \dth} \eqsp.
\end{align}
The following theorem holds:
\begin{theorem}
\label{th:martingal_pr_clt}
Assume \Cref{assum:zero-mean},\Cref{assum:bound-conditional-moments-p}($\log n$), \Cref{assum:quadratic_characteristic}, \Cref{assum:hurwitz}, \Cref{assum:stepsize}($\log n$), and \Cref{assum:aij_bound}. Then, it holds that
\begin{equation}
\kolmogorov\bigl(\sqrt{n}\Delta(\bar{\theta}_n - \thetas),\mathcal{N}(0,\lineG_{\eps})\bigr) \lesssim_{\log n} \frac{1}{n^{a/2}} 
+ \frac{1}{n^{(1-b)/2}}  \eqsp.
\end{equation}
\end{theorem}
\paragraph{Proof sketch.} We apply \Cref{prop:nonlinearapprox} to the decomposition \eqref{eq:pr_theta_decomposition} and obtain, with $\nu \sim \mathcal{N}(0,  \lineG_\eps )$, that 
\begin{multline}
\label{eq:main:th_kolmogorov_first_bound_appendix}
\textstyle 
\kolmogorov\bigl({\sqrt{n} \Delta(\btheta_n - \thetas), \nu \bigr)} \leq \underbrace{\kolmogorov\bigl({n^{-1/2} \sum_{k=1}^n  \psi_{k+1}, \nu}\bigr)}_{T_1} \\
\textstyle
+ \underbrace{2 c_d^{p/p+1}\PE^{1/(p+1)}\bigl[\norm{\lineG_\eps^{-1/2} R_n^{\operatorname{pr}}}^{p}\bigr]}_{T_2}\eqsp.
\end{multline}
Due to \Cref{assum:zero-mean} and \Cref{assum:aij_bound}, sequence $\{\psi_{k+1}\}_{k \in \nset}$ is a bounded martingale-difference sequence w.r.t. $\F_k$ with constant quadratic characteristic. Hence, $T_1$ can be estimated applying a slight modification of \cite[Theorem 1]{wu2025uncertainty}. It remains to bound the moments of $T_2$, which is done using \Cref{prop:moments_bound}.

\paragraph{Discussion.} In the theorem above, the coefficients before the terms depend upon the initial errors $\norm{\theta_0 - \thetas}$, $\norm{w_0 - \wstar}$, and upon the factors $1/(1-a)$ and $1/(1-b)$. That is why the result in its current form does not apply directly if $b = 1$. We expect that the result holds in this case as well, perhaps at a price of introducing additional logarithmic factors. The same remark applies to \Cref{th:last_iter_clt_final}-\Cref{th:markov:last_iter_clt} stated below.
\par 
Since $1/2 < a < b < 1$, the bound of \Cref{th:martingal_pr_clt} is optimized when setting $a = 1/2 + 1/\log{n}$ and $b = a + 1/\log{n}$, yielding the final rate of convergence of order 
\begin{equation}
\label{eq:non_linear_ttsa}
\kolmogorov\bigl(\sqrt{n}\Delta(\bar{\theta}_n - \thetas),\mathcal{N}(0,\lineG_{\eps})\bigr) \lesssim_{\log n} n^{-1/4}\eqsp.
\end{equation}
The result of \eqref{eq:non_linear_ttsa} improves upon the previously established results of \cite{srikant20252ts}. The authors of that paper obtained a rate of $n^{-1/4}$, up to $\log{n}$ factors, in terms of Wasserstein distance. This implies convergence rate $n^{-1/8}$ in the convex distance, which is slower than \eqref{eq:non_linear_ttsa}. The choice of $a$ and $b$ in \eqref{eq:non_linear_ttsa} corresponds to nearly the same scales for $\beta_k$ and $\gamma_k$, effectively reducing the problem to a single-scale LSA. The obtained $n^{-1/4}$ rate aligns with the one established for this problem with i.i.d. noise in \cite{samsonov2024gaussian}.

\vspace{-2mm}
\subsection{GAR for the last iterate.}
\label{sec:TTSA_last_iterate}
In this section, we derive the normal approximation rates for the last iterate $\beta_n^{-1/2} \ttheta_{n+1}$. Following \cite{konda:tsitsiklis:2004} and using \eqref{eq:2ts2-1}, equations for \(\ttheta_k\) and \(\tw_k\) writes as
\begin{align}
\label{eq:difference_iterations_theta_w}
\ttheta_{k+1} &= ( \Id - \beta_k \Delta ) \ttheta_k - \beta_k A_{12} \Tilde{w}_k - \beta_k V_{k+1} + \beta_k\delta_k^{(1)}\eqsp,\\
\tw_{k+1} &= (\Id - \gamma_kA_{22})\tw_k -\beta_kD_{k}V_{k+1} - \gamma_k W_{k+1} - \beta_k\delta_k^{(2)}\eqsp,
\end{align}
where we set
\begin{align}
\label{eq:delta_definition}
\textstyle
\delta_k^{(1)} = A_{12}L_k\ttheta_k \eqsp, \eqsp \eqsp \eqsp  \delta_k^{(2)} = -(L_{k+1} + A_{22}^{-1}A_{21})A_{12}\tw_k\eqsp.
\end{align}
Throughout the analysis we use the following convention:
\begin{equation}
\label{eq:matr_product_pr}
\textstyle
G_{m:k}^{(1)} := \prod_{i=m}^{k} (\Id - \beta_i\Delta ), \eqsp G_{m:k}^{(2)} := \prod_{i=m}^{k} (\Id - \gamma_i A_{22} )\eqsp. 
\end{equation}
Enrolling the above recurrence and following \cite{konda:tsitsiklis:2004}, we get from the previous recurrence that 
\begin{equation}
\label{eq:last_iteration_repr_new}
\textstyle
\ttheta_{n+1} = -\sum_{j=0}^{n} \beta_j G_{j+1:n}^{(1)} \psi_{j+1} + R_n^{\operatorname{last}} \eqsp,
\end{equation}
where $R_n^{\operatorname{last}}$ is a remainder term defined in the supplement paper. The leading term in representation \eqref{eq:last_iteration_repr_new} is a linear statistics of $\eps_V, \eps_W$ which are martingale difference sequences with constant quadratic characteristics due to \Cref{assum:quadratic_characteristic}. Now we define
\begin{align}
\textstyle 
\lineG_n^{\operatorname{last}} = \var\bigl[\sum_{j=0}^{n} \beta_j G_{j+1:n}^{(1)} \psi_{j+1}\bigr] \eqsp.
\end{align}
It is known that $ \beta_n^{-1}\lineG_n^{\operatorname{last}}$ converges to a fixed matrix $\lineG_{\infty}^{\operatorname{last}}$ which is a solution of the Ricatti equation
\begin{align}
\lineG_{\infty}^{\operatorname{last}}  = \beta_0(\Delta \lineG_{\infty}^{\operatorname{last}} + \lineG_{\infty}^{\operatorname{last}}\Delta^\top - \lineG_{\eps}) \eqsp,
\end{align}
where $\lineG_{\eps}$ is defined in \eqref{def:sigma_star_main}. Moreover, the convergence rate is proportional to $\beta_n$, i.e. 
$$
\norm{\beta_n^{-1} \lineG_{n}^{\operatorname{last}} - \lineG_{\infty}^{\operatorname{last}}} \lesssim n^{-b} \eqsp.
$$ 
The proof of the above result is given in the supplement paper. The following assumption guarantees that the covariance matrix $\beta_n^{-1} \lineG_n^{\operatorname{last}}$ is non-degenerate, which is important for the further applications of \Cref{prop:nonlinearapprox}.
\begin{assum}
\label{assum:last_iter}
Step size exponents $a, b$ satisfy $2b > 1 + a$. Moreover, assume that the total number of iterations $n$ satisfies $n^b \geq C_{\Cref{assum:last_iter}}$, where $C_{\Cref{assum:last_iter}}$ does not depend on $a,b$, and can be traced following the supplement paper.
\end{assum}
\begin{theorem}
\label{th:last_iter_clt_final}
Assume \Cref{assum:zero-mean},\Cref{assum:bound-conditional-moments-p}($\log n$), \Cref{assum:quadratic_characteristic}, \Cref{assum:hurwitz}, \Cref{assum:stepsize}($\log n$), \Cref{assum:aij_bound}, \Cref{assum:last_iter}.  Then, it holds that
\begin{multline}
\label{eq:normal_approximation_rhs_main}
\kolmogorov{\bigl(\beta_n^{-1/2}\ttheta_{n+1},\mathcal{N}(0,  \lineG_{\infty}^{\operatorname{last}})\bigr)} \\  
\qquad \lesssim_{\log n} n^{b/2} \prod_{j=0}^n (1 - \frac{a_\Delta}{8} \beta_j) + \frac{1}{n^{(3b-a-2)/2}}  \eqsp.
\end{multline}
\end{theorem}
\paragraph{Discussion} The proof of \Cref{th:last_iter_clt_final} is similar to the one of \Cref{th:martingal_pr_clt}, but relies on the decomposition \eqref{eq:last_iteration_repr_new} instead of \eqref{eq:pr_theta_decomposition} used in the averaged setting. Additional technical difficulties arises when controlling the moments of the term $R_n^{\operatorname{last}}$. Bounding the latter term requires additional constraint $2b > 1+a$ imposed in \Cref{assum:last_iter}.
\par 
Since $1/2 < a < b < 1$, the bound of \Cref{th:last_iter_clt_final} is optimized when setting $a = 1/2 + 1/\log{n}$ and $b = 1 - 1/\log{n}$, yielding the final rate
\begin{equation}
\textstyle 
\kolmogorov{\bigl(\beta_n^{-1/2}\ttheta_{n+1},\mathcal{N}(0,  \lineG_{\infty}^{\operatorname{last}})\bigr)} \lesssim_{\log n} n^{-1/4}\eqsp, 
\end{equation}
provided that $n$ is large enough. To the best of our knowledge, this is the first result concerning the Gaussian approximation rate for the TTSA last iterate. 
\par 
 Note that \Cref{th:last_iter_clt_final} reveals phenomenon, which is completely different from what was previously observed for the Polyak-Ruppert averaged iterates in \Cref{th:martingal_pr_clt}. Indeed, the right-hand side of the bound \eqref{eq:normal_approximation_rhs_main} contains the term $n^{-(3b-a-2)/2}$, which favors separation between $\beta_k$ and $\gamma_k$, and vanishes when the scale exponents are close.

\section{GAR for TTSA with Markov noise}
\label{sec:markov_case}
In this section we generalize the results obtained in \Cref{sec:TTSA_martingale} to the more practical scenario when $\{X_k\}_{k \in \nset}$ form a Markov chain. Namely, we impose the following assumption:
\begin{assumB}
\label{assum:UGE}
The sequence $\{X_k\}_{k \in \nset}$ is a Markov chain taking values in a Polish space $(\Xset,\Xsigma)$ with the Markov kernel $\MKQ$. Moreover, $\MKQ$ admits $\pi$ as a unique invariant distribution and is uniformly geometrically ergodic, that is, there exists $\taumix \in \nset$, such that for any $k \in \nset$, it holds that 
\begin{equation}
\label{eq:tau_mix_contraction}
\dobrush(\MKQ^{k}) := \sup_{x,x' \in \Xset} \tvdist(\MKQ^{k}(x,\cdot),\MKQ^{k}(x',\cdot)) \leq (1/4)^{\lceil k / \taumix \rceil}\eqsp.
\end{equation}
Moreover, for all $k \in \nset$ and $i, j \in \{1,2\}$ it holds that
\begin{align}
    \PE_\pi[\funcAw_{ij}^k] = A_{ij} \text{ and } \PE_\pi[\funcbw_i^k] = b_i \eqsp.
\end{align}
\end{assumB}

Parameter $\taumix$ in \Cref{assum:UGE} is referred to as a \emph{mixing time}, see e.g. \cite{paulin_concentration_spectral}, and controls the rate of convergence of the iterates $\MKQ^{k}$ to $\pi$ as $k$ increases.

\subsection{Moment bounds for TTSA with Markov noise}
First, we introduce a counterpart to \Cref{assum:stepsize} that is needed to derive moment bounds for the setting of Markov noise.
\begin{assumB}[$p$]
\label{assum:markov:stepsize}
$(\gamma_k)_{k \geq 1}$, $(\beta_k)_{k \geq 1}$ are non-increasing sequences of the form 
\[
\textstyle 
\beta_k = c_{0, \beta} (k+k_0)^{-b}, \quad \gamma_k = c_{0, \gamma} (k+k_0)^{-a} \eqsp, 
\]  
where \(1/2 < a < b < 1\), fraction $c_{0,\beta}/c_{0,\gamma}$ is small enough, and constant $k_0$ satisfies the bound $k_0 \geq C_{\Cref{assum:markov:stepsize}} p^{4/b}$, where the constant $C_{\Cref{assum:markov:stepsize}}$ does not depend upon $p$.
\end{assumB}
The proof of moment bounds is more involved compared to the martingale noise case. Following the decomposition outlined in \cite{kaledin2020finite}, we first represent the noise variables $(V_{k+1},W_{k+1})$ as a sum of their martingale $(V^{(0)}_{k+1},W^{(0)}_{k+1})$ and Markovian components $(V^{(1)}_{k+1},W^{(1)}_{k+1})$ in a way that 
\[
\textstyle 
V_{k+1} = V^{(0)}_{k+1} + V^{(1)}_{k+1}\eqsp, \quad W_{k+1} = W^{(0)}_{k+1} + W^{(1)}_{k+1}\eqsp.
\]
Here $\CPE{V_{k+1}^{(0)}}{\F_k} = 0$ and $\CPE{W_{k+1}^{(0)}}{\F_k} = 0$. This representation is obtained using the decomposition associated with the Poisson equation, see \cite[Chapter~21]{douc:moulines:priouret:soulier:2018} and additional summation by parts. Then we define a pair of coupled recursions, which form exact counterparts of \eqref{eq:2ts2-1}:
\begin{equation}
\label{eq:2ts2-1-markov}
\begin{split}
\ttheta_{k+1}^{(i)} &= ( \Id - \beta_k B_{11}^k ) \ttheta_k^{(i)} - \beta_k A_{12} \Tilde{w}_k^{(i)} - \beta_k V_{k+1}^{(i)}\eqsp, \\
\tilde{w}_{k+1}^{(i)} &= (\Id - \gamma_k B_{22}^k ) \Tilde{w}_k^{(i)} - \beta_k D_k V_{k+1}^{(i)} - \gamma_k W_{k+1}^{(i)}\eqsp,
\end{split}
\end{equation}
where $i \in \{0,1\}$. Then it is easy to see that $\ttheta_k = \ttheta_k^{(0)} + \ttheta_k^{(1)}$ and $\tw_k = \tw_k^{(0)} + \tw_k^{(1)}$. Precise expressions for $\ttheta_k^{(i)}, \tw_k^{(i)},V^{(i)}_{k},W^{(i)}_{k}$ can be found in the supplement paper. 
\begin{proposition}
\label{prop:markov:moment_bounds}
Let $p \geq 2 $. Assume \Cref{assum:hurwitz}, \Cref{assum:aij_bound},  \Cref{assum:UGE}, \Cref{assum:markov:stepsize}($p$). Thus, it holds for any $k \geq 0$ that
\begin{align}
\label{eq:markov:main_ttheta_bound}
\textstyle 
M_{k+1, p}^{\ttheta} & \textstyle \lesssim \prod_{j=0}^k (1 - \frac{a_\Delta \beta_j}{8}) + p^2 \sqrt{\beta_k} \eqsp, \\
\label{eq:markov:main_tw_bound}
\textstyle M_{k+1, p}^{\tw} & \textstyle \lesssim \prod_{j=0}^{k} (1 - \frac{a_{22} \gamma_j}{8}) + p^3 \sqrt{\gamma_k} \eqsp.
\end{align}
\end{proposition}
\paragraph{Proof sketch.} The idea of the proof is to bound martingale and Markov parts separately using the techniques from \Cref{sec:TTSA_martingale}. Note that \Cref{prop:markov:moment_bounds} directly mimics the similar result obtained under the martingale noise setting in \Cref{prop:moments_bound}. The only difference is that the constants hidden under $\lesssim$ additionally depends upon the parameter $\taumix$.

\subsection{GAR for Polyak-Ruppert averaged TTSA}
To proceed with Gaussian approximation for Polyak-Ruppert averaging, we use the decomposition \eqref{eq:pr_theta_decomposition} to transform the linear statistic $\sum_{k=1}^n \psi_{k+1}$ to a sum of martingale-increments. This transformation is done through the Poisson equation, see \cite[Chapter~21]{douc:moulines:priouret:soulier:2018}. Under \Cref{assum:aij_bound}, function $\psi(x) = \funnoisew_V(x) - A_{12} A_{22}^{-1} \funnoisew_W(x) $ is a.s. bounded, which implies that there exists a function $\pois{\psi}{}: \Xset \rightarrow \rset^{d_\theta}$, such that
\begin{align}
\pois{\psi}{}(x) - \MKQ \pois{\psi}{}(x) = \psi(x) \eqsp.
\end{align}
We set $\pois{\psi}{k+1} := \pois{\psi}{}(X_{k+1})$ and define 
\begin{equation}
\label{eq:mart_increment_markov}
\textstyle 
\Mart_k = \pois{\psi}{k+1} - \MKQ \pois{\psi}{k}\eqsp,
\end{equation}
which form a martingale-increment w.r.t. $\F_k$. Then we can rewrite \eqref{eq:pr_theta_decomposition} as 
\begin{align}
\label{eq:markov:main:pr_decomposition}
\textstyle
\sqrt{n} \Delta (\btheta_n - \thetas) = \frac{1}{\sqrt{n}} \sum_{k=1}^n \Mart_k + R_{n}^{\operatorname{pr, m}} \eqsp,
\end{align}
where $R_{n}^{\operatorname{pr, m}}$ is a residual term defined in the supplement. Under \Cref{assum:UGE} there exists a matrix $\cmark{\lineG}{\infty} \in \rset^{d_\theta \times d_\theta}$ such that
\begin{equation}
\label{eq:markov:main:sigma_inf_def}
\textstyle
n^{-1/2} \sum_{k=1}^n  \{\psi_{k+1} - \pi(\psi)\} \xrightarrow{d} \mathcal{N}(0, \cmark{\lineG}{\infty}) \eqsp.
\end{equation}
Due to \cite[Theorem~21.2.5]{douc:moulines:priouret:soulier:2018}, we get that 
\begin{align}
\var[\Mart_k] = \cmark{\lineG}{\infty}\eqsp.
\end{align}
Now we state the counterpart to \Cref{th:martingal_pr_clt}:
\begin{theorem}
\label{th:markov:pr_clt}
Assume \Cref{assum:hurwitz}, \Cref{assum:aij_bound},  \Cref{assum:UGE}, \Cref{assum:markov:stepsize}($\log n$). Then it holds that
\begin{align}
\label{eq:pr_markov_rate} 
&\kolmogorov(\sqrt{n} \Delta (\btheta_n - \thetas), \mathcal{N}(0, \cmark{\lineG}{\infty})) \\
&\lesssim_{\log n} \frac{1}{n^{1/4}} + \frac{1}{n^{(1-b)/2}} + \frac{1}{n^{a-\frac{1}{2}}} + \sqrt{n} \prod_{j=0}^{n-1} (1 - \frac{a_\Delta \beta_j}{16})\eqsp.
\end{align}
\end{theorem}

\paragraph{Proof sketch.} The proof of \Cref{th:markov:pr_clt} consists of two main parts. First, we derive a Gaussian approximation rate for the linear statistic $\frac{1}{\sqrt{n}} \sum_{k=1}^n \Mart_k$ using an appropriate martingale CLT. It is especially non-trivial, since $\CPE{\Mart_k \{\Mart_k\}^{\top}}{\F_k}$ is not constant. We circumvent this problem using an appropriate modification of the argument due to \cite{fan2019exact}. Next, we estimate the moments of $R_n^{\operatorname{pr, m}}$ using the techniques established in \Cref{prop:moments_bound} for $\ttheta^{(0)}_k, \tw_k^{(0)}$ and then combining this with a separate bounds for the Markov part $\ttheta^{(1)}_k, \tw_k^{(1)}$.

\paragraph{Discussion.} It is easy to see that, given that $b > a$, the right-hand side of \eqref{eq:pr_markov_rate} is optimized when setting $a = 2/3$ and $b = 2/3 + 1/(\log{n})$. This yields the final rate of order $n^{-1/6}$ up to logarithmic factors:
\begin{equation}
\label{eq:pr_markov_rate_optimized}
\textstyle 
\kolmogorov(\sqrt{n} \Delta (\btheta_n - \thetas), \mathcal{N}(0, \cmark{\lineG}{\infty})) \lesssim_{\log n} n^{-1/6} \eqsp.
\end{equation}  
To the best of our knowledge, \eqref{eq:pr_markov_rate_optimized} provides the first result concerning the Gaussian approximation rates for the TTSA problems with Markov noise. The suggested step size schedule mimics the one predicted by \Cref{th:martingal_pr_clt} and essentially reduces the TTSA scheme to a single-timescale one. 

\subsection{GAR for last iterate of TTSA}
We start this section by introducing a counterpart to \eqref{eq:last_iteration_repr_new} based on the idea of the decomposition \eqref{eq:markov:main:pr_decomposition} for Polyak-Ruppert averaging:
\begin{equation}
\label{eq:markov:main:last_iter_repr}
\textstyle 
\beta_n^{-1/2} \ttheta_{n+1} = -\sum_{j=0}^{n} \beta_j G_{j+1:n}^{(1)} \Mart_j + R_n^{\operatorname{last, m}} \eqsp,
\end{equation}
where $R_n^{\operatorname{last, m}}$ is a residual term that is given in the supplement paper. Note that the leading term in representation \eqref{eq:markov:main:last_iter_repr} is martingale difference sequence. Now we define
\begin{align}
\textstyle
\lineG_n^{\operatorname{last, m}} = \var\bigl[\sum_{j=0}^{n} \beta_j G_{j+1:n}^{(1)} \Mart_j \bigr] \eqsp.
\end{align}
It is known that $\beta_n^{-1} \lineG_n^{\operatorname{last, m}}$ converges to a fixed matrix $\lineG_{\infty}^{\operatorname{last, m}}$ which is a solution of the Ricatti equation
\begin{align}
\label{eq:markov:main:sigma_last_inf_def}
    \lineG_{\infty}^{\operatorname{last, m}}  = \beta_0(\Delta \lineG_{\infty}^{\operatorname{last, m}} + \lineG_{\infty}^{\operatorname{last, m}}\Delta^\top - \cmark{\lineG}{\infty}) \eqsp,
\end{align}
where $\cmark{\lineG}{\infty}$ is defined in \eqref{eq:markov:main:sigma_inf_def}. Moreover, the convergence rate is proportional to $\beta_n$, i.e. 
$$
\norm{\beta_n^{-1} \lineG_{n}^{\operatorname{last, m}} - \lineG_{\infty}^{\operatorname{last, m}}} \lesssim n^{-b} \eqsp.
$$ 
The proof of the above result is given in the supplement paper. Now we formulate a counterpart to \Cref{assum:last_iter}:
\begin{assumB}
\label{assum:markov:last_iter}
Step size exponents $a, b$ satisfy $2b > 1 + a$. Moreover, assume that the total number of iterations $n$ satisfies $n^b \geq C_{\Cref{assum:markov:last_iter}}$, where $C_{\Cref{assum:markov:last_iter}}$ does not depend on $a,b$, and can be traced from the supplement paper.
\end{assumB}

\begin{theorem}
    \label{th:markov:last_iter_clt}
    Assume \Cref{assum:hurwitz}, \Cref{assum:aij_bound},    \Cref{assum:UGE}, \Cref{assum:markov:stepsize}($\log n$), \Cref{assum:markov:last_iter}. Then it holds that
    \begin{align}
        \label{eq:main:last_iter_markov_not_optimized}&\kolmogorov(\beta_n^{-1/2} \ttheta_{n+1}, \mathcal{N}(0, \lineG_{\infty}^{\operatorname{last, m}})) \\
        &\lesssim_{\log n} n^{b/2} \prod_{j=0}^n (1 - \frac{a_\Delta}{8} \beta_j) +\frac{1}{n^{\frac{b}{2} - \frac{1}{4}}} + \frac{1}{n^{a - \frac{b}{2}}} + \frac{1}{n^{\frac{3b-a-2}{2}}} \eqsp.
    \end{align}
\end{theorem}
\paragraph{Proof sketch.} The proof of \Cref{th:markov:last_iter_clt} uses the same machinery of Gaussian approximation for non-linear statistics based on representation \eqref{eq:markov:main:last_iter_repr}. In this setting control of the moments of the term $R_n^{\operatorname{last, m}}$ is a delicated problem, which requires the additional constraint $2b > 1 + a$ imposed in \Cref{assum:markov:last_iter}.

\paragraph{Discussion.} It is easy to see that, given that $b \geq a$, the right-hand side of \eqref{eq:main:last_iter_markov_not_optimized} is optimized when setting $a = 2/3$ and $b = 1 - 1/(\log{n})$, and yields the final rate in terms of $n$ of order up to $n^{-1/6}$ up to logarithmic factors:
\begin{equation}
\label{eq:last_iter_markov_rate_optimized}
\kolmogorov( \beta_{n}^{-1/2} \ttheta_{n+1}, \mathcal{N}(0, \lineG_{\infty}^{\operatorname{last, m}})) \lesssim_{\log n} n^{-1/6}  \eqsp.
\end{equation}
This rate, to the best of our knowledge, is the first one obtained for the last iterate of TTSA with Markov noise.

\section{Applications to TDC and GTD}
\label{sec:applications_GTD}
In this section, we show that the results derived in \Cref{sec:TTSA_martingale} and \Cref{sec:markov_case} apply to the Gradient Temporal Difference (GTD) \cite{Sutton_2008} and Temporal Difference with Gradient Correction (TDC) \cite{sutton:gtd2:2009} methods. These methods address the problem of classical TD learning, which is based on single-timescale stochastic approximation and is known to fail in off-policy RL settings where data are drawn from a \emph{behavior policy} different from the target policy \cite{baird:resid:rl:funcappr:1995,tsitsiklis:td:1997}. We consider a discounted MDP (Markov Decision Process) given by a tuple $(\S,\A,\PMDP,r,\lambda)$. Here $\S$ and $\A$ denote state and action spaces, which are assumed to be complete separable metric spaces with their Borel $\sigma$-algebras $\borel{\S}$ and $\borel{\A}$, and $\lambda \in (0,1)$ is a discount factor. Let $\PMDP(\cdot | s,a)$ be a state-action transition kernel, which determines the probability of moving from $(s,a)$ to a set $B \in \borel{\S}$. Reward function $r\colon \S \times \A \to [0,1]$ is deterministic. A \emph{Policy} $\pi(\cdot|s)$ is the distribution over action space $\A$ corresponding to agent's action preferences in state $s \in \S$. We aim to estimate \emph{value function}
\[
\textstyle 
V^{\pi}(s) = \PE\bigl[\sum_{k=0}^{\infty}\lambda^{k}r(S_k,A_k)|S_0 = s\bigr]\eqsp,
\]
where $A_{k} \sim \pi(\cdot | s_k)$, and $S_{k+1} \sim \PMDP(\cdot | S_{k}, A_{k})$. Define the transition kernel under $\pi$,
\begin{equation}
\label{eq:transition_matrix_P_pi}
\textstyle 
\PMDP_{\pi}(B | s) = \int_{\A} \PMDP(B | s, a)\pi(\rmd a|s)\eqsp.
\end{equation}
We consider the \emph{linear function approximation} for $V^{\pi}(s)$, defined for $s \in \S$, $\theta \in \rset^{d}$, and a feature mapping $\varphi\colon \S \to \rset^{d}$ as $V_{\theta}^{\pi}(s) =  \varphi^\top(s) \theta$. Our goal is to find a parameter $\thetas$, which defines the best linear approximation to $V^{\pi}$. We denote by $\mu$ the invariant distribution over the state space $\S$ induced by $\PMDP^{\pi}(\cdot | s)$ in \eqref{eq:transition_matrix_P_pi}. Consider the following assumptions on the generative mechanism and on the feature mapping $\varphi(\cdot)$:
\begin{assumTD}
\label{assum:generative_model}
Tuples $(s_k,a_k,s_{k}')$ are generated \iid with $s_k \sim \mu$, $a_k \sim \pi(\cdot|s_k)$, $s_{k}' \sim \PMDP(\cdot|s_k,a_k)$\eqsp.
\end{assumTD}
\begin{assumTD}
\label{assum:phi_norm_bound}
Feature mapping $\varphi(\cdot)$ satisfies 
$\sup_{s \in \S} \|\varphi(s) \| \leq 1$.
\end{assumTD}
As an alternative to the generative model setting \Cref{assum:generative_model}, our analysis covers the Markov noise setting:
\begin{assumTD}
\label{assum:P_pi_ergodicity_td}
Suppose that we obtain a Markovian sample trajectory \(\{(s_k,a_k,r_k)\}_{k=0}^{\infty}\) which is generated when a stationary behavior policy \(\pi\) is employed. Assume that the Markov kernel $\PMDP_{\pi}$ admits a unique invariant distribution $\mu$ and is uniformly geometrically ergodic, that is, there exist $\taumix \in \nset$, such that for any $k \in \nset$, it holds that
\begin{equation}
\label{eq:drift-condition-main}
 \sup_{s,s' \in \S} \tvdist(\PMDP_{\pi}^{k}(\cdot|s), \PMDP_{\pi}^{k}(\cdot|s')) \leq  (1/4)^{\lceil k / \taumix \rceil}\eqsp.
\end{equation}
\end{assumTD}
We introduce the $k$-th step TD error for the linear setting:
\begin{equation}
\label{eq:td_error}
\delta_k = r_k + \lambda \theta_{k}^\top \varphi_{k+1} - \theta_{k}^\top \varphi_{k}\eqsp,
\end{equation}
where we have defined 
\begin{align}
\varphi_k = \varphi(s_k)\eqsp, \quad r_k = r(s_k, a_k) \eqsp.
\end{align}
\paragraph{Generalized Temporal Difference learning.} The GTD(0) algorithm is defined by the following recurrence for $k \geq 1$:
\begin{equation}
\label{eq:GTD}
\begin{cases}
\theta_{k+1} = \theta_k + \beta_k (\varphi_k - \lambda \varphi_{k+1}) (\varphi_k)^\top w_k \eqsp,\quad  \theta_0 \in \rset^{d}\eqsp, \\
w_{k+1} = w_{k} + \gamma_k (\delta_k \varphi_k - w_k) \eqsp, \quad w_0 = 0\eqsp.
\end{cases}
\end{equation}
It is clear that the GTD(0) recurrence \eqref{eq:GTD} is a particular case of the linear TTSA given in \eqref{eq:2ts1}-\eqref{eq:2ts2}.

\paragraph{Temporal-difference learning with gradient correction.} The TDC algorithm employs dual updates for the primary parameter vector \(\theta_k\) and the auxiliary weight vector \(w_k\). Its update rule is given by
\begin{equation}
\label{eq:TDC}
\begin{cases}
\theta_{k+1} = \theta_k + \beta_k \delta_k\varphi_k - \beta_k\lambda \varphi_{k+1} (\varphi_k^\top w_k) \eqsp, \\
w_{k+1} = w_{k} + \gamma_k (\delta_k - \varphi_{k}^\top w_k)\varphi_k \eqsp.
\end{cases}
\end{equation}

It is possible to check that both updates schemes \eqref{eq:GTD} and \eqref{eq:TDC} satisfy the general assumptions \Cref{assum:zero-mean}-\Cref{assum:hurwitz} and \Cref{assum:aij_bound}   under \Cref{assum:generative_model} and \Cref{assum:phi_norm_bound}. Similar, \Cref{assum:UGE} holds under \Cref{assum:P_pi_ergodicity_td}. Thus, all the results from \Cref{sec:TTSA_martingale} and \Cref{sec:markov_case} applies for both algorithms. We provide details in the supplemental paper.

\section{Conclusion}
\label{sec:conclusion}
In this paper, we provided, to the best of our knowledge, the first rate of normal approximation for the last iterate and Polyak-Ruppert averaged TTSA iterates in a sense of convex distance, covering both the martingale-difference and Markov noise settings. A natural further research direction is to consider the problem of constructing confidence intervals for the TTSA solution $(\thetas,\wstar)$ based on bootstrap approach or asymptotic covariance matrix estimation, and perform a fully non-asymptotic analysis of the suggested procedure. Another important direction is the construction of lower bounds to ensure tightness of the rates obtained in \Cref{th:martingal_pr_clt}-\ref{th:markov:last_iter_clt}.

\clearpage
\newpage

\section*{Acknowledgment}
The work was supported by the grant for research centers in the field of AI provided by the Ministry of Economic Development of the Russian Federation in accordance with the agreement 000000C313925P4E0002 and the agreement with HSE University № 139-15-2025-009.

\bibliography{references}

\onecolumn 
\appendix
\begin{appendices}

\section{Martingale limit theorems}
\label{appendix:martingale_limits}

Let $\{X_k\}_{k=1}^{n}$ be a martingale difference process in $\rset^d$ with respect to the natural filtration $\{\mathcal{F}_k\}_{k=0}^{n}$, $\F_k = \sigma(X_s : s \leq k)$. From now on we introduce the following expressions
\begin{align}
    \label{def:martingale_expressions}
    & V_k = \sum_{j=1}^k \mathbb{E}^{\F_{j-1}}[X_j X_j^\top] \eqsp, \eqsp \Sigma_k = \frac{1}{k} \sum_{j=1}^k \mathbb{E}[X_j X_j^\top] \eqsp ,
\end{align}
where $\PE^{\F}[\cdot]$ stands for the conditional expectation w.r.t. a sigma-algebra $\F$.

In order to derive a modification of \cite[Corollary 2]{wu2025uncertainty} we first state \cite[Proposition A.1]{nourdin2022multivariate} that controls convex distance in terms of Wasserstein one:
\begin{lemma}[Proposition A.1 in \cite{nourdin2022multivariate}]
\label{lem:convex_was_bound}
Fix $d \geq 1$, and let $\eta \sim \mathcal{N}(0,\Sigma)$ denote a $d$-dimensional centered Gaussian vector
with invertible covariance matrix $\Sigma$. Then, for any $d$-dimensional random vector $F$ one has
that
\begin{align}
    \kolmogorov(F, \eta) \lesssim \Gamma(\Sigma)^{1/2} d_W(F, \eta)^{1/2} \eqsp,
\end{align}
where $d_W(\cdot, \cdot)$ stands for the Wasserstein distance and $\Gamma(\Sigma)$ is the isoperimetric constant defined by
\begin{align}
    \Gamma(\Sigma) := \sup_{Q \in \Conv(\rset^{d}), \eps > 0} \frac{\P(\eta \in Q^\eps) - \P(\eta \in Q)}{\eps} \eqsp,
\end{align}
where $Q^\eps$ indicates the set of all
elements of $\mathbb{R}^d$ whose Euclidean distance from $Q$ does not exceed $\eps$. 
\end{lemma}

\begin{remark}
    Following \cite[Remark A.2]{nourdin2022multivariate} one can check that for the absolute constants $0 < c < C < \infty$ it holds that
    \begin{align}
        c \sqrt{\norm{\Sigma}[\text{Fr}]} \leq \Gamma(\Sigma) \leq C \sqrt{\norm{\Sigma}[\text{Fr}]} \eqsp,
    \end{align}
    where $\norm{\cdot}[\text{Fr}]$ stands for the Frobenius norm.
\end{remark}
Now we give a slight modification of the result proven in \cite{wu2025uncertainty} that can be obtained applying \Cref{lem:convex_was_bound}:
\begin{lemma}[modified Corollary 2 in \cite{wu2025uncertainty}]
    \label{th:wu_martingale_limit}
    Let $\{X_k\}_{k=1}^n$ be a martingale difference process in with respect to the filtration $\{\mathcal{F}_k\}_{k=0}^n$. Assume that
    \begin{align}
        V_n = n \Sigma_n \text{ a.s.} \eqsp,
    \end{align}
    and for any $1 \leq k \leq n$, $\funcAw \in \rset^{d \times d}$ it holds that
    \begin{align}
        \PE^{\F_{k-1}}[\norm{\funcAw X_k}^2 \norm{X_k} ] \leq M \PE^{\F_{k-1}}[\norm{\funcAw X_k}^2] \eqsp.
    \end{align}
    Then for every $\Sigma \in \mathbb{S}_+^d$ it can be guaranteed that
    \begin{align}
        \kolmogorov\biggl(\frac{1}{\sqrt{n}} \sum_{k=1}^n X_k, \mathcal{N}(0, \Sigma_n)\biggr) \lesssim \Gamma(\Sigma_n)^{1/2} [1 + M(2 + \log(dn \norm{\Sigma_n}))^+]^{1/2} \frac{\sqrt{d \log n}}{n^{1/4}} + \Gamma(\Sigma_n)^{1/2} \frac{\{\trace(\Sigma_n)\}^{1/4}}{n^{1/4}} \eqsp.
    \end{align}
    where $\Sigma_n$ is defined in \eqref{def:martingale_expressions}.
\end{lemma}

The following lemma states the upper bound for the convex distance for any bounded martingale difference sequence:
\begin{lemma}    \label{lem:martinglale_limits}
    Let $0 < \kappa < \infty$ and $\{X_k\}_{k=1}^n$ be a martingale difference process in $\rset^d$ with respect to the filtration $\{\mathcal{F}_k\}_{k=0}^n$ \;, $\mathcal{F}_k = \sigma(X_s:s\leq k)$. Assume that
    \begin{align}
        \max_{1 \leq i \leq n} \norm{X_i} \leq \kappa \text{ almost surely} \eqsp,
    \end{align}
    and there exist constants $C_1, C_2 > 0$ such that for all $t > 0$ it holds that:
    \begin{align}
    \label{eq:lem:mart_limits:proba_assum}
        \P[\norm{V_n - n\Sigma_n} \geq nt] \leq C_1 \exp\bigl(-C_2 nt^2 \bigr) \eqsp,
    \end{align}
    where $V_n$, $\Sigma_n$ are given in \eqref{def:martingale_expressions}.
    Then for any $p \geq 1$ it holds that
    \begin{align}
        &\kolmogorov\biggl(\frac{1}{\sqrt{n}} \sum_{k=1}^n X_k, \mathcal{N}(0, \Sigma_n)\biggr) \lesssim \\
        &\qquad
        \Gamma(\Sigma_n)^{1/2} \biggl\{1 + \kappa \biggl(2 + \log \frac{nd^2 p^{1/2} \norm{\Sigma_n}(2\kappa^2+\norm{\Sigma_n}+ C_2^{-1/2})}{\kappa^2} \biggr)^+\biggr\}^{1/2} \frac{\sqrt{d \log \bigl(ndp^{1/2} \frac{2\kappa^2+\norm{\Sigma_n}+ C_2^{-1/2}}{\kappa^2} \bigr)}}{n^{1/4}} \\ 
        &\qquad + \Gamma(\Sigma_n)^{1/2} \frac{\kappa^{1/2}}{n^{1/4}} + p^{3/4}d  \frac{c_d^{\frac{2p}{2p+1}}}{\norm{n\Sigma_n}^{\frac{2p}{2p+1}}} \biggl\{ \biggl( \frac{n \log n}{C_2}\biggr)^{\frac{p}{2(2p+1)}} + C_1^{\frac{1}{2p+1}} \biggl(\frac{2\kappa^2+\norm{\Sigma_n} + (\frac{\log n}{nC_2})^{1/2}} {\kappa}\biggr)^{\frac{2p}{2p+1}} \biggr\}  \eqsp,
    \end{align}
    where $\Gamma(\cdot)$ is introduced in \Cref{lem:convex_was_bound}, $c_d$ is defined in \Cref{prop:nonlinearapprox}.
\end{lemma}

\begin{proof}
We adapt the arguments from \cite{fan2019exact} and \cite[p. 35]{wu2025uncertainty} for the multidimensional case. Consider the following stopping time
\begin{align}
\label{lem:mart_lim:tau}
\tau = \max\{0 \leq k \leq n : V_k \preceq n \Sigma_n + tn \Id \} \eqsp ,
\end{align}
where $t \in \mathbb{R}_+$ is a parameter we will choose later. Now introduce 
\begin{align}
m = \left\lceil \frac{1}{\kappa^2} \trace(n\Sigma_n + tn\Id - V_\tau) \right\rceil \eqsp, \eqsp N = \left\lceil \frac{\trace(n\Sigma_n + tn\Id)}{\kappa^2} \right\rceil + n \eqsp.
\end{align}
For the further analyses, we observe that
\begin{align}
    n \leq N \leq nd \frac{2\kappa^2+\norm{\Sigma_n} + t} {\kappa^2} \eqsp.
\end{align}
Our goal is to construct the sequence $\{X_i'\}_{i=1}^N$ that has a constant quadratic characteristic equal to $n\Sigma_n$. To proceed, consider the spectral decomposition of $n \Sigma_n -V_\tau$:
\begin{align}
\label{eq:mart_lim:spectral}
n\Sigma_n-V_\tau = \sum_{j=1}^d \lambda_j u_ju_j^\top \eqsp.
\end{align}
Now we set for $i \in \{1, 2, \ldots, N\}$  
\begin{align}
X_i' = \begin{cases}
X_i\eqsp, \eqsp  &1 \leq  i\leq \tau \eqsp, \\
\frac{1}{\sqrt{m}} \sum_{j=1}^d (\lambda_j)^{1/2} \eps_{ij} u_j \eqsp, \eqsp &\tau +1 \leq i \leq \tau + m \eqsp , \\
0\eqsp,  &\tau + m+1 \leq i \leq N  \eqsp,
\end{cases}
\end{align}
where $\eps_{ij}$ are i.i.d. Rademacher random variables. For the natural fitration $\F_i' = \sigma(X_s':s \leq i)$ one can check that $\PE^{\F'_{i-1}}[X_i'] = 0$ almost surely. Moreover, $\norm{X_i'} \leq \kappa$ and $\PE^{\F_{i-1}}[X_i' X_i'^\top] = \frac{1}{m} (n\Sigma_n-V_\tau)$ a.s. by the definition of $\kappa$ and $m$. Thus, we obtain by the construction
\begin{align}
\sum_{i=1}^N \mathbb{E}^{\F_{i-1}}[X_i'X_i'^\top] = V_\tau + m \cdot \frac{1}{m} (n\Sigma_n - V_\tau) = n \Sigma_n \eqsp. 
\end{align}
Now we apply \Cref{prop:nonlinearapprox} and get
\begin{align}
\kolmogorov\bigl( \tfrac{1}{\sqrt{n}} S_n, \; \mathcal{N}(0, \Sigma_n)\bigr) \leq &\eqsp\kolmogorov\bigl( \tfrac{1}{\sqrt{N}} S_N', \; \mathcal{N}(0, \Sigma_n)\bigr) + 2 c_d^{2p/(2p+1)} \norm{n \Sigma_n}^{-\frac{2p}{2p+1}} \bigl(\PE \bigl[ \norm{S_n - S_N'}^{2p} \bigr]\bigr)^{1/(2p+1)} \eqsp,
\end{align}
where we have set
\begin{align}
S_n = \sum_{j=1}^{n} X_j \eqsp , \eqsp S_N' = \sum_{j=1}^{N} X_j' \eqsp.
\end{align}
Since $X'$ satisfies the assumptions of \Cref{th:wu_martingale_limit} with $M := \kappa$ and, moreover, $\trace(\Sigma_n) = \frac{1}{n} \sum_{i=1}^n \PE[X_i^\top X_i] \leq \kappa^2$, we get using \Cref{th:wu_martingale_limit}:
\begin{align}
\kolmogorov(\tfrac{1}{\sqrt{N}} S_N', \mathcal{N}(0, \Sigma_n)) &\lesssim \Gamma(\Sigma_n)^{1/2} \biggl\{1 + \kappa \biggl(2 + \log (d N\norm{\Sigma_n})\biggr)^+ \biggr\}^{1/2} \frac{\sqrt{d \log N}}{N^{1/4}} + \Gamma(\Sigma_n)^{1/2} \frac{\kappa^{1/2}}{N^{1/4}} \eqsp,
\\ & \lesssim  \Gamma(\Sigma_n)^{1/2} \biggl\{1 + \kappa \biggl(2 + \log \frac{nd^2\norm{\Sigma_n}(2\kappa^2+\norm{\Sigma_n}+t)}{\kappa^2} \biggr)^+\biggr\}^{1/2} \frac{\sqrt{d \log (nd \frac{2\kappa^2+\norm{\Sigma_n}+t}{\kappa^2} )}}{n^{1/4}} + \Gamma(\Sigma_n)^{1/2} \frac{\kappa^{1/2}}{n^{1/4}} \eqsp,
\end{align}
To control the moments of $S_n - S_N'$ we consider two events:
\begin{align}
    \Omega_1 = \{\omega: \norm{V_n - n\Sigma_n} \leq tn \} \eqsp, \quad \Omega_2 = \{\omega: \norm{V_n - n\Sigma_n} > tn \} \eqsp.
\end{align}
Hence,
\begin{align}
    \label{lem:convex_clt_Smoments}
    \PE[\norm{S_n - S_N'}^{2p}] = \underbrace{\PE[\norm{S_n - S_N'}^{2p} \1\{\Omega_1\}]}_{\mathcal{T}_1} + \underbrace{\PE[\norm{S_n - S_N'}^{2p} \1\{\Omega_2\}]}_{\mathcal{T}_2} \eqsp. 
\end{align}
Now we bound $\mathcal{T}_1$, $\mathcal{T}_2$ separately.

\paragraph{Bound for $\mathcal{T}_1$.} On the event $\Omega_1$ it holds that $V_n \preceq n\Sigma_n + tn\Id$. Thus, $\tau = n$ and
\begin{align}
    S_n - S_N' = \frac{1}{\sqrt{m}} \sum_{i=n+1}^{n+m} \sum_{j=1}^d (\lambda_j)^{1/2} \eps_{ij} u_j = \frac{1}{\sqrt{m}} \sum_{j=1}^{d} \bigl(\sum_{i=n+1}^{n+m} \eps_{ij} \bigr) (\lambda_j)^{1/2} u_j \eqsp. 
\end{align}
Note that $\norm{V_n - n\Sigma_n} \leq tn$ yields $\lambda_j \leq tn$ for any $j$. Therefore, we obtain
\begin{align}
    \norm{S_n - S_N'}^2 \leq \frac{tn}{m} \sum_{j=1}^{d} \bigl(\sum_{i=n+1}^{n+m} \eps_{ij}\bigr)^2 \eqsp.
\end{align}
Thus, applying Minkowski's inequality combined with Khintchine inequality \cite[Theorem 2.7.5]{vershynin:2018}, we get
\begin{align}
    \mathcal{T}_1 \leq \PE\biggl[ \biggl| \frac{tn}{m} \sum_{j=1}^{d} \bigl(\sum_{i=n+1}^{n+m} \eps_{ij}\bigr)^2 \biggr|^{p} \biggr] = \PE\biggl[ \frac{(tn)^p}{m^p}\PE^{\F_n}\biggl[ \biggl|  \sum_{j=1}^{d} \bigl(\sum_{i=n+1}^{n+m} \eps_{ij}\bigr)^2 \biggr|^{p} \biggr] \biggr] &\leq (d t n)^p \PE\biggl[ \frac{1}{m^p} \PE^{\F_n} \biggl[ \biggl| \sum_{i=n+1}^{n+m} \eps_{i1} \biggr|^{2p} \biggr] \biggr] \lesssim (2tnpd)^{p} \eqsp.
\end{align}

\paragraph{Bound for $\mathcal{T}_2$.}  First, we use \eqref{eq:lem:mart_limits:proba_assum} and get
\begin{align}
    \P[\Omega_2] \leq C_1 \exp\bigl(-C_2 nt^2 \bigr) \eqsp .
\end{align}
Note that
\begin{align}
    \norm{S_n - S_N'} \leq 2N\kappa \leq  2nd \frac{2\kappa^2+\norm{\Sigma_n} + t} {\kappa} \, .
\end{align}
Thus, we obtain that
\begin{align}
    \mathcal{T}_2 \leq 2^{2p}d^{2p} n^{2p} \biggl(\frac{2\kappa^2+\norm{\Sigma_n} + t} {\kappa}\biggr)^{2p} \P[\Omega_2] \leq C_1 2^{2p} d^{2p}  \biggl(\frac{2\kappa^2+\norm{\Sigma_n} + t} {\kappa}\biggr)^{2p} n^{2p} \exp\bigl(-C_2 nt^2\bigr) \eqsp.
\end{align}
Choose $t = (\frac{2p\log n}{nC_2})^{1/2}$. Thus,
\begin{align}
    \mathcal{T}_2 \leq C_1  (2d)^{2p} (2p)^{p} \biggl(\frac{2\kappa^2+\norm{\Sigma_n} + (\frac{\log n}{nC_2})^{1/2}} {\kappa}\biggr)^{2p}   \eqsp,
\end{align}
and
\begin{align}
    \mathcal{T}_1 \lesssim (8C_2^{-1}p^3 d^2 n\log n)^{p/2} \eqsp.
\end{align}
Now we substitute the latter inequalities into \eqref{lem:convex_clt_Smoments} and get applying Minkowski's inequality:
\begin{align}
    \PE^{\frac{1}{2p+1}}[\|S_n-S_N'\|^{2p}] \lesssim p^{3/4}d \biggl( \biggl(\frac{n \log n}{C_2}\biggr)^{\frac{p}{2(2p+1)}}  + C_1^{\frac{1}{2p+1}} \biggl(\frac{2\kappa^2+\norm{\Sigma_n} + (\frac{\log n}{nC_2})^{1/2}} {\kappa}\biggr)^{\frac{2p}{2p+1}} \biggr) \eqsp,
\end{align}
and the proof follows.
\end{proof}
\section{High-order bounds on the error moments}
\label{appendix:high_order}
We follow the decoupling idea of \cite{konda:tsitsiklis:2004} and perform the change of variables in the recurrence \eqref{eq:2ts1}-\eqref{eq:2ts2}, which is similar to the Gaussian elimination. Using \Cref{prop:kaledine_decoupling}, we obtain, with $\ttheta_k$ and $\tw_k$ defined in \eqref{eq:tilde_params}, that the two-timescale SA \eqref{eq:2ts1}-\eqref{eq:2ts2} reduces to the system of updates:
\begin{equation}
\label{eq:2ts2-1-appendix}
\begin{split}
\begin{cases}
\ttheta_{k+1} &= ( \Id - \beta_k B_{11}^k ) \ttheta_k - \beta_k A_{12} \Tilde{w}_k - \beta_k V_{k+1}\eqsp, \\
\tilde{w}_{k+1} &= (\Id - \gamma_k B_{22}^k ) \Tilde{w}_k - \beta_k D_k V_{k+1} - \gamma_k W_{k+1}\eqsp.
\end{cases}
\end{split}
\end{equation}
Recall that the sequence of matrices $D_k$ has a form
\[
D_k = L_{k+1} + A_{22}^{-1} A_{21}\eqsp,
\]
where $L_k$ are defined in \eqref{eq:L_k_matr_def}. The following proposition shows that norms of matrices $D_k$ are bounded. Moreover, $L_k$ converges to $0$ under \Cref{assum:stepsize}($2$). This result is due to \cite{kaledin2020finite}.
\begin{lemma}[Lemma~19 in \cite{kaledin2020finite}]
\label{L_converges} 
Assume \Cref{assum:hurwitz} and \Cref{assum:stepsize}($2$). Then for any $k \in \nset$, 
\begin{align}
\|L_k\| &\leq \ell_{\infty}  \frac{\beta_k}{\gamma_k}\eqsp, \quad \|D_k\| \leq c_{\infty}\eqsp,
\end{align}
 where the value of the constant \(\ell_{\infty}\)
can be found in \cite{kaledin2020finite} and $c_{\infty}$ has form 
\begin{equation}
\label{eq:l_infty_def}
\begin{split}
c_{\infty} = \ell_{\infty} \rstep + \normop{A_{22}^{-1} A_{21}} \eqsp,\quad \rstep = c_{0,\beta}/c_{0,\gamma}\eqsp. 
\end{split}
\end{equation}
\end{lemma}
Let us note the important properties of our steps. Since  \(a < b\eqsp\), the ratio \(\beta_i/\gamma_i\) decreases as \(i\) increases, hence \(\beta_i/\gamma_i \leq \beta_0/\gamma_0\) for all $i \in \mathbb{N}$. Moreover, \(k_0^{a-b} < 1\), therefore \(\beta_0/\gamma_0 \leq c_{0,\beta}/c_{0,\gamma} = \rstep\).
To proceed with the $p$-th moment bounds for $\tilde{w}_{k+1}$ and $\tilde{\theta}_{k+1}$, we introduce the random vectors
\begin{equation}
\label{eq:xi_k_def}
\xi_{k+1} = \gamma_k W_{k+1} + \beta_k D_{k}V_{k+1}\eqsp.
\end{equation}
Our next lemma allows to bound moments of $V_{k+1},W_{k+1},\xi_{k+1}$ in terms of $M_{k,p}^{\tw}$ and $M_{k,p}^{\ttheta}$ introduced in \eqref{eq:tilde_moments}.  

\begin{lemma}
\label{lem:tilde_bounds}
Assume \Cref{assum:bound-conditional-moments-p}($p$), \Cref{assum:hurwitz}, and \Cref{assum:stepsize}($2$). Then it holds that 
\begin{equation}
\begin{split}
\PE^{1/p}[ \|V_{k+1}\|^p] &\leq \widetilde{m}_V(1 + M_{k,p}^{\ttheta} + M_{k,p}^{\tw}) \eqsp , \\
\PE^{1/p}[ \|W_{k+1}\|^p] &\leq \widetilde{m}_W (1 + M_{k,p}^{\ttheta} + M_{k,p}^{\tw}) \eqsp,  \\
\PE^{1/p}[\|\xi_{k+1}\|^p] &\leq \widetilde{m} \gamma_k\bigl(1 + M_{k,p}^{\ttheta} + M_{k,p}^{\tw}\bigr)\eqsp, 
\end{split}
\end{equation}
where we have defined 
\begin{equation}
\label{eq:widetilde_m_V_W_def}
\widetilde{m}_V = m_V (1 + c_{\infty})\eqsp, \quad \widetilde{m}_W = m_W (1 + c_{\infty})\eqsp, \quad \widetilde{m} = \rstep \widetilde{m}_{V} c_{\infty} + \widetilde{m}_W\eqsp.
\end{equation}
and $\rstep$ is defined in \eqref{eq:l_infty_def}.
\end{lemma}
\begin{proof}
Since \(w_k-w^* = \tw_k - D_{k-1}\ttheta_k\), we get applying \Cref{L_converges}:  
\begin{align}
\PE^{1/p}[\norm{w_k - \wstar}^p] \leq \PE^{1/p}[\|\tw_k\|^p] + c_{\infty}\PE^{1/p}[\|\ttheta_k\|]\eqsp.
\end{align} 
Combining the above bound with \Cref{assum:bound-conditional-moments-p}($p$), we obtain 
\begin{align}
\PE^{1/p}[ \|V_{k+1}\|^p] \leq  m_V ( 1 +  \PE^{1/p}[\|\theta_{k} - \thetas\|^p] +  \PE^{1/p}[\|w_{k} - \wstar\|^{p}]) \leq \widetilde{m}_V(1 + M_{k,p}^{\ttheta} + M_{k,p}^{\tw})\eqsp, 
\end{align}
and
\begin{align}
\PE^{1/p}[ \|W_{k+1}\|^p] \leq  \widetilde{m}_W(1 + M_{k,p}^{\ttheta} + M_{k,p}^{\tw})\eqsp. 
\end{align}
Similarly, 
\begin{align}
    \PE^{1/p}[\|\xi_{k+1}\|^p]
    &\leq \gamma_k\PE^{1/p}[\|W_{k+1}\|^p] + \beta_k c_{\infty}\PE^{1/p}[\|V_{k+1}\|^p] \\
    &\leq \gamma_k\widetilde{m}_W\big(1 + M_{k,p}^{\ttheta} + M_{k,p}^{\tw}\big) +\beta_k\widetilde{m}_{V} c_{\infty} \bigl(1 + M_{k,p}^{\ttheta} + M_{k,p}^{\tw}\bigr)\\
    &\leq \widetilde{m}\gamma_k\bigl(1 + M_{k,p}^{\ttheta} + M_{k,p}^{\tw}\bigr),
\end{align}
where $\widetilde{m}$ is defined in \eqref{eq:widetilde_m_V_W_def}.
\end{proof}

\subsection{Bounding the products of deterministic matrices}
Now we state and prove the results regarding the stability of matrix products. The key element of the proof is the Hurwitz stability assumption \Cref{assum:hurwitz}. Below we state and prove the Lyapunov stability lemma:
\begin{lemma}
\label{contraction_pr op}
\label{prop:hurwitz_stability}
Let $-A$ be a Hurwitz matrix. Then there exists a unique matrix $Q = Q^{\top} \succ 0$, satisfying the Lyapunov equation $A^\top Q + Q A = \Id$. Moreover, setting
\begin{equation}
\textstyle 
a = \frac{1}{2\normop{Q}}\eqsp, \quad
\text{and} \quad \alpha_\infty = \frac{1}{2\|Q\|\normop{A}[Q]^{2}} \eqsp,
\end{equation}
it holds for any $\alpha \in [0, \alpha_{\infty}]$ that
\begin{equation}
\label{eq:contractin_q_norm}
\normop{\Id - \alpha A}[Q]^2 \leq 1 - \alpha a\eqsp.   
\end{equation}
\end{lemma}
\begin{proof}
    The fact that there exists a unique matrix $Q$, such that the following Lyapunov equation holds:
\begin{equation}
\label{eq:Lyapunov_equation}
A^\top Q + Q A = \Id\eqsp,
\end{equation}
follows directly from \cite[Lemma~$9.1$, p. 140]{poznyak:control}. In order to show the second part of the statement, we note that for any non-zero vector $x \in \rset^{d}$, we have
\begin{align}
\frac{x^{\top}(\Id - \alpha A)^{\top}Q(\Id - \alpha A)x}{x^{\top} Q x} 
&= 1 - \alpha \frac{x^{\top}(A^{\top}Q + QA)x}{x^{\top} Q x} + \alpha^2 \frac{x^{\top} A^{\top} Q A x}{x^{\top} Q x} \\
&= 1 - \alpha \frac{x^{\top} x}{x^{\top} Q x} + \alpha^2\, \frac{x^{\top} A^{\top} Q A x}{x^{\top} Q x} \\
&\leq 1 - \frac{\alpha}{\normop{Q}} + \alpha^2\, \norm{A}[Q]^2 \\
&\leq 1 - \alpha a\eqsp,
\end{align}
where used the fact that $\alpha \leq \alpha_{\infty}$.
\end{proof}

Note that \Cref{prop:hurwitz_stability} implies the existence of matrices $Q_{22}$ and $Q_\Delta$, such that 
\begin{equation}
\label{eq:definition-Q-22}
A_{22}^\top Q_{22} + Q_{22} A_{22}= \Id, \quad 
Q_\Delta \Delta + \Delta^\top Q_\Delta = \Id\eqsp.
\end{equation}
This ensures the contraction in the respective matrix $Q$-norm: provided that $\gamma_k \in [0, 1 / (2 \|A_{22}\|_{Q_{22}}^2 \| Q_{22} \|)], \beta_k \in [0, 1 / (2 \| A_\Delta \|_{Q_\Delta}^2 \| Q_\Delta \| )]$, it holds, that
\begin{equation}
\label{eq:contraction_p_appendix}
\begin{split}
&\normop{\Id - \gamma_k A_{22}}[Q_{22}] \leq 1 - a_{22} \gamma_k, \quad a_{22} := \tfrac{1}{2 \| Q_{22}\|}\eqsp, \\
&\normop{\Id - \beta_k \Delta}[Q_\Delta] \leq 1 - a_\Delta \beta_k, \quad  a_\Delta := \tfrac{1}{2 \| Q_\Delta \|}\eqsp. 
\end{split}
\end{equation}
We now define a few constants related to the matrices $Q_\Delta, Q_{22}$:
\[
\myqcond_{\Delta} :=\frac{\lambda_{\sf max}( Q_\Delta )}{\lambda_{\sf min}( Q_\Delta )}  , \quad \myqcond_{22} :=  \frac{ \lambda_{\sf max}( Q_{22})}{\lambda_{\sf min}( Q_{22} )}.
\]
Next we show that the factors $\Id - \beta_k B_{11}^k$ and $\Id - \gamma_k B_{22}^k$ in the transformed recursion \eqref{eq:2ts2-1} are also contractive in the same matrix norms induced by $Q_\Delta$ and $Q_{22}$, respectively. 
\begin{lemma}
\label{lem:prod_determ_matr}
Assume \Cref{assum:hurwitz} and \Cref{assum:stepsize}($2$). Then it holds that 
\begin{equation}
\label{eq:contraction-B11}
\normop{\Id - \beta_k B_{11}^k}[Q_\Delta] \leq 1 - (1/2) \beta_k a_\Delta \eqsp, \quad \normop{\Id - \gamma_k B_{22}^k}[Q_{22}] \leq 1 - (1/2) \gamma_k a_{22} \eqsp.
\end{equation}
\end{lemma}
\begin{proof}
Using \eqref{eq:contraction_p}, we observe that 
\begin{align} 
\normop{\Id - \beta_k B_{11}^k}[Q_\Delta] &= \normop{\Id - \beta_k \Delta + \beta_k A_{12} L_k}[Q_\Delta] \leq \normop{\Id - \beta_k \Delta}[Q_\Delta] + \beta_k \normop{A_{12}L_k}[Q_{\Delta}]  \nonumber \\
&\leq (1-\beta_k a_\Delta) + \beta_k \sqrt{\myqcond_{\Delta}}\normop{A_{12}} \normop{L_k} \leq (1-\beta_k a_\Delta) + \beta_k \sqrt{\myqcond_{\Delta}}\normop{A_{12}} \rstep \ell_{\infty}
\end{align}
Using $\rstep \leq a_{\Delta}/(2\|A_{12}\|\sqrt{\myqcond_{\Delta}}\ell_{\infty})$, the above inequality yields the first part of \eqref{eq:contraction-B11}. Similarly, using \eqref{eq:contraction_p}, we get that 
\begin{align}
    \|\Id - \gamma_k B_{22}^{k}\|_{Q_{22}} = \|\Id - \gamma_kA_{22} - \beta_k D_k A_{12}\|_{Q_{22}} 
    &\leq \|\Id - \gamma_kA_{22}\|_{Q_{22}} + \rstep \gamma_k\sqrt{\myqcond_{22}} c_{\infty}\|A_{12}\| 
    \\ &\leq (1 - \gamma_k a_{22}) + \rstep \gamma_k\sqrt{\myqcond_{22}} c_{\infty}\|A_{12}\|.
\end{align}
Recalling that $\rstep \leq a_{22}/(2\|A_{12}\|\sqrt{\myqcond_{22}} c_{\infty})$, the second part of \eqref{eq:contraction-B11} follows.
\end{proof}

Throughout our analysis we use the following notations:
\begin{equation}
\label{eq:matr_product_determ}
\begin{split}
&\ProdB_{m:n}^{(1)} := \prod_{i=m}^n (\Id - \beta_i B_{11}^i ), \quad \ProdB_{m:n}^{(2)} := \prod_{i=m}^n (\Id - \gamma_i B_{22}^i ), \\
& P_{m:n}^{(1)} := \prod_{i=m}^n (1 - (1/2) \beta_i a_{\Delta} ), \quad P_{m:n}^{(2)} := \prod_{i=m}^n (1 - (1/2) \gamma_i a_{22}).
\end{split}
\end{equation}
As a convention, we define $\ProdB_{m:n}^{(1)} = \Id $ and $\ProdB_{m:n}^{(2)} = \Id$ if $m > n$. 

\begin{corollary}
\label{cor:determinist_matrix_products}
Under the assumptions of \Cref{lem:prod_determ_matr}, it holds for any $n,m \geq 0$, that 
\begin{equation}
\label{pth_contraction}
\| \ProdB_{m:n}^{(1)} \| 
\leq \sqrt{\pth} P_{m:n}^{(1)}\eqsp.
\end{equation}
Similarly, we have 
\begin{equation}\label{pw_contraction}
    \|\ProdB_{m:n}^{(2)} \| \leq \sqrt{\pw} P_{m:n}^{(2)}\eqsp.
\end{equation}
\end{corollary}
We shall prove that
\begin{align}
\label{eq:mth_moment_bound_appendix}
M_{k+1,p}^{\ttheta} &\leq \ConstC_0^{\ttheta} \prod_{j=0}^{k}\big(1 - \beta_ja_{\Delta}/8\big) + \ConstC_{\operatorname{slow}} p^2 \beta_{k}^{1/2}\eqsp, \\
\label{eq:mtw_bound_appendix}
M_{k+1,p}^{\tw} &\leq  \ConstC_0^{\tw} \prod_{j=0}^{k}\big(1-\gamma_ja_{22}/8\big) + \ConstC_{\operatorname{fast}} p^3 \gamma_{k}^{1/2}\eqsp,
\end{align}
First, we introduce the constants
\begin{align}
\ConstC_0^{\ttheta} = \{C_0^{\ttheta}\}^{1/2} \eqsp, \quad \ConstC_0^{\tw} = \{C_0^{\tw}\}^{1/2} \eqsp, \quad
\ConstC_{\operatorname{slow}} = \{24 C_1^{\ttheta}\}^{1/2}\eqsp, \quad \ConstC_{\operatorname{fast}} = \{C_1^{\tw} + 24 a_{22}^{-1}C_2^{\tw} (2 C_0^{\ttheta} + 2\ConstC_{\operatorname{slow}}^2\beta_0)\}^{1/2}\eqsp. 
\end{align}
where $C_0^{\tw}$, $C_1^{\tw}$, $C_2^{\tw}$ and $C_0^{\ttheta}$, $C_1^{\ttheta}$ are defined in \eqref{eq:C_tw_defenition} and \eqref{eq:Cttheta_definition} respectively. In order to prove \Cref{prop:moments_bound}, we employ the following scheme. We first consider the moments "fast" scale $M_{k+1,p}^{\tw}$ and upper bound them in terms of the moments of "slow" time scale $M_{j,p}^{\ttheta}$ with $j \in \{1,\ldots,p\}$. This is formalized in the following proposition:
\begin{proposition}
\label{prop:tw_bond}
Let $p \geq 2$ and assume \Cref{assum:zero-mean},\Cref{assum:bound-conditional-moments-p}($p$), \Cref{assum:hurwitz}, \Cref{assum:stepsize}($p$). Then for any \(k\in\mathbb{N}\) it holds that
\begin{equation}
\label{eq:tw_req_ineq}
\bigl(M_{k+1,p}^{\tw}\bigr)^2\leq  C_0^{\tw}P_{0:k}^{(2)} + p^2C_1^{\tw}\gamma_{k} + p^2C_2^{\tw}\sum_{j=0}^{k}\gamma_j^2P_{j+1:k}^{(2)}\bigl(M_{j,p}^{\ttheta}\bigr)^2,
\end{equation}
where the constants $C_0^{\tw}, C_1^{\tw}, C_2^{\tw}$ are given in \eqref{eq:C_tw_defenition}. 
\end{proposition}

Now we derive the following recursive bounds for the moments of "slow" time scale $M_{k,p}^{\ttheta}$:
\begin{proposition}
\label{prop:mth_closed_bound}
Let $p \geq 2$ and assume \Cref{assum:zero-mean},\Cref{assum:bound-conditional-moments-p}($p$), \Cref{assum:hurwitz}, \Cref{assum:stepsize}($p$). Then for any \(k\in\mathbb{N}\) it holds that
\begin{equation}
\label{eq:mth_bound}
\bigl(M_{k+1,p}^{\ttheta}\bigr)^2\leq C_0^{\ttheta} P_{0:k}^{(1)} + p^4C_1^{\ttheta}\beta_{k} + p^4C_2^{\ttheta}\sum_{j=0}^{k}\beta_j^2P_{j+1:k}^{(1)}\bigl(M_{j,p}^{\ttheta}\bigr)^2,
\end{equation}
where the constants $C_0^{\ttheta}, C_1^{\ttheta}, C_2^{\ttheta}$ are given in \eqref{eq:Cttheta_definition}.
\end{proposition}

\begin{proof}[Proof of \Cref{prop:moments_bound}]
~ \\
\textbf{(I) Proof of the bound \eqref{eq:mth_moment_bound}.} Now we aim to solve the recurrence \eqref{eq:mth_bound} and prove the upper bound \eqref{eq:mth_moment_bound}. Towards this, we consider the recurrence $\widetilde{U}_k$, which is driven by the right-hand side of \eqref{eq:mth_bound}:
\begin{equation}\label{eq:tilde_U_k_recurrence}
\widetilde{U}_{k+1} =  C_0^{\ttheta} P_{0:k}^{(1)} + p^4C_1^{\ttheta}\beta_{k} + p^4C_2^{\ttheta}\sum_{j=0}^{k}\beta_j^2P_{j+1:k}^{(1)}\widetilde{U}_j, \quad \widetilde{U}_0 = C_0^{\ttheta}\eqsp,
\end{equation}
and $C_0^{\ttheta}$ is defined in \eqref{eq:Cttheta_definition}. 
The constructed sequence \(\widetilde{U}_k\) provides an upper bound for the moments, that is, 
\begin{equation}\label{eq:u_tilde_mth_induction}
(M_{k+1,p}^{\ttheta})^2\leq \widetilde{U}_{k+1}\eqsp.
\end{equation}
To verify \eqref{eq:u_tilde_mth_induction}, observe that \((M_{0,p}^{\ttheta})^2\leq \widetilde{U}_{0}\) by definition of $\widetilde{U}_0$ in \eqref{eq:tilde_U_k_recurrence}. By induction, assuming the validity of \eqref{eq:u_tilde_mth_induction} for all \(j\leq k\), we establish its correctness for \(k+1\). Indeed, using \Cref{prop:mth_closed_bound}, we get
\begin{align}
    \bigl(M_{k+1,p}^{\ttheta}\bigr)^2 &\leq  C_0^{\ttheta}P_{0:k}^{(1)} + p^4 C_1^{\ttheta}\beta_k +p^4 C_2^{\ttheta}\sum_{j=0}^{k}\beta_j^2 P_{j+1:k}^{(1)} \big(M_{j,p}^{\ttheta}\big)^2 \\
    &\leq C_0^{\ttheta} P_{0:k}^{(1)} + p^4C_1^{\ttheta}\beta_{k} + p^4C_2^{\ttheta}\sum_{j=0}^{k}\beta_j^2P_{j+1:k}^{(1)}\widetilde{U}_k = \widetilde{U}_{k+1} \eqsp.
\end{align}
Using the definition of the product \(P_{j+1:k}^{(1)}\) in \eqref{eq:matr_product_determ}, we observe that 
\[
\widetilde{U}_{k+1} = (1 - \beta_ka_{\Delta}/2 + p^4C_2^{\ttheta}\beta_k^2)\widetilde{U}_k + p^4C_1^{\ttheta}(\beta_{k} - \beta_{k-1} + \beta_{k-1}\beta_ka_{\Delta}/2)\eqsp.
\]
Since $\beta_k \leq a_{\Delta}/(4p^4C_{2}^{\ttheta})$ due to \Cref{assum:stepsize}, and \(\beta_{k-1}\leq 2\beta_k\), we have 
\begin{equation}
    \widetilde{U}_{k+1} \leq (1 - \beta_ka_{\Delta}/4)\widetilde{U}_k + p^4C_1^{\ttheta}a_{\Delta}\beta_k^2\eqsp.
\end{equation}
Enrolling the above recurrence, we get 
\begin{equation}
\widetilde{U}_{k+1} \leq C_0^{\ttheta}\prod_{j=0}^{k}\bigl(1 - \beta_j\frac{a_{\Delta}}{4}\bigr) + p^4C_1^{\ttheta}a_{\Delta}\sum_{j=0}^{k}\beta_j^2\prod_{i=j+1}^{k}\bigl(1 - \beta_i\frac{a_{\Delta}}{4}\bigr)\eqsp.
\end{equation}
Applying \Cref{lem:summ_alpha_k}-\ref{lem:summ_alpha_k_p_item}, the bound \eqref{eq:u_tilde_mth_induction}, and the inequality $\sqrt{a+b}\leq\sqrt{a} + \sqrt{b}$, we get \eqref{eq:mth_moment_bound}.
\\
\textbf{(II) Proof of the bound \eqref{eq:mtw_bound}.} Substituting \eqref{eq:mth_moment_bound} into \Cref{prop:tw_bond} we obtain that
    \begin{align}
    \bigl(M_{k+1,p}^{\tw}\bigr)^2 
    &\leq  C_0^{\tw}P_{0:k}^{(2)} + p^2C_1^{\tw}\gamma_{k} + p^2C_2^{\tw}\sum_{j=0}^{k}\gamma_j^2P_{j+1:k}^{(2)}\bigl(M_{j,p}^{\ttheta}\bigr)^2 \\
    &\leq  C_0^{\tw}P_{0:k}^{(2)} + p^2C_1^{\tw}\gamma_{k} + p^2C_2^{\tw}\sum_{j=0}^{k}\gamma_j^2P_{j+1:k}^{(2)}\big(2C_0^{\ttheta}\prod_{i=0}^{j-1}\big(1 - \beta_i a_{\Delta}/4\big) + 2p^4\ConstC_{\operatorname{slow}}^2\beta_{j}\big) \\
    &\leq C_0^{\tw}P_{0:k}^{(2)} + p^2C_1^{\tw}\gamma_{k} + p^2 C_2^{\tw} (2 C_0^{\ttheta} + 2p^4\ConstC_{\operatorname{slow}}^2\beta_0)24\gamma_k a_{22}^{-1} \\ 
    &\leq C_0^{\tw}P_{0:k}^{(2)} + p^6\gamma_k \underbrace{(C_1^{\tw} + 24 a_{22}^{-1}C_2^{\tw} (2 C_0^{\ttheta} + 2\ConstC_{\operatorname{slow}}^2\beta_0))}_{\ConstC_{\operatorname{fast}}^2}\eqsp.
\end{align}
The inequality \(\sqrt{a + b} \leq \sqrt{a} + \sqrt{b}\) completes the proof. 
\end{proof}

Now we derive a moment bound for $w_{k+1 } - w^\star$ using \Cref{prop:moments_bound}.
\begin{lemma}
\label{lem:martingale:w_minus_wstar_moment_bound}
Let $p \geq 2$. Assume \Cref{assum:zero-mean},\Cref{assum:bound-conditional-moments-p}($p$), \Cref{assum:quadratic_characteristic}, \Cref{assum:hurwitz}, \Cref{assum:stepsize}($p$), and \Cref{assum:aij_bound}. Then it holds for $k \in \nset$ that
\begin{align}
    \PE^{1/p}[\normop{w_{k+1} - \wstar}^p] \leq \widehat{\ConstC}_0^{\tw} \prod_{j=0}^k (1 - \beta_j \frac{a_\Delta}{8}) + p^3 \widehat{\ConstC}_{\operatorname{fast}} \gamma_k^{1/2} \eqsp,
\end{align}
where we have set
\begin{align}
    \widehat{\ConstC}_0^{\tw} = c_\infty \ConstC_0^{\ttheta} + \ConstC_0^{\tw} \eqsp, \eqsp \widehat{\ConstC}_{\operatorname{fast}} = \ConstC_{\operatorname{fast}} + \rstep^{1/2} c_\infty \ConstC_{\operatorname{slow}} \eqsp.
\end{align}
\end{lemma}
\begin{proof}
    Recall that \(w_{k+1}-w^\star = \tw_{k+1} - D_{k}(\theta_{k+1} - \theta^\star)\). Since \Cref{assum:stepsize} guarantees that \(\beta_k/\gamma_k\leq 2a_{22}/a_{\Delta}\), we have  \((1-\gamma_ka_{22}/8)\leq (1 - \beta_ka_{\Delta}/8).\) Hence, applying \Cref{L_converges}, \Cref{prop:moments_bound} together with Minkowski's inequality we get
\begin{align}
    \PE^{1/p}[\|w_{k+1}-w^*\|^p] &\leq M_{k+1,p}^{\tw} + c_{\infty} M_{k+1,p}^{\ttheta} \leq (c_\infty \ConstC_0^{\ttheta} + \ConstC_0^{\tw}) \prod_{j=0}^k (1 - \beta_j \frac{a_\Delta}{8}) + p^3 \gamma_k^{1/2} (c_\infty \ConstC_{\operatorname{slow}} \rstep^{1/2} + \ConstC_{\operatorname{fast}})  \eqsp,
\end{align}
and the proof is complete.
\end{proof}

Now we prove \Cref{prop:tw_bond} and \Cref{prop:mth_closed_bound}.

\begin{proof}[Proof of \Cref{prop:tw_bond}]
We first introduce the constants:
\begin{equation}
\label{eq:C_tw_defenition}
C_0^{\tw} = 2 \pw\|\tw_0\|^2 \eqsp, \quad C_1^{\tw} = \frac{72}{a_{22}} \pw \widetilde{m}^2\eqsp, \quad C_2^{\tw} = 6 \pw \widetilde{m}^2 \eqsp.
\end{equation}
The recursion \eqref{eq:2ts2-1}
 implies that 
\begin{align}
\label{eq:w_rec}
\Tilde{w}_{k+1} = \prod_{j=0}^{k}(\Id - \gamma_{j}B_{22}^{j})\Tilde{w}_0 - \sum_{j=0}^{k}\prod_{i = j+1}^{k}(\Id - \gamma_i B_{22}^{i}) (\gamma_j W_{j+1} + \beta_j D_j V_{j+1}) = \Gamma_{0:k}^{(2)}\tw_0 - \sum_{j=0}^{k}\Gamma_{j+1:k}^{(2)}\xi_{j+1}\eqsp,
\end{align}
where $\Gamma_{j+1:k}^{(2)}$ is defined in \eqref{eq:matr_product_determ}. 
Using Minkowski inequality and   \eqref{pw_contraction}, we get
\begin{equation}
\bigl(M_{k+1,p}^{\tw}\bigr)^2 = \PE^{2/p}[\|\tw_{k+1}\|^p]\leq 2 \pw \bigl(P_{0:k}^{(2)}\bigr)^2\|\tw_0\|^2 + 2\PE^{2/p}\bigl[ \norm{\sum_{j=0}^{k}\Gamma_{j+1:k}^{(2)}\xi_{j+1}}^p\bigr] .
\end{equation}
Now we proceed with the second term. Applying Burholder’s inequality \cite[Theorem 8.6]{osekowski},  we obtain that
\begin{align}
\label{eq:xi_bound_for_w_moment}
\PE^{2/p}\bigl[\bigl\|\sum_{j=0}^{k}\Gamma_{j+1:k}^{(2)}\xi_{j+1}\bigr\|^p\bigr] 
\leq p^2\mathbb{E}^{2/p}\bigl[\bigl(\sum_{j=0}^{k} \norm{\Gamma_{j+1:k}^{(2)}\xi_{j+1}}^2\bigr)^{p/2}\bigr] &\leq  p^2\sum_{j=0}^{k}\|\ \Gamma_{j+1:k}^{(2)}\|^2\mathbb{E}^{2/p}\bigl[\|\xi_{j+1}\|^p\bigr]  \\
&\leq 3p^2\widetilde{m}^2 \pw \sum_{j=0}^{k}\gamma_j^2\bigl(P_{j+1:k}^{(2)}\bigr)^2\bigl(1 + \bigl(M_{j,p}^{\ttheta}\bigr)^2 + \bigl(M_{j,p}^{\tw}\bigr)^2\bigr)\eqsp,
\end{align}
where in the last step we have used \Cref{lem:tilde_bounds} and \eqref{pw_contraction}. Finally, we get
\begin{equation}
\label{eq:mth_reccurence_bound}
\bigl(M_{k+1,p}^{\tw}\bigr)^2\leq  C_0^{\tw'}\bigl(P_{0:k}^{(2)}\bigr)^2 + p^2C_1^{\tw'}\sum_{j=0}^{k}\gamma_j^2\bigl(P_{j+1:k}^{(2)}\bigr)^2\bigl(1 + \bigl(M_{j,p}^{\ttheta}\bigr)^2 + \bigl(M_{j,p}^{\tw}\bigr)^2\bigr)\eqsp, 
\end{equation}
where \(
    C_0^{\tw'} = 2 \pw \|\tw_0\|^2, \quad C_1^{\tw'} = 6 \pw \widetilde{m}^2\). Define the sequence $U_{k}$ by the following recurrence:
\begin{equation}
\label{eq:U_k_recurrence}
U_{k+1} = C_0^{\tw'}\bigl(P_{0:k}^{(2)}\bigr)^2 + p^2C_1^{\tw'}\sum_{j=0}^{k}\gamma_j^2\bigl(P_{j+1:k}^{(2)}\bigr)^2\bigl(1 + \bigl(M_{j,p}^{\ttheta}\bigr)^2 + U_{j}\bigr)\eqsp, \quad U_{0} = C_0^{\tw'}\eqsp. 
\end{equation}
The constructed sequence \(U_{k+1}\) provides an upper bound for the moments, that is, 
\begin{equation}
\label{eq:mth_induction}
(M_{k+1,p}^{\tw})^2\leq U_{k+1}\eqsp.
\end{equation}
To verify \eqref{eq:mth_induction}, observe that \((M_{0,p}^{\tw})^2\leq U_{0}\) by definition of $U_0$ in \eqref{eq:U_k_recurrence}. By induction, assuming the validity of \eqref{eq:mth_induction} for all \(j\leq k\), we establish its correctness for \(k+1\). Indeed, using \eqref{eq:mth_reccurence_bound}, we get
\begin{align}
    \bigl(M_{k+1,p}^{\tw}\bigr)^2 
    &\leq C_0^{\tw'}\bigl(P_{0:k}^{(2)}\bigr)^2 + p^2C_1^{\tw'}\sum_{j=0}^{k}\gamma_j^2\bigl(P_{j+1:k}^{(2)}\bigr)^2\bigl(1 + \bigl(M_{j,p}^{\ttheta}\bigr)^2 + U_{j}\bigr) = U_{k+1} \eqsp.
\end{align}
Using the definition of the product $P_{j+1:k}^{(2)}$, we observe that the sequence \(\bigl(U_{k}\bigr)_{k\geq 0}\) satisfies the following recursion:
\begin{equation}
    U_{k+1} = (1 - a_{22}\gamma_k/2)^2U_{k} + p^2C_{1}^{\tw'}\gamma_k^2\bigl(1 + \bigl(M_{k,p}^{\ttheta}\bigr)^2 + U_{k}\bigr), \quad U_0 = C_0^{\tw'}.
\end{equation}
Since \(\gamma_k\leq a_{22}/(2p^2C_1^{\tw'} + a_{22}^2/2)\), we have 
\begin{equation}
    U_{k+1} \leq (1 - a_{22}\gamma_k/2)U_k + p^2C_1^{\tw'}\gamma_k^2\bigl(1 + \bigl(M_{k,p}^{\ttheta}\bigr)^2\bigr)
\end{equation}
which implies 
\begin{equation}
    U_{k+1} \leq C_0^{\tw'}P_{0:k}^{(2)} + p^2C_1^{\tw'}\sum_{j=0}^{k}\gamma_j^2 P_{j+1:k}^{(2)}\bigl(1 + \bigl(M_{j,p}^{\ttheta}\bigr)^2\bigr)\eqsp.
\end{equation}
Applying \Cref{lem:summ_alpha_k}-\ref{lem:summ_alpha_k_p_item} we get \eqref{eq:tw_req_ineq}.
\end{proof}

\begin{proof}[Proof of \Cref{prop:mth_closed_bound}]
First we introduce the constants
\begin{align}
\label{eq:Cttheta_definition}
C_0^{\ttheta} &=  4\pth \|\ttheta_0\|^2  +  4 \pth\pw\|A_{12}\|^2\rstep^2(\ConstC_\gamma^P)^2\|\tw_0\|^2 \eqsp , \\
C_1^{\ttheta} & = 144(a_\Delta)^{-1}\bigl(\widetilde{m}_{V}^2\pth + (\ConstC_\gamma^P)^2 \pth\pw\widetilde{m}^2\|A_{12}\|^2\bigr)\bigl( 1 + C_0^{\tw} + \gamma_0 C_1^{\tw}\bigr) \eqsp,  \\
C_2^{\ttheta} &= 12\bigl(\widetilde{m}_{V}^2\pth + (\ConstC_\gamma^P)^2 \pth\pw\widetilde{m}^2\|A_{12}\|^2\bigr)\bigl(C_2^{\tw}\ConstC_\gamma^P\gamma_0 + 1\bigr) \eqsp .
\end{align}
Expanding the recursion \eqref{eq:2ts2-1}, we get with $\Gamma_{j+1:k}^{(1)}$ defined in \eqref{eq:matr_product_determ}, that 
\begin{equation}
\label{eq:theta_rec}
\widetilde{\theta}_{k+1} =  \Gamma_{0:k}^{(1)}\ttheta_0 - \sum_{j=0}^{k}\beta_j\Gamma_{j+1:k}^{(1)}A_{12}\tw_j - \sum_{j=0}^{k}\beta_j\Gamma_{j+1:k}^{(1)}V_{j+1}\eqsp.
\end{equation}
Next, we substitute $\widetilde{w}_j$ from \eqref{eq:w_rec}:
\begin{align}
\sum_{j=0}^{k}\beta_j\Gamma_{j+1:k}^{(1)}A_{12}\widetilde{w}_j &=  \sum_{j=0}^{k}\beta_j\Gamma_{j+1:k}^{(1)}A_{12}\bigl(\Gamma_{0:j-1}^{(2)}\widetilde{w}_{0} - \sum_{i=0}^{j-1}\Gamma_{i+1:j-1}^{(2)}\xi_{i+1}\bigr) \\
   &=  \sum_{j=0}^{k}\beta_j\Gamma_{j+1:k}^{(1)}A_{12}\Gamma_{0:j-1}^{(2)}\widetilde{w}_0  - \sum_{i=0}^{k-1}\bigl(\sum_{j=i+1}^{k}\beta_{j}\Gamma_{j+1:k}^{(1)}A_{12}\Gamma_{i+1:j-1}^{(2)}\bigr)\xi_{i+1}\eqsp.
\end{align}
Define the quantity 
\begin{equation}
    T_{m:n} = \sum_{\ell=m}^{n}\beta_{\ell}\Gamma_{\ell+1:n}^{(1)}A_{12}\Gamma_{m:\ell-1}^{(2)}\eqsp, \eqsp m \leq n \eqsp .
\end{equation}
Thus, with $P_{k:j}^{(1)}$, $P_{k:j}^{(2)}$ defined in \eqref{eq:matr_product_determ}, it holds that 
\begin{equation}
\label{eq:t_bound}
\|T_{m:n}\| \leq \sqrt{\pth\pw}\|A_{12}\|\sum_{\ell=m}^{n}\beta_{\ell}P_{\ell+1:n}^{(1)}P_{m:\ell-1}^{(2)}\eqsp.
\end{equation}
Now we rewrite \eqref{eq:theta_rec} as follows:
\begin{equation}
    \ttheta_{k+1} = \Gamma_{0:k}^{(1)}\ttheta_0 - T_{0:k}\tw_0 + \sum_{j=0}^{k-1}T_{j+1:k}\xi_{j+1} - \sum_{j=0}^{k}\beta_j\Gamma_{j+1:k}^{(1)}V_{j+1}\eqsp.
\end{equation}
Applying Minkowski inequality and \eqref{pw_contraction},  we obtain that
\begin{align}
    \bigl(M_{k+1,p}^{\ttheta}\bigr)^2 \leq \underbrace{4\pth\bigl(P_{0:k}^{(1)}\bigr)^2\|\ttheta_0\|^2}_{\mathcal{R}_1} +  \underbrace{4\|T_{0:k}\|^2\|\tw_0\|^2}_{\mathcal{R}_2} &+ \underbrace{4\PE^{2/p}\big[\big\|\sum_{j=0}^{k}\beta_j\Gamma_{j+1:k}^{(1)}V_{j+1}\big\|^p\big]}_{\mathcal{R}_3} +\underbrace{4\PE^{2/p}\big[\big\|\sum_{j=0}^{k-1}T_{j+1:k}\xi_{j+1}\big\|^p\big]}_{\mathcal{R}_4}\eqsp.
\end{align}
Next, we get the upper bounds for $\mathcal{R}_i$ separately. Applying \Cref{convolution_inequality}  with \(j+1=0\) and using \(\beta_j\leq \rstep\gamma_j\), one can get:
\begin{align}
\mathcal{R}_2 \leq 4 \pth\pw\|A_{12}\|^2\bigl(\eqsp \sum_{j=0}^{k}\beta_jP_{j+1:k}^{(1)}P_{0:j-1}^{(2)}\bigr)^2 \|\tw_0\|^2 &\leq 4 \pth\pw\|A_{12}\|^2\rstep^2\bigl(\eqsp \sum_{j=0}^{k}\gamma_jP_{j+1:k}^{(1)}P_{0:j-1}^{(2)}\bigr)^2 \|\tw_0\|^2\\
&\leq 4 \pth\pw\|A_{12}\|^2\rstep^2(\ConstC_\gamma^P)^2\bigl(P_{0:k}^{(1)}\bigr)^2 \|\tw_0\|^2\eqsp.
\end{align}
Applying \Cref{lem:tilde_bounds} and Burholder’s inequality, we
obtain that
\begin{align}\label{eq:r_2_bound}
\mathcal{R}_3  \leq 4p^2\mathbb{E}^{2/p}\bigl[\bigl(\sum_{j=0}^{k}\beta_j^2\|\Gamma_{j+1:k}^{(1)}V_{j+1}\|^2\bigr)^{p/2}\bigr] &\leq  4p^2\sum_{j=0}^{k}\beta_j^2\|\ \Gamma_{j+1:k}^{(1)}\|^2\mathbb{E}^{2/p}\bigl[\|V_{j+1}\|^p\bigr]  \\
    &\leq 12p^2\widetilde{m}_{V}^2\pth\sum_{j=0}^{k}\beta_j^2\bigl(P_{j+1:k}^{(1)}\bigr)^2\bigl(1 + \bigl(M_{j,p}^{\ttheta}\bigr)^2 + \bigl(M_{j,p}^{\tw}\bigr)^2\bigr)\eqsp.
\end{align}
In order to derive a bound for $\mathcal{R}_4$, we apply \Cref{lem:tilde_bounds}, \eqref{eq:t_bound} and Burholder’s inequality:
\begin{align}
\mathcal{R}_4  
&\leq 4p^2\mathbb{E}^{2/p}\big[\big(\sum_{j=0}^{k-1}\|T_{j+1:k}\xi_{j+1}\|^2\big)^{p/2}\big] \leq 4p^2\sum_{j=0}^{k-1}\|T_{j+1:k}\|^2\bigl(\PE^{2/p}[\|\xi_{j+1}\|^p]\bigr) \\
    &\leq  12p^2\pth\pw\widetilde{m}^2\|A_{12}\|^2\sum_{j=0}^{k-1}\gamma_j^2\big(\sum_{i=j+1}^{k}\beta_iP_{i+1:k}^{(1)}P_{j+1:i-1}^{(2)}\big)^2\big(1 + \big(M_{j,p}^{\ttheta}\big)^2 + \big(M_{j,p}^{\tw}\big)^2\big) \\
    & \leq  12p^2\pth\pw\widetilde{m}^2\|A_{12}\|^2\sum_{j=0}^{k-1}\beta_j^2\big(\sum_{i=j+1}^{k}\gamma_j P_{i+1:k}^{(1)}P_{j+1:i-1}^{(2)}\big)^2 \big(1 + \big(M_{j,p}^{\ttheta}\big)^2 + \bigl(M_{j,p}^{\tw}\bigr)^2\big)\\
    &\overset{(a)}{\leq}  12p^2(\ConstC_\gamma^P)^2 \pth\pw\widetilde{m}^2\|A_{12}\|^2 \sum_{j=0}^{k-1}\beta_j^2\bigl(P_{j+1:k}^{(1)}\bigr)^2 \big(1 + \big(M_{j,p}^{\ttheta}\big)^2 + \big(M_{j,p}^{\tw}\big)^2\big)\eqsp,
\end{align}
where the inequality (a) follows from \Cref{convolution_inequality}. Combining the above bounds, we get
\begin{align}
\label{mth_ineq_1}
\big(M_{k+1,p}^{\ttheta}\big)^2 \leq & \eqsp \bigl( 4\pth \|\ttheta_0\|^2  +  4 \pth\pw\|A_{12}\|^2\rstep^2(\ConstC_\gamma^P)^2\|\tw_0\|^2 \bigr)\bigl(P_{0:k}^{(1)}\bigr)^2 \\
&+  12p^2\bigl(\widetilde{m}_{V}^2\pth + (\ConstC_\gamma^P)^2 \pth\pw\widetilde{m}^2\|A_{12}\|^2\bigr)\sum_{j=0}^{k-1}\beta_j^2\big(P_{j+1:k}^{(1)}\big)^2 \big(1 + \big(M_{j,p}^{\ttheta}\big)^2 + \big(M_{j,p}^{\tw}\big)^2\big)\eqsp,
\end{align}
Moreover, applying \Cref{prop:tw_bond} and \Cref{lem:summ_alpha_k}-\ref{lem:summ_alpha_k_p_item} we bound last term in \eqref{mth_ineq_1} as follows
\begin{align}
\sum_{j=0}^k\beta_j^2\bigl(P_{j+1:k}^{(1)}\bigr)^2\bigl(M_{j,p}^{\tw}\bigr)^2 &\leq \sum_{j=0}^k\beta_j^2\bigl(P_{j+1:k}^{(1)}\bigr)^2\bigl( C_0^{\tw}P_{0:j-1}^{(2)} + p^2C_1^{\tw}\gamma_{j} + p^2C_2^{\tw}\sum_{i=0}^{j-1}\gamma_i^2P_{i+1:j-1}^{(2)}\bigl(M_{i,p}^{\ttheta}\bigr)^2\bigr) \\
&\leq p^2C_0^{\tw}\sum_{j=0}^{k}\beta_j^2P_{j+1:k}^{(1)} + p^2\gamma_0 C_1^{\tw} \sum_{j=0}^{k}\beta_j^2P_{j+1:k}^{(1)}  + p^2C_2^{\tw}\sum_{j=0}^{k}\beta_j^2\bigl(P_{j+1:k}^{(1)}\bigr)^2\sum_{i=0}^{j-1}\gamma_i^2P_{i+1:j-1}^{(2)}\bigl(M_{i,p}^{\ttheta}\bigr)^2 \\
&\leq  p^2\bigl( C_0^{\tw} + \gamma_0 C_1^{\tw}\bigr)\frac{12}{a_\Delta}\beta_{k} + p^2C_2^{\tw} \sum_{i=0}^{k-1}\sum_{j=i+1}^{k}\beta_j^2\gamma_i^2P_{i+1:j-1}^{(2)}P_{j+1:k}^{(1)} \bigl(M_{i,p}^{\ttheta}\bigr)^2\eqsp.
\end{align}
Using $\beta_j^2 \leq \beta_i^2$ for $j \geq i+1$ and $\gamma_i^2 \leq \gamma_0 \gamma_i$, we get 
\begin{align}
\sum_{j=0}^k\beta_j^2\bigl(P_{j+1:k}^{(1)}\bigr)^2\bigl(M_{j,p}^{\tw}\bigr)^2 
&\leq  p^2\bigl( C_0^{\tw} + \gamma_0 C_1^{\tw}\bigr) \frac{12}{a_\Delta} \beta_{k} + p^2C_2^{\tw}\gamma_0 \sum_{i=0}^{k-1}\beta_i^2\bigl(\sum_{j = i+1}^{k}\gamma_iP_{i+1:j-1}^{(2)}P_{j+1:k}^{(1)}\bigr)\bigl(M_{i,p}^{\ttheta}\bigr)^2 \\ 
&\overset{(a)}{\leq}  p^2\bigl( C_0^{\tw} + \gamma_0 C_1^{\tw}\bigr) \frac{12}{a_\Delta} \beta_{k} + p^2C_2^{\tw}\gamma_0\ConstC_\gamma^P \sum_{i=0}^{k-1}\beta_i^2P_{i+1:k}^{(1)}\bigl(M_{i,p}^{\ttheta}\bigr)^2\eqsp \eqsp, 
\end{align}
where $(a)$ follows from \Cref{convolution_inequality}.
Substituting the above inequalities into \eqref{mth_ineq_1} we obtain 
\eqref{eq:mth_bound}.
\end{proof}

\section{CLT for the Polyak-Ruppert averaged estimator}
\label{appendix:pr_clt}
We preface the proof of \Cref{th:martingal_pr_clt} with a key decomposition isolating a linear statistics of
\(\funnoisew_V^{k+1}, \funnoisew_W^{k+1}\), which form a martingale difference sequences w.r.t. the natural filtration $\F_k = \sigma(X_s : s \leq k)$. 

\begin{lemma}
\label{lem:pr_main_decomp}
The following decomposition holds:
\begin{equation}
\label{eq:theta_k_thetas_error}
\Delta(\theta_k-\theta^*) = \beta_k^{-1}(\theta_k - \theta_{k+1})- \gamma_k^{-1}A_{12}A_{22}^{-1}(w_k - w_{k+1})  + (V_{k+1} - A_{12}A_{22}^{-1}W_{k+1})\eqsp.
\end{equation}
Moreover, it holds that 
\begin{equation}
\label{eq:theta_n_averaged_error}
\sqrt{n}\Delta(\btheta_n - \thetas) = \frac{1}{\sqrt{n}} \sum_{k=1}^n (\funnoisew_V^{k+1} - A_{12} A_{22}^{-1} \funnoisew_W^{k+1}) + R_{n}^{\operatorname{pr}}\eqsp,
\end{equation}
where the residual term $R_n^{\operatorname{pr}} = Y_1 + Y_2 + Y_3$, and  $Y_1, Y_2, Y_3$ are given by 
\begin{equation}
\label{eq:pr_theta_decomposition_appendix}
\begin{split}
Y_1 &= \frac{1}{\sqrt{n}} \sum_{k=1}^{n}\beta_k^{-1}(\theta_k - \theta_{k+1})\eqsp, \\
Y_2 &= -\frac{1}{\sqrt{n}}\sum_{k=1}^{n}A_{12}A_{22}^{-1}\gamma_k^{-1}(w_k - w_{k+1}) \eqsp, \\
Y_3 &= \frac{1}{\sqrt{n}} \sum_{k=1}^n \{\underbrace{(A_{12} A_{22}^{-1} \funcAwtilde_{21}^{k+1} - \funcAwtilde_{11}^{k+1})}_{R^\theta_{k+1}} \ttheta_{k+1} + \underbrace{(A_{12} A_{22}^{-1} \funcAwtilde_{22}^{k+1} - \funcAwtilde_{12}^{k+1})}_{R^w_{k+1}} (w_{k+1}-w^\star) \}\eqsp.
\end{split}
\end{equation}

\end{lemma}
\begin{proof}
First, we prove the representation \eqref{eq:theta_k_thetas_error}. Equation \eqref{eq:2ts2} implies:
\begin{equation}\label{eq:2ts2_newversion}
w_k = (\gamma_kA_{22})^{-1}(w_k-w_{k+1}) + A_{22}^{-1}(b_2 - A_{21}\theta_k+W_{k+1})\eqsp.
\end{equation}
Substituting \eqref{eq:2ts2_newversion} into the slow-time-scale variable's recursion \eqref{eq:2ts1} we obtain that
\begin{align}
\label{eq:pr_eq_0}
    \theta_{k+1} &= \theta_k +\beta_k(b_1 - A_{11}\theta_k+V_{k+1})- \beta_kA_{12}(\gamma_kA_{22})^{-1}(w_k-w_{k+1}) - \beta_k A_{12}A_{22}^{-1}(b_2 - A_{21}\theta_k + W_{k+1})\label{pr_eq_1}\\
    &= (\Id - \beta_k\Delta)\theta_k - \frac{\beta_k}{\gamma_k}A_{12}A_{22}^{-1}(w_k - w_{k+1}) + \beta_k(b_1 - A_{12}A_{22}^{-1}b_2) + \beta_k(V_{k+1} - A_{12}A_{22}^{-1}W_{k+1})\label{pr_eq_2}\eqsp,
\end{align}
Recall that $(\theta^*, w^*)$ is the solution of the system \eqref{eq:linear_sys}. This implies that \(b_1 - A_{12}A_{22}^{-1}b_2 = (A_{11} - A_{12}A_{22}^{-1}A_{21})\theta^* = \Delta\theta^*\). Substituting this equality into \eqref{eq:pr_eq_0}, we obtain the formula \eqref{eq:theta_k_thetas_error}. To establish \eqref{eq:theta_n_averaged_error}, we sum \eqref{eq:theta_k_thetas_error} over $k=1, \ldots, n$ using the expressions for $V_{k+1}$, $W_{k+1}$ \eqref{eq:noise_term_new} and unrolling the corresponding recurrence.
\end{proof}

To proceed with \Cref{th:martingal_pr_clt}, we first formulate the moment bounds for $Y_1, Y_2, Y_3$.

\begin{lemma}
\label{lem:pr_y1bound}
Let $p \geq 2$. Assume \Cref{assum:zero-mean},\Cref{assum:bound-conditional-moments-p}($p$), \Cref{assum:quadratic_characteristic}, \Cref{assum:hurwitz}, \Cref{assum:stepsize}($p$), and \Cref{assum:aij_bound}. Then for any \(k\in\mathbb{N}\) it holds that
\begin{align}
    \PE^{1/p}[\|Y_1\|^p] &\leq \frac{\ConstC_{1}^{Y_1} }{\sqrt{n}} +  \ConstC_{2}^{Y_1} (1+k_0)^{b+1} \frac{(n + k_0)^{b/2}}{\sqrt{n}} \eqsp,
\end{align}
where we have set
\begin{align}
    &\ConstC_{1}^{Y_1} = \ConstC_0^{\ttheta} \eqsp, \eqsp \ConstC_{2}^{Y_1} = 5\ConstC_{\operatorname{slow}} c_{0, \beta}^{-1/2} + c_{0, \beta}^{-2} C_0^{\ttheta} (c_{0, \beta} + \frac{8}{a_\Delta (1- b)}) + c_{0, \beta}^{-1} (\ConstC_0^{\ttheta} + \ConstC_{\operatorname{slow}} c_{0, \beta}^{1/2}) \eqsp.
\end{align}
\end{lemma}

\begin{lemma}
\label{lem:pr_y2bound}
Let $p \geq 2$. Assume \Cref{assum:zero-mean},\Cref{assum:bound-conditional-moments-p}($p$), \Cref{assum:quadratic_characteristic}, \Cref{assum:hurwitz}, \Cref{assum:stepsize}($p$), and \Cref{assum:aij_bound}. Then for any \(k\in\mathbb{N}\) it holds  
    \begin{align}
        \PE^{1/p}[\|Y_2\|] &\leq  \frac{\ConstC_{1}^{Y_2}}{\sqrt{n}}  + \ConstC_{2}^{Y_2} (1 + k_0)^{b+1} \frac{(n+k_0)^{a/2}}{\sqrt{n}}
        \eqsp,
    \end{align}
    where we have set
    \begin{align}
         \ConstC_1^{Y_2} = \|A_{12}A_{22}^{-1}\|\widehat{\ConstC}_0^{\tw}\rstep \eqsp, \eqsp
         \ConstC_2^{Y_2} = 5\|A_{12}A_{22}^{-1}\| c_{0, \gamma}^{-1/2} \widehat{\ConstC}_{\operatorname{fast}} + \|A_{12}A_{22}^{-1}\|\biggl(\frac{\rstep}{c_{0, \gamma} c_{0, \beta}} (c_{0, \beta} + \frac{8}{a_\Delta (1-b)}) + c_{0, \gamma}^{-1} (\widehat{\ConstC}_0^{\tw} + c_{0, \gamma}^{1/2} \widehat{\ConstC}_{\operatorname{fast}})\biggr) \eqsp. 
    \end{align}
\end{lemma}

\begin{lemma}
    \label{lem:pr_y3bound}
    Let $p \geq 2$. Assume \Cref{assum:zero-mean},\Cref{assum:bound-conditional-moments-p}($p$), \Cref{assum:quadratic_characteristic}, \Cref{assum:hurwitz}, \Cref{assum:stepsize}($p$), and \Cref{assum:aij_bound}. Therefore, it holds that
\begin{align}
    \expep{Y_3}{p} \leq \ConstC_{1}^{Y_3} \frac{p^4}{n^{a/2}} + \ConstC_{2}^{Y_3} \frac{k_0^{b/2}p}{\sqrt{n}}  \eqsp,
\end{align}
where $X$ is defined in \eqref{eq:pr_theta_decomposition} and we have set
\begin{align}
    &\ConstC_{1}^{Y_3} = 2 q_R\ConstC_{\operatorname{slow}} \biggl(\frac{c_{0, \beta}}{1-b}\biggr)^{1/2} +  2q_R\widehat{\ConstC}_{\operatorname{fast}} \biggl(\frac{c_{0, \gamma}}{1-a}\biggr)^{1/2} ,
    \quad \ConstC_{2}^{Y_3} = 2c_{0, \beta}^{-1/2} q_R  \biggl(1 + \frac{2\sqrt{2}}{\sqrt{a_\Delta (1-b)}}\biggr) ( \ConstC_0^{\ttheta}  + \widehat{\ConstC}_0^{\tw} )  \eqsp,
\end{align}
and
\begin{align}\label{eq:q_R_definition}
    &q_{R} = \bConst{A} + \norm{A_{12} A_{22}^{-1}} \bConst{A} \eqsp.
\end{align}
\end{lemma}

\begin{proof}[Proof of \Cref{th:martingal_pr_clt}]
\label{sec:cor_pr_clt_proof}
Our proof starts from the error decomposition \eqref{eq:pr_theta_decomposition}, which allows us to write 
\[
\sqrt{n} \Delta(\btheta_n - \thetas) = \frac{1}{\sqrt{n}} \sum_{k=1}^n (\funnoisew_V^{k+1} - A_{12} A_{22}^{-1} \funnoisew_W^{k+1}) + R_{n}^{\operatorname{pr}}\eqsp,
\]
where the term $R_n$ is given in \Cref{lem:pr_main_decomp}. Note that the term 
\[
\frac{1}{\sqrt{n}} \sum_{k=1}^n (\funnoisew_V^{k+1} - A_{12} A_{22}^{-1} \funnoisew_W^{k+1})
\]
is a linear statistic of the random variables $\psi_{k+1} = \funnoisew_{V}^{k+1} - A_{12} A_{22}^{-1}\funnoisew_{W}^{k+1}$, while $R_n^{\operatorname{pr}}$ is a "remainder" term, which moments are small, as we show below. Under \Cref{assum:zero-mean}, $\{\psi_{k+1}\}_{k \in \nset}$ form a martingale-difference w.r.t. $\F_k$. Since the convex distance is invariant to non-degenerate linear transformations, 
\[
\kolmogorov\bigl(\sqrt{n} \Delta(\btheta_n - \thetas), \mathcal{N}(0, \lineG_\eps)\bigr) = \kolmogorov\bigl(\sqrt{n} \lineG_\eps^{-1/2} \Delta(\btheta_n - \thetas), \mathcal{N}(0, \Id)\bigr)\eqsp.
\]
Set $p = \log n$. To control this term we apply \Cref{prop:nonlinearapprox} and obtain
\begin{align} 
\label{eq:th_kolmogorov_first_bound_appendix}
\kolmogorov\bigl({\sqrt{n} \Delta(\btheta_n - \thetas), \eqsp \mathcal{N}(0,  \lineG_\eps )\bigr)} \leq &\underbrace{\kolmogorov\bigl({\frac{1}{\sqrt{n}} \sum_{k=1}^n  \psi_{k+1}, \eqsp \mathcal{N}(0,  \lineG_\eps )}\bigr)}_{T_1} + \underbrace{2 c_{d_\theta}^{p/p+1}\PE^{1/(p+1)}\bigl[\norm{\lineG_\eps^{-1/2} R_n^{\operatorname{pr}}}^{p}\bigr]}_{T_2}\eqsp,
\end{align}
where $c_d$ is the isoperimetric constant of the convex sets, see \Cref{prop:nonlinearapprox} for detailed discussion. Hence, now it remains to estimate the normal approximation rate for $T_1$, and to control $T_2$. To proceed with $T_1$, we use the martingale CLT \cite[Theorem~1]{wu2025uncertainty}. For completeness, we state this result in the current paper, see \Cref{th:wu_martingale_limit}. It is important to acknowledge that this result requires that $\psi_{k+1}$ are a.s. bounded and have constant quadratic characteristic. Both assumptions hold in our setting, since, due to \Cref{assum:aij_bound}, 
\begin{align}
\norm{\psi_{k+1}} \leq (1 + \|A_{12}A_{22}^{-1}\|)(\bConst{b} + \bConst{A} (\norm{\thetas} + \norm{w^\star}) =: \Psi \eqsp,
\end{align}
and assumption \Cref{assum:quadratic_characteristic} implies that 
\begin{align}
\CPE{\psi_{k+1} \psi_{k+1}^\top}{\F_k} = \Sigma_V + A_{12} A_{22}^{-1} \Sigma_W (A_{12} A_{22}^{-1})^\top + \Sigma_{VW} (A_{12} A_{22}^{-1})^\top + A_{12} A_{22}^{-1} \Sigma_{VW}^\top =: \lineG_\eps \eqsp.
\end{align}
Hence, applying \Cref{th:wu_martingale_limit}, we get:
\begin{align}
\kolmogorov\bigl({\frac{1}{\sqrt{n}} \sum_{k=1}^n  \psi_{k+1}, \eqsp \mathcal{N}(0,  \lineG_\eps )} \bigr) \lesssim [1 + (2 + \log(d_\theta n \norm{\lineG_\eps}))^+]^{1/2} \frac{\sqrt{d_\theta \log n}}{n^{1/4}}  \eqsp.
\end{align}
Now we proceed with the term $T_2$ defined in \eqref{eq:th_kolmogorov_first_bound_appendix}. We use the representation $R_n^{\operatorname{pr}} = Y_1 + Y_2 + Y_3$ and \Cref{lem:pr_y1bound,lem:pr_y2bound,lem:pr_y3bound} to control the moments of each $Y_i$. Combining these lemmas,
\begin{align}
\expep{R_n^{\operatorname{pr}}}{p} &\leq \expep{Y_1}{p} +  \expep{Y_2}{p} +  \expep{Y_3}{p} \\
&\leq \ConstC_{1}^{Y_1} \frac{1}{\sqrt{n}}  +  \ConstC_{2}^{Y_1}  k_0^{b+1} \frac{(n + k_0)^{b/2}}{\sqrt{n}} + \ConstC_{1}^{Y_2} \frac{1}{\sqrt{n}}  + \ConstC_{2}^{Y_2}  (1 + k_0)^{b+1} \frac{(n+k_0)^{a/2}}{\sqrt{n}} + \ConstC_{1}^{Y_3} \frac{p^4}{n^{a/2}} + \ConstC_{2}^{Y_3} \frac{k_0^{b/2}p}{\sqrt{n}}  \\
&\lesssim p^{4 + 4/b} \biggl( \frac{1}{n^{a/2}} 
+ \frac{1}{n^{(1-b)/2}} \biggr)\eqsp,
\end{align}
where we use inequality \((1+k_0)^{b + 1}\lesssim p^{4+4/b}\).
Combining the above inequalities, we obtain 
\begin{equation}
         \kolmogorov\bigl({\sqrt{n} \Delta(\btheta_n - \thetas), \eqsp \mathcal{N}(0,  \lineG_\eps )\bigr)} \lesssim [1 + (2 + \log(d_\theta n \norm{\lineG_\eps}))^+]^{1/2} \frac{\sqrt{d_\theta \log n}}{n^{1/4}} + c_{d_\theta} p^{\frac{p}{p+1} (4 + 4/b)}\bigl(\frac{1}{n^{a/2}}
        + \frac{1}{n^{(1-b)/2}} \bigr)^{\frac{p}{p+1}}
\end{equation}
Note that $(n^\alpha)^{\frac{\log n}{1 + \log n}} \lesssim n^{\alpha}$ for all $\alpha \in \rset$ and for all $a, b > 0$ and $q \in (0, 1)$ it holds that $(a+b)^q \leq a^q + b^q$.
Hence,  we get
\begin{align}
    \kolmogorov\bigl({\sqrt{n} \Delta(\btheta_n - \thetas), \eqsp \mathcal{N}(0,  \lineG_\eps )\bigr)} \lesssim \{\log n\}^{4+4/b} \biggl( \frac{1}{n^{a/2}} + \frac{1}{n^{(1-b)/2}}\biggr) \eqsp.
\end{align}
\end{proof}

We finish this section by giving the proofs for Lemmas \ref{lem:pr_y1bound}-\ref{lem:pr_y3bound}.

\begin{proof}[Proof of \Cref{lem:pr_y1bound}]
Observe that
    \begin{align}
    Y_1 = \frac{1}{\sqrt{n}}\sum_{k=1}^{n}\beta_k^{-1}(\theta_k - \theta_{k+1}) = \frac{1}{\sqrt{n}}\beta_1^{-1} \ttheta_1 - \frac{1}{\sqrt{n}}\beta_{n}^{-1}\ttheta_{n+1} + \frac{1}{\sqrt{n}} \sum_{k=1}^{n-1}(\beta_{k+1}^{-1} - \beta_k^{-1})\ttheta_{k+1}\eqsp.
\end{align}
Applying Minkowski's inequality and an elementary inequality $\beta_{k+1}^{-1} - \beta_{k}^{-1} \leq (k\beta_k)^{-1}$, we obtain that
\begin{align}
    \mathbb{E}^{1/p}[\|Y_1\|^p] &\leq \frac{1}{\sqrt{n}} \beta_1^{-1}\mathbb{E}^{1/p}[\|\ttheta_1\|^p] + \frac{1}{\sqrt{n}} \beta_{n}^{-1}\mathbb{E}^{1/p}[\|\ttheta_{n+1}\|^p] +  \frac{1}{\sqrt{n}} \sum_{k=1}^{n-1}(k\beta_k)^{-1}\mathbb{E}^{1/p}[\|\ttheta_{k+1}\|^p]\eqsp.
\end{align}
Using \Cref{prop:moments_bound} we get:
\begin{equation}
 \beta_{n}^{-1}\mathbb{E}^{1/p}[\|\ttheta_{n+1}\|^p] \leq \beta_n^{-1} \ConstC_0^{\ttheta} \prod_{j=0}^{n}\bigl(1 - \beta_j\frac{a_{\Delta}}{8}\bigr) + p^2 \ConstC_{\operatorname{slow}} \beta_{n}^{-1/2} \overset{(i)}{\leq} \ConstC_{0}^{\ttheta} + p^2 \ConstC_{\operatorname{slow}} \beta_{n}^{-1/2}\eqsp,
\end{equation}
where in $(i)$ we have used the inequality \(\prod_{j=0}^{n}\bigl(1 - \beta_j\frac{a_{\Delta}}{8}\bigr) \leq \beta_{n}\). Next, we observe that
\begin{equation}
\begin{split}
\beta_1^{-1} \PE^{1/p}[\norm{\ttheta_1}^p] &\leq \frac{p^2 (1 + k_0)^b}{c_{0, \beta}} \{ \ConstC_{0}^{\ttheta} + \ConstC_{\operatorname{slow}} c_{0, \beta}^{1/2} \} \eqsp, \\
 \sum_{k=1}^{n-1}(k\beta_k)^{-1}\mathbb{E}^{1/p}[\|\ttheta_{k+1}\|^p] &\leq \ConstC_0^{\ttheta}\sum_{k=1}^{n-1}(k\beta_k)^{-1} \prod_{j=0}^{k}\bigl(1 - \beta_j\frac{a_{\Delta}}{8}\bigr) + p^2 \ConstC_{\operatorname{slow}} \sum_{k=1}^{n-1}(k\beta_k)^{-1}\beta_{k}^{1/2}\eqsp, 
\end{split}
\end{equation}
Now we derive an upper bound for the r.h.s. of the latter inequality. Applying \Cref{lem:summ_alpha_k}-\ref{lem:sum_as_Qell_item} and \((k\beta_k)^{-1}\leq c_{0,\beta}^{-1}k_0\), we get:
\begin{align}
    &\ConstC_0^{\ttheta}\sum_{k=1}^{n-1}(k\beta_k)^{-1} \prod_{j=0}^{k}\bigl(1 - \beta_j\frac{a_{\Delta}}{8}\bigr) \leq \ConstC_0^{\ttheta} \frac{k_0}{c_{0, \beta}} \sum_{k=1}^n \prod_{j=0}^k (1 - \beta_j \frac{a_\Delta}{8}) \leq \ConstC_0^{\ttheta}\frac{k_0^{b+1}}{c_{0, \beta}^2} (c_{0, \beta} + \frac{8}{a_\Delta(1-b)}) \eqsp.
\end{align}
Bound for the second term can be obtained from the straightforward computations:
\begin{equation}
     p^2 \ConstC_{\operatorname{slow}} \sum_{k=1}^{n-1}(k\beta_k)^{-1}\beta_{k}^{1/2} = \frac{p^2 \ConstC_{\operatorname{slow}}}{c_{0, \beta}^{1/2}} \sum_{k=1}^{n-1} \frac{(k+k_0)^{b/2}}{k} \leq \frac{p^2 \ConstC_{\operatorname{slow}}}{c_{0, \beta}^{1/2}} (1 + k_0) \frac{2}{b} (n + k_0)^{b/2} \leq \frac{4p^2 \ConstC_{\operatorname{slow}}}{c_{0, \beta}^{1/2}}  (1 + k_0) (n + k_0)^{b/2} \eqsp.
\end{equation}
The proof follows from gathering the above inequalities.
\end{proof}

\begin{proof}[Proof of \Cref{lem:pr_y2bound}]
Observe that
    \begin{equation}
        Y_2 =  A_{12}A_{22}^{-1}\frac{1}{\sqrt{n}} \sum_{k=1}^{n}\gamma_k^{-1}(w_k - w_{k+1}) = A_{12}A_{22}^{-1}\biggl(\frac{1}{\sqrt{n}}\gamma_1^{-1}(w_1 - w^*) - \frac{1}{\sqrt{n}} \gamma_{n}^{-1}(w_{n+1} - w^{*}) + \frac{1}{\sqrt{n}} \sum_{k=1}^{n-1}(\gamma_{k+1}^{-1} - \gamma_{k}^{-1})(w_{k+1} - w^*)\biggr)\eqsp.
    \end{equation}
Applying Minkowski's inequality and $\gamma_{k+1}^{-1} - \gamma_{k}^{-1} \leq (k \gamma_k)^{-1}$, we obtain that
\begin{align}
    \mathbb{E}^{1/p}[\|Y_2\|^p] &\leq \|A_{12}A_{22}^{-1}\|\biggl(\frac{1}{\sqrt{n}} \gamma_1^{-1}\mathbb{E}^{1/p}[\|w_1-w^*\|^p] + \frac{1}{\sqrt{n}} \gamma_{n}^{-1}\mathbb{E}^{1/p}[\|w_{n+1} - w^*\|^p] +  \frac{1}{\sqrt{n}} \sum_{k=1}^{n-1}(k\gamma_k)^{-1}\mathbb{E}^{1/p}[\|w_{k+1} - w^*\|^p]\biggr)\eqsp.
\end{align}
Using \Cref{lem:martingale:w_minus_wstar_moment_bound} we get:
\begin{equation}
    \gamma_{n}^{-1}\mathbb{E}^{1/p}[\|w_{n+1} - w^*\|^p] \leq \widehat{\ConstC}_0^{\tw} \gamma_n^{-1}\prod_{j=0}^n (1 - \beta_j \frac{a_\Delta}{8}) + p^3 \widehat{\ConstC}_{\operatorname{fast}}\frac{(n+k_0)^{a/2}}{c_{0, \gamma}^{1/2}} \leq  \widehat{\ConstC}_0^{\tw} \rstep + p^3 \widehat{\ConstC}_{\operatorname{fast}}\frac{(n+k_0)^{a/2}}{c_{0, \gamma}^{1/2}}  \eqsp, 
\end{equation}
where the last inequality follows from the inequalities \(\gamma_n^{-1}\leq \beta_n^{-1}\rstep\) and \(\prod_{j=0}^{n}\bigl(1 - \beta_j\frac{a_{\Delta}}{8}\bigr) \leq \beta_{n}\). Note that
\begin{equation}
\begin{split}
\gamma_1^{-1} \PE^{1/p}[\norm{w_{1} - \wstar}^p] &\leq p^3 \frac{(1+k_0)^a}{c_{0, \gamma}} (\widehat{\ConstC}_0^{\tw} + \widehat{\ConstC}_{\operatorname{fast}} c_{0, \gamma}^{1/2}) \eqsp, \\
\sum_{k=1}^{n-1}(k\gamma_k)^{-1}\mathbb{E}^{1/p}[\|w_{k+1} - w^*\|^p] &\leq \widehat{\ConstC}_{0}^{\tw} \sum_{k=1}^{n-1}(k\gamma_k)^{-1} \prod_{j=0}^{k}\bigl(1 - \beta_j\frac{a_{\Delta}}{8}\bigr) + p^3 \widehat{\ConstC}_{\operatorname{fast}} \sum_{k=1}^{n-1}(k\gamma_k)^{-1}\gamma_{k}^{1/2}\eqsp, 
\end{split}
\end{equation}
We bound the r.h.s of the latter inequality using \Cref{lem:summ_alpha_k}-\ref{lem:sum_as_Qell_item}:
\begin{align}
    &\sum_{k=1}^{n-1}(k\gamma_k)^{-1} \prod_{j=0}^{k}\bigl(1 - \beta_j\frac{a_{\Delta}}{8}\bigr) \leq \frac{k_0^{b+1}\rstep}{c_{0, \gamma} c_{0, \beta}} (c_{0, \beta} + \frac{8}{a_{\Delta}(1-b)}) \eqsp.
\end{align}
Bound for the second term can be obtained from the straightforward computations:
\begin{equation}
     p^3 \widehat{\ConstC}_{\operatorname{fast}} \sum_{k=1}^{n-1}(k\gamma_k)^{-1}\beta_{k}^{1/2} = \frac{p^3 \widehat{\ConstC}_{\operatorname{fast}}}{c_{0, \gamma}^{1/2}} \sum_{k=1}^{n-1} \frac{(k+k_0)^{a/2}}{k} \leq \frac{p^3 \widehat{\ConstC}_{\operatorname{fast}}}{c_{0, \gamma}^{1/2}} (1 + k_0) \frac{2}{a} (n + k_0)^{a/2} \leq \frac{4p^3 \widehat{\ConstC}_{\operatorname{fast}}}{c_{0, \gamma}^{1/2}} \ConstC_{\operatorname{slow}} (1 + k_0) (n + k_0)^{a/2} \eqsp.
\end{equation}
The proof follows from gathering the above bounds.
\end{proof}

\begin{proof}[Proof of \Cref{lem:pr_y3bound}]
Note that since $V_{k+1}$, $W_{k+1}$ and $\eps_V$, $\eps_W$ are martingale difference sequences, $R^{\theta}_{k+1}\ttheta_{k} + R^w_{k+1} (w_k-w^\star)$ is a martingale difference sequence. Therefore, Burkholder's inequality \cite[Theorem 8.1]{osekowski} implies that
\begin{align}
    \label{eq:pr_x_bound_first}
    \expep{\frac{1}{\sqrt{n}} \sum_{k=1}^n (R^{\theta}_{k+1}\ttheta_{k+1} + R^w_{k+1} (w_{k + 1} - \wstar) )}{p} &\leq p \bigl(\frac{1}{n} \sum_{k=1}^n \PE^{2/p} [\norm{R^{\theta}_{k+1}\ttheta_{k+1} + R^w_{k+1} (w_{k+1} - \wstar)}^p] \bigr)^{1/2} \\
    &\leq p\bigl(\frac{2}{n} \sum_{k=1}^n (\PE^{2/p} [\norm{R^{\theta}_{k+1}\ttheta_{k+1}}^p] + \PE^{2/p} [\norm{R^w_{k+1} (w_{k+1} - \wstar)}^p])\bigr)^{1/2}
\end{align}

Using \Cref{assum:aij_bound} easy to see that $\norm{R^\theta_{k+1}} \leq q_R$ and $\norm{R^w_{k+1}} \leq q_R$, where \(q_R\) is defined in \eqref{eq:q_R_definition}. Hence,  \Cref{prop:moments_bound} and \Cref{lem:summ_alpha_k}-\ref{lem:sum_as_Qell_item} imply that:
\begin{align}
    \sum_{k=1}^n (\PE^{2/p} [\norm{R^{\theta}_{k+1}\ttheta_{k+1}}^p] &\leq 2q_R^2 \sum_{k=1}^{n} \bigl\{ \{\ConstC_0^{\ttheta}\}^2 \prod_{j=0}^{k} \bigl(1 - \beta_j\frac{a_{\Delta}}{8}\bigr) + p^4 \{\ConstC_{\operatorname{slow}}\}^2 \beta_{k+1} \bigr\} \\
    &\leq 2q_R^2 \{\ConstC_0^{\ttheta}\}^2 \frac{k_0^b}{c_{0, \beta}} (1 + \frac{8}{a_\Delta(1-b)}) + 2q_{R}^2 p^4 \{\ConstC_{\operatorname{slow}}\}^2 \frac{c_{0, \beta} n^{1-b}}{1-b} \eqsp.
\end{align}
Similarly, one can get using \Cref{lem:martingale:w_minus_wstar_moment_bound}:
\begin{align}
    \sum_{k=1}^n (\PE^{2/p} [\norm{R^{\theta}_{k+1}(w_{k+1} - w^\star)}^p] &\leq 2q_R^2 \sum_{k=1}^{n} \bigl\{ (\widehat{\ConstC}_0^{\tw})^2 \prod_{j=0}^{k}\bigl(1 - \beta_j\frac{a_{\Delta}}{8}\bigr) + p^6 \{\widehat{\ConstC}_{\operatorname{fast}}\}^2 \gamma_{k+1} \bigr\} \\
    &\leq 2q_{R}^2 (\widehat{\ConstC}_0^{\tw})^2 \frac{k_0^b}{c_{0, \beta}} (1 + \frac{8}{a_\Delta(1-b)}) + 2q_{R}^2 p^6 \{\widehat{\ConstC}_{\operatorname{fast}}\}^2 \frac{c_{0, \gamma} n^{1-a}}{1-a} \eqsp.
\end{align}
The proof follows from gathering similar terms.
\end{proof}

\section{CLT for the Last iteration estimator}
\label{appendix:last_iter_clt}
The proof of \Cref{th:last_iter_clt_final} is conceptually similar to \Cref{th:martingal_pr_clt} but is complicated by two additional factors. First, the covariance of the linear statistic formed by martingale differences now depends on $n$, which complicates the identification of the limiting covariance matrix. Second, the analysis of the remaining terms becomes more involved.

Before proceeding to the main part of the proof, we state an auxiliary lemma that decomposes the approximation error at the final iteration into linear and nonlinear components.

\begin{lemma}\label{lem:last_iter_decomp}
The following decomposition holds:
    \begin{equation}
         \ttheta_{n+1} = \sum_{j=0}^n \beta_j G_{j+1:n}^{(1)} \bigl( A_{12} A_{22}^{-1} \funnoisew_W^{j+1} - \funnoisew_V^{j+1}\bigr) + R_n^{\operatorname{last}} \eqsp,
    \end{equation}
where the residual term $R_n^{\operatorname{last}}$ is defined in \eqref{eq:def:martingale:rn_last}.
\end{lemma}
\begin{proof}
    Following \cite{konda:tsitsiklis:2004} and using \eqref{eq:2ts2-1},  the equations for \(\ttheta_{n+1}\) and \(\tw_{n+1}\) can be rewritten as follows:
\begin{align}    \label{eq:tt_hard_rec_appendix}
    \ttheta_{n+1} &= ( \Id - \beta_n \Delta ) \ttheta_n - \beta_n A_{12} \Tilde{w}_n - \beta_n V_{n+1} + \beta_n\delta_n^{(1)}\eqsp,\\
    \tw_{n+1} &= (\Id - \gamma_n A_{22})\tw_n -\beta_n D_{n}V_{n+1} - \gamma_n W_{n+1} + \beta_n\delta_n^{(2)}\eqsp,
\end{align}
where we have set
\begin{align}
    \delta_n^{(1)} = A_{12}L_n\ttheta_n \eqsp, \eqsp \eqsp \eqsp  \delta_n^{(2)} = -(L_{n+1} + A_{22}^{-1}A_{21})A_{12}\tw_n\eqsp.
\end{align}
Throughout the analysis we use the following convention:
\begin{equation}
\begin{split}
&G_{m:n}^{(1)} := \prod_{i=m}^n (\Id - \beta_i\Delta ), \quad G_{m:n}^{(2)} := \prod_{i=m}^n (\Id - \gamma_i A_{22} )\eqsp. 
\end{split}
\end{equation}
Hence, \(\ttheta_{n+1}\) and \(\tw_{n+1}\) can be rewritten as follows:
\begin{equation}
\label{eq:last_iter_decompos_2_1}
    \ttheta_{n+1} = G_{0:n}^{(1)}\ttheta_0 - \sum_{j=0}^{n}\beta_jG_{j+1:n}^{(1)}A_{12}\tw_j - \sum_{j=0}^{n}\beta_j G_{j+1:n}^{(1)}V_{j+1} + \sum_{j=0}^{n}\beta_jG_{j+1:n}^{(1)}\delta_j^{(1)}\eqsp,
\end{equation}
\begin{equation}
\label{eq:last_iter_decompos_2_2}
    \tw_{n+1} = G_{0:n}^{(2)}\tw_0 - \sum_{j=0}^{n}\beta_j G_{j+1:n}^{(2)}D_jV_{j+1} - \sum_{j=0}^{n}\gamma_jG_{j+1:n}^{(2)}W_{j+1} + \sum_{j=0}^{k}\beta_j G_{j+1:n}^{(2)}\delta_j^{(2)}\eqsp.
\end{equation}
We substitute the right-hand side of \eqref{eq:last_iter_decompos_2_2} into \eqref{eq:last_iter_decompos_2_1} and obtain:
\begin{align}
    \label{eq:last_iteration_repr}
    \ttheta_{n+1} = G_{0:n}^{(1)}\ttheta_0 - \sum_{j=0}^{n}\beta_jG_{j+1:n}^{(1)}A_{12}G_{0:j-1}^{(2)}\tw_0 &+ \sum_{j=0}^{n}\beta_jG_{j+1:n}^{(1)}\;\delta_j^{(1)} + S_{n}^{(1)} + S_{n}^{(2)} + S_{n}^{(3)} \\
    &+ \sum_{j=0}^{n}\beta_j G_{j+1:n}^{(1)}\bigl(A_{12}A_{22}^{-1}W_{j+1} - V_{j+1}\bigr)\eqsp,
\end{align}
where
\begin{align}\label{eq:last_iter_decompos_2}
    S_n^{(1)} &= -\sum_{j=0}^{n}\beta_jG_{j+1:n}^{(1)}A_{12}\sum_{i=0}^{j-1}\beta_iG_{i+1:j-1}^{(2)}\delta_i^{(2)}\eqsp, \\
    S_n^{(2)} &= \sum_{j=0}^{n}\beta_jG_{j+1:n}^{(1)}A_{12}\sum_{i=0}^{j-1}\beta_iG_{i+1:j-1}^{(2)}D_iV_{i+1}\eqsp,\\
    S_n^{(3)} &= \sum_{j=0}^{n}\beta_jG_{j+1:n}^{(1)}A_{12}\sum_{i=0}^{j-1}\gamma_iG_{i+1:j-1}^{(2)}W_{i+1} - \sum_{j=0}^{n}\beta_jG_{j+1:n}^{(1)}A_{12}A_{22}^{-1}W_{j+1}\eqsp.
\end{align}
Recall that \(\psi_{j+1} =  \funnoisew_V^{j+1} - A_{12} A_{22}^{-1} \funnoisew_W^{j+1}\). Substituting $V_{j+1}, W_{j+1}$ from \eqref{eq:noise_term_new} we get
\begin{align}\label{eq:last_iter_decompos}    
    \ttheta_{n+1} &= -\sum_{j=0}^n \beta_j G_{j+1:n}^{(1)} \psi_{j+1} + R_n^{\operatorname{last}} \eqsp,
\end{align}
where we have set
\begin{align}
    \label{eq:def:martingale:rn_last}
    R_n^{\operatorname{last}} = G_{0:n}^{(1)}\ttheta_0 &- \sum_{j=0}^{n}\beta_jG_{j+1:n}^{(1)}A_{12}G_{0:j-1}^{(2)}\tw_0 + \sum_{j=0}^{n}\beta_jG_{j+1:n}^{(1)}\;\delta_j^{(1)} + S_{n}^{(1)} + S_{n}^{(2)} + S_{n}^{(3)} \\
    &+ \underbrace{\sum_{j=0}^{n}\beta_j G_{j+1:n}^{(1)}\bigl(\overbrace{\bigl\{\funcAwtilde_{11}^{j+1}- A_{12}A_{22}^{-1} \funcAwtilde_{21}^{j+1}\bigr\}}^{R_{j+1}^\theta} \ttheta_j + \overbrace{\bigl\{\funcAwtilde_{12}^{j+1} -  A_{12}A_{22}^{-1} \funcAwtilde_{21}^{j+1}\bigr\}}^{R_{j+1}^w} (w_j-w^\star)   \bigr)}_{H_{n}}
\end{align}
\end{proof}

Now we proceed with \Cref{th:last_iter_clt_final} applying the decomposition that is proven above together with Lemmas \ref{lem:delta_bound}-\ref{lem:s_3_bound} which imply a moment bound for $R_n^{\operatorname{last}}$. 

\begin{lemma}
\label{lem:delta_bound}
    Let $p \geq 2$. Assume \Cref{assum:zero-mean}, \Cref{assum:bound-conditional-moments-p}($p$), \Cref{assum:quadratic_characteristic}, \Cref{assum:hurwitz}, \Cref{assum:stepsize}, \Cref{assum:aij_bound}, \Cref{assum:last_iter}. Therefore, it holds that
\begin{align}
        \expep{\sum_{j=0}^n \beta_j G_{j+1:n}^{(1)} \delta_j^{(1)}}{p}  \lesssim (2b-a-1)^{-1} P_{0:n}^{(1)} + p^4 (2b-a-1)^{-1} \beta_n^{\frac{2b-a/2-1}{b}} \eqsp.
\end{align}
\end{lemma}

\begin{lemma}
\label{lem:h_bound}
    Let $p \geq 2$. Assume \Cref{assum:zero-mean}, \Cref{assum:bound-conditional-moments-p}($p$), \Cref{assum:quadratic_characteristic}, \Cref{assum:hurwitz}, \Cref{assum:stepsize}, \Cref{assum:aij_bound}. Therefore, it holds that    
    \begin{align}
        \expep{H_n}{p} \lesssim p(2b-1)^{-1/2} \prod_{j=0}^n (1-\frac{a_\Delta}{8}\beta_j) + p^4  \beta_n^{\frac{b+a}{2b}} \eqsp.
    \end{align}
\end{lemma}

\begin{lemma}\label{lem:s_1_bound}
    Let $p \geq 2$. Assume \Cref{assum:zero-mean}, \Cref{assum:bound-conditional-moments-p}($p$), \Cref{assum:quadratic_characteristic}, \Cref{assum:hurwitz}, \Cref{assum:stepsize}, \Cref{assum:aij_bound}. Therefore, it holds that
    \begin{align}
        \expep{S_n^{(1)}}{p} \lesssim P_{0:n}^{(1)} +  p^4 \beta_n^{(3b-2a)/2b} \eqsp,
    \end{align}
    where $S_n^{(1)}$ is defined in \eqref{eq:last_iter_decompos_2}.
\end{lemma}

\begin{lemma}\label{lem:s_2_bound}
    Let $p \geq 2$. Assume \Cref{assum:zero-mean}, \Cref{assum:bound-conditional-moments-p}($p$), \Cref{assum:quadratic_characteristic}, \Cref{assum:hurwitz}, \Cref{assum:stepsize}, \Cref{assum:aij_bound}. Therefore, it holds that
    \begin{align}
        \expep{S_n^{(2)}}{p} \lesssim p^4  \beta_n^{\frac{3b-2a}{2b}} \eqsp,
    \end{align}
    where $S_n^{(2)}$ is defined in \eqref{eq:last_iter_decompos_2}.
\end{lemma}

\begin{lemma}\label{lem:s_3_bound}
    Let $p \geq 2$. Assume \Cref{assum:zero-mean}, \Cref{assum:bound-conditional-moments-p}($p$), \Cref{assum:quadratic_characteristic}, \Cref{assum:hurwitz}, \Cref{assum:stepsize}, \Cref{assum:aij_bound}. Therefore, it holds that
    \begin{align}
        \expep{S_n^{(3)}}{p} \lesssim  p^4\beta_n^{\frac{2b-a}{2b}}  \eqsp,
    \end{align}
    where $S_n^{(3)}$ is defined in \eqref{eq:last_iter_decompos_2}.
\end{lemma}

\begin{proof}[Proof of \Cref{th:last_iter_clt_final}]
Our proof starts from the error decomposition \eqref{eq:last_iter_decompos}, which allows us to write 
\[
 \beta_n^{-1/2}\ttheta_{n+1} = -\beta_n^{-1/2}\sum_{j=0}^k \beta_j G_{j+1:n}^{(1)}\psi_{j+1} + \beta_n^{-1/2}R_n^{\operatorname{last}} \eqsp,
\]
where the term $R_n^{\operatorname{last}}$ is given in \Cref{lem:pr_main_decomp}.
 Assumption \Cref{assum:quadratic_characteristic} implies that 
\begin{align}
        \lineG_n^{\operatorname{last}} :=  \sum_{j=0}^{n}\beta_j^2 G_{j+1:n}^{(1)} \lineG_\eps \big(G_{j+1:n}^{(1)}\big)^{\top}\eqsp.
    \end{align}
As established in \Cref{prop:ricatti_limit}, the sequence \(\lineG_n^{\operatorname{last}}\) converges to the matrix \(\lineG_{\operatorname{\infty}}^{\operatorname{last}}.\)
Since the convex distance is invariant to non-degenerate linear transformations, we get
\[
\kolmogorov\bigl(\beta_n^{-1/2}\ttheta_{n+1}, \mathcal{N}(0,  \lineG^{\operatorname{last}}_{\infty} )\bigr) = \kolmogorov\bigl(\beta_n^{-1/2} (\lineG^{\operatorname{last}}_{\infty}\bigl)^{-1/2} \ttheta_{n+1}, \mathcal{N}(0, \Id)\bigr)\eqsp.
\]
Hence, to control this term we apply \Cref{prop:nonlinearapprox} and triangle inequality for the convex distance and obtain
\begin{align} 
\label{eq:last_iter_3_term}
\kolmogorov\bigl(\beta_n^{-1/2}\ttheta_{n+1}\eqsp, \mathcal{N}(0,  \lineG^{\operatorname{last}}_{\infty} )\bigr) &\leq  2 c_d^{p/p+1}\PE^{1/(p+1)}\bigl[\norm{\big(\lineG_{\infty}^{\operatorname{last}}\big)^{-1/2} \beta_n^{-1/2}R_n^{\operatorname{last}}}^{p}\bigr]\\
&+\kolmogorov\bigl(-\beta_n^{-1/2}\sum_{j=0}^n \beta_j G_{j+1:n}^{(1)}\psi_{j+1} \bigr)\eqsp, \mathcal{N}(0, \beta_n^{-1}\lineG_n^{\operatorname{last}})\bigr)\\
& +\kolmogorov\bigl(\mathcal{N}(0, \beta_n^{-1}\lineG_n^{\operatorname{last}})\eqsp, \mathcal{N}(0, \lineG_{\infty}^{\operatorname{last}}))\bigr)\eqsp,
\end{align}
To handle the second term in \eqref{eq:last_iter_3_term} we apply \Cref{th:wu_martingale_limit} and get
\begin{align}
\kolmogorov\bigl(-\beta_n^{-1/2}\sum_{j=0}^n \beta_j G_{j+1:n}^{(1)}\psi_{j+1} \bigr)\eqsp, \mathcal{N}(0, \beta_n^{-1}\lineG_n^{\operatorname{last}})\bigr) &\lesssim [1 + (2 + \log(d_\theta \norm{\beta_n^{-1}\lineG_n^{\operatorname{last}}}))^+]^{1/2} \frac{\sqrt{d_\theta \log n}}{n^{1/4}} + \frac{d_\theta^{1/4} \norm{\frac{1}{n}\beta_n^{-1}\lineG_n^{\operatorname{last}}}^{1/4}}{n^{1/4}} \\
&\overset{(a)}{\lesssim} [1 + (2 + \log(d_\theta \norm{\lineG_\infty^{\operatorname{last}}}))^+]^{1/2} \frac{\sqrt{d_\theta \log n}}{n^{1/4}} \eqsp,
\end{align}
where in $(a)$ we have used \Cref{assum:last_iter}, i.e. $\frac{1}{2} \norm{\lineG_\infty^{\operatorname{last}}} \leq \norm{\beta_n^{-1} \lineG_n^{\operatorname{last}}} \lesssim \norm{\lineG_\infty^{\operatorname{last}}}$.
The bound for the third term in \eqref{eq:last_iter_3_term} follows from \cite[Theorem 1.1]{Devroye2018} and \Cref{prop:ricatti_limit}:
    \begin{align}
        \kolmogorov(\mathcal{N}(0, \beta_n^{-1}\lineG_{n}^{\operatorname{last}}), \mathcal{N}(0, \lineG_{\infty}^{\operatorname{last}})) \lesssim \normop{\{\lineG_{\infty}^{\operatorname{last}}\}^{-1/2} (\beta_n^{-1}\lineG_{n}^{\operatorname{last}}) \{\lineG_{\infty}^{\operatorname{last}}\}^{-1/2} - \Id}[\text{Fr}] \lesssim \frac{\sqrt{d}}{n^{b} \lambda_{\min}(\lineG_{\infty}^{\operatorname{last}})} \eqsp.
    \end{align}
Now we proceed with the first term in \eqref{eq:last_iter_3_term}. The moment bound for $R_n^{\operatorname{last}}$ follows from Lemmas \ref{lem:delta_bound}-\ref{lem:s_3_bound} that we have stated above. Combining these lemmas  together with the inequality $\beta_n^{\frac{2b-a/2-1}{b}} > \beta_n^{\frac{2b-a}{2b}} > \beta_n^{\frac{3b-2a}{2b}}$ and \(\beta_n^{{\frac{2b-a}{2b}}}\geq\beta_n^{\frac{b+a}{2b}}\), we obtain 
\begin{align}
    \PE^{1/p}[\|R_n^{\operatorname{last}}\|^p] &\lesssim  \frac{p}{2b-1} \prod_{j=0}^n (1-\frac{a_\Delta}{8}\beta_j) +p^4\big(\frac{1}{2b-a-1}\beta_n^{\frac{2b-a/2-1}{b}} +  \beta_n^{\frac{b+a}{2b}}+  \beta_n^{\frac{3b-2a}{2b}}  + \beta_n^{\frac{2b-a}{2b}}\big) \\
    &\lesssim  \frac{p}{2b-1} \prod_{j=0}^n (1-\frac{a_\Delta}{8}\beta_j) + \frac{p^4}{2b-a-1} \beta_n^{\frac{2b-a/2-1}{b}} \eqsp.
\end{align}
Now we set \(p=\log n\eqsp.\) Next, the bound for the first term in \eqref{eq:last_iter_3_term} follows from $(n^{-\alpha})^{\frac{\log n}{1+\log n}} \lesssim n^{-\alpha}$ and $(\sum a_i)^{q} \leq \sum a_i^q$ for $a_i > 0$ and $q \in (0, 1)$:
\begin{equation}
    \PE^{1/(p+1)}\bigl[\norm{\big(\lineG_{\infty}^{\operatorname{last}}\big)^{-1/2} \beta_n^{-1/2} R_n^{\operatorname{last}}}^{p}\bigr] \lesssim \beta_n^{-1/2} \frac{\log n}{2b-1} \prod_{j=0}^n (1-\frac{a_\Delta}{8}\beta_j) +  \beta_n^{\frac{3b-a-2}{2b}} \frac{\log^4 n}{2b-a-1} \eqsp.
\end{equation}
Gathering previous bounds we obtain 
\begin{align}    \kolmogorov\bigl(\beta_n^{-1/2}\ttheta_{n+1}, \mathcal{N}(0,  \lineG^{\operatorname{last}}_{\infty} )\bigr) \lesssim n^{b/2} \frac{\log n}{2b-1} \prod_{j=0}^n (1 - \frac{a_\Delta}{8} \beta_j) + \frac{\log^4 n}{(2b-a-1)n^{\frac{3b-a-2}{2}}} \eqsp. 
\end{align}
\end{proof}

Now we give proofs for Lemmas \ref{lem:delta_bound}-\ref{lem:s_3_bound}.

\begin{proof}[Proof of \Cref{lem:delta_bound}]
    Recall that
    \begin{align}
        \delta_j^{(1)} = A_{12} L_j \bigl\{\Gamma_{0:j-1}^{(1)}\ttheta_0 -\sum_{i=0}^{j-1} \beta_i \Gamma_{i+1:j-1}^{(1)} A_{12} \tw_{i} - \sum_{i=0}^{j-1} \beta_i \Gamma_{i+1:j-1}^{(1)} V_{i+1}  \bigr\} \eqsp.
    \end{align}
    Therefore, easy to see that
    \begin{align}
       \sum_{j=0}^n \beta_j G_{j+1:n}^{(1)} \delta_j^{(1)} = \eqsp &\underbrace{\sum_{j=0}^n \beta_j G_{j+1:n}^{(1)} A_{12} L_j \Gamma_{0:j-1}^{(1)} \ttheta_0}_{T_1} - \underbrace{\sum_{i=0}^{n-1}\sum_{j=i+1}^n \beta_j G_{j+1:n}^{(1)} A_{12} L_j \beta_i \Gamma_{i+1:j-1}^{(1)} A_{12} \tw_i}_{T_2} \\
       &-\underbrace{\sum_{i=0}^{n-1} \sum_{j=i+1}^n \beta_j G_{j+1:n}^{(1)} A_{12} L_j \beta_i \Gamma_{i+1:j-1}^{(1)} V_{i+1}}_{T_3} \eqsp.
    \end{align}
    First, we derive a bound for $T_1$ using Minkowski's inequality
    \begin{align}
        \expep{T_1}{p} \leq \kappa_\Delta \ell_\infty \norm{A_{12}}  \norm{\ttheta_0} (1-\frac{a_\Delta} {2} \beta_0)^{-1} P_{0:n}^{(1)} \sum_{j=0}^{n} \frac{\beta_j^2}{\gamma_j} \lesssim (2b-a-1)^{-1} P_{0:n}^{(1)} \eqsp,
    \end{align}
    where the last transition employs \Cref{assum:last_iter} and integral bound
    \begin{align}
        \sum_{j=0}^{n}\frac{\beta_j^2}{\gamma_j} = \frac{c_{0,\beta}^2}{c_{0,\gamma}}\sum_{j=0}^{n}(j+k_0)^{a-2b}\leq \frac{c_{0,\beta}^2}{c_{0,\gamma}}\frac{k_0^{a+1-2b}}{2b-a-1} \leq \frac{c_{0,\beta}^2}{c_{0,\gamma}(2b-a-1)}\eqsp.
    \end{align}
    We conclude that \(\expep{T_1}{p}\lesssim (2b-a-1)^{-1}P_{0:n}^{(1)}\eqsp.\)
    Since $V_{i+1}$ is a martingale difference sequence, Burkholder's inequality \cite[Therorem 8.1]{osekowski} implies 
    \begin{align}
        \PE^{2/p}[\norm{T_3}^p] \leq p^2 \sum_{i=0}^{k-1} \beta_i^2 \big\|\sum_{j=i+1}^n \beta_j G_{j+1:n}^{(1)} A_{12} L_j \Gamma_{i+1:j-1}^{(1)}\big\|^2 \PE^{2/p}[\|V_{i+1}\|^p]  \eqsp.
    \end{align}
    On the other hand, bounding the term \(\sum_{j=i+1}^{n}\beta_j^2/\gamma_j\) via integral estimates gives
    \begin{align}
        \norm{\sum_{j=i+1}^n \beta_j G_{j+1:n}^{(1)} A_{12} L_j \Gamma_{i+1:j-1}^{(1)}} &\leq \frac{\kappa_\Delta \norm{A_{12}} \ell_\infty}{1-\frac{a_\Delta}{2}\beta_0} P_{i+1:n}^{(1)} \sum_{j=i+1}^n \frac{\beta_j^2}{\gamma_j} \leq \frac{\kappa_\Delta \norm{A_{12}} \ell_\infty c_{0,\beta}^2 (i+k_0)^{a+1-2b}}{(1-\frac{a_\Delta}{2}\beta_0)c_{0,\gamma}(2b-a-1)}P_{i+1:n}^{(1)}\\
        &\lesssim \beta_i^{\frac{2b-a-1}{b}}(2b-a-1)^{-1}P_{i+1:n}^{(1)}\eqsp.
    \end{align}
    Note that $\PE^{1/p}[\norm{V_{i+1}}^p] \lesssim p^3$ due to \Cref{prop:moments_bound}. Thus,  \Cref{lem:summ_alpha_k}-\ref{lem:summ_alpha_k_p_item} implies the following bound
    \begin{align}
        \PE^{2/p}[\norm{T_3}^p] &\lesssim  p^8 (2b-a-1)^{-2} \sum_{i=0}^{n-1} \beta_i^{2+\frac{2(2b-a-1)}{b}} P_{i+1:n}^{(1)} \lesssim  (2b-a-1)^{-2}p^8\beta_n^{1+\frac{2(2b-a-1)}{b}} \eqsp.
    \end{align}
    Now we will get a bound for $\expep{T_2}{p}$. Minkowski's inequality and \Cref{prop:moments_bound} imply that
    \begin{align}
        \expep{T_2}{p} &\leq \frac{\kappa_\Delta \ell_\infty \norm{A_{12}}^2}{1-\frac{a_\Delta}{2}\beta_0} \sum_{i=0}^{n-1} \beta_i P_{i+1:n}^{(1)} \sum_{j=i+1}^n \frac{\beta_j^2}{\gamma_j} (C_{0}^{\tw}P_{0:i-1}^{(2)} + \widetilde{\ConstC}_{\operatorname{fast}}p^3\sqrt{\gamma_{i-1}}) \\
        &\lesssim (2b-a-1)^{-1}\sum_{i=0}^{n-1}\beta_i^{1 + \frac{2b-a-1}{b}} P_{i+1:n}^{(1)}(P_{0:i-1}^{(2)} + p^3\sqrt{\gamma_{i-1}})\eqsp.
    \end{align}
    Using \Cref{assum:stepsize} and \Cref{convolution_inequality}-\ref{convolution_inequality_beta}  we get
    \begin{align}
       \sum_{i=0}^{n-1} \beta_i^{1+\frac{2b-a-1}{b}} P_{i+1:n}^{(1)} P_{0:i-1}^{(2)}&\leq \sum_{i=0}^{n-1} \beta_0 P_{i+1:n}^{(1)} P_{0:i-1}^{(2)} \lesssim P_{0:k}^{(1)}\eqsp.
    \end{align}
    and
    \begin{align}
        \sum_{i=0}^{n-1} \beta_{i}^{1+\frac{2b-a-1}{b}} P_{i+1:n}^{(1)} \sqrt{\gamma_{i-1}} &\lesssim   \sum_{i=0}^{k-1} \beta_i^{1+\frac{2b-a/2-1}{b}} P_{i+1:n}^{(1)} \lesssim \beta_n^{\frac{2b-a/2-1}{b}}  \eqsp. 
    \end{align}
    Hence,
    \begin{align}
        \expep{T_2}{p} \lesssim  (2b-a-1)^{-1}\big(p^3\beta_n^{\frac{2b-a/2-1}{b}} + P_{0:n}^{(1)}\big)\eqsp.
    \end{align}
    We finish the proof applying Minkowski's inequality
    \begin{align}
        \expep{\sum_{j=0}^n \beta_j G_{j+1:n}^{(1)} \delta_j^{(1)}}{p} &\leq \sum_{i=1}^3 \expep{T_i}{p} \lesssim (2b-a-1)^{-1}\big(P_{0:n}^{(1)} + p^4\beta_n^{\frac{2b-a-1}{b}}\big)\eqsp.
    \end{align}
\end{proof}

\begin{proof}[Proof of \Cref{lem:h_bound}]
Note that since $V_{k+1}$, $W_{k+1}$ and $\eps_V$, $\eps_W$ are martingale difference sequences, $R^{\theta}_{k+1}\ttheta_{k+1} + R^w_{k+1} (w_k-w^\star)$ is a martingale difference sequence. Recall that 
\[H_n = \sum_{j=0}^{n}\beta_jG_{j+1:n}^{(1)}R_{j+1}^{\theta}\ttheta +\sum_{j=0}^{n}\beta_jG_{j+1:n}^{(1)}R_{j+1}^{w}(w_j - w^\star) \eqsp .\]
Note that $\norm{R^\theta_{k+1}} \leq q_R$ and $\norm{R^w_{k+1}} \leq q_R$, where \(q_R\) is defined in \eqref{eq:q_R_definition}.
Burkholder's inequality \cite[Theorem 8.1]{osekowski} implies that
\begin{align}
    \PE^{1/p}[\|H_n\|^p] \leq p \biggl( \sum_{j=0}^n \beta_j^2 P_{j+1:n}^{(1)}  \PE^{2/p}[\|R_{j+1}^\theta \ttheta_j\|^p] \biggr)^{1/2} +p \biggl( \sum_{j=0}^n \beta_j^2 P_{j+1:n}^{(1)}  \PE^{2/p}[\|R_{j+1}^\theta (w_j-w^\star)\|^p] \biggr)^{1/2} \eqsp.
\end{align} 
 Hence, we get applying \Cref{prop:moments_bound} and \Cref{lem:summ_alpha_k}-\ref{lem:summ_alpha_k_p_item}:
\begin{align}
    \sum_{j=0}^n \beta_j^2 P_{j+1:n}^{(1)} \PE^{2/p}[\|R_{j+1}^\theta \ttheta_j\|^p] &\lesssim  \bigl(q_R\ConstC_0^{\ttheta} \bigr)^2\sum_{j=0}^n \beta_j^2 P_{j+1:n}^{(1)} \prod_{t=0}^{j-1}(1-\frac{a_\Delta}{4}\beta_t) + p^4 \bigl(q_R C_{\operatorname{slow}} \bigr)^2\sum_{j=0}^n \beta_{j}^3 P_{j+1:n}^{(1)} \\
    &\lesssim (2b-1)^{-1}  \prod_{j=0}^n (1-\frac{a_\Delta}{4}\beta_j) + p^4 \beta_n^2 \eqsp.
\end{align}
Now we use \Cref{lem:martingale:w_minus_wstar_moment_bound} and obtain the similar bound
\begin{align}
    \sum_{j=0}^n \beta_j^2 P_{j+1:n}^{(1)} (\expep{R_{j+1}^w (w_j-w^\star)}{p})^2 &\lesssim \bigl(q_R\widehat{\ConstC}_0^{\tw}\bigr)^2 \sum_{j=0}^{n}\beta_j^2 P_{j+1:n}^{(1)}\prod_{t=0}^{j-1} (1-\frac{a_\Delta}{4}\beta_t) +  + p^6 \bigl(q_R \widehat{\ConstC}_{\operatorname{fast}}\bigr)^2 \sum_{j=0}^{n}\beta_j^{2 + a/b}P_{j+1:n}^{(1)}\\
    &\lesssim (2b-1)^{-1} \prod_{j=0}^n (1-\frac{a_\Delta}{4}\beta_j) + p^6\beta_n^{\frac{a+b}{b}}\eqsp.
\end{align}
Square-root operation followed by dominant-term selection proves the claim.
\end{proof}

\begin{proof}[Proof of \Cref{lem:s_1_bound}]
    First, rewrite $\delta_i^{(2)}$ as follows:
    \begin{align}
        \delta_i^{(2)} &= -(L_{i+1} + A_{22}^{-1} A_{21}) A_{12} \bigl(\Gamma_{0:i-1}^{(2)} \tw_0 - \sum_{t=0}^{i-1} \Gamma_{t+1:i-1}^{(2)} \xi_{t+1}\bigr) \eqsp.
    \end{align}
    Thus, we get
    \begin{align}
        S_n^{(1)} = &-\underbrace{\sum_{j=0}^n \sum_{i=0}^{j-1} \beta_j G_{j+1:n}^{(1)} A_{12} \beta_i G_{i+1:j-1}^{(2)} (L_{i+1} + A_{22}^{-1} A_{21}) A_{12} \Gamma_{0:i-1}^{(2)} \tw_0}_{T_1} \\
        &+\underbrace{\sum_{j=0}^n \sum_{i=0}^{j-1} \sum_{t=0}^{i-1} \beta_j G_{j+1:n}^{(1)} A_{12} \beta_i G_{i+1:j-1}^{(2)} (L_{i+1} + A_{22}^{-1} A_{21}) A_{12} \Gamma_{t+1:i-1}^{(2)} \xi_{t+1}}_{T_2} \eqsp. 
    \end{align}
    First, we derive a bound for $T_1$ 
    \begin{align}
        \expep{T_1}{p} = \norm{T_1} \leq \kappa_{22}\sqrt{\kappa_\Delta} \norm{A_{12}} (\ell_\infty + \norm{A_{22}^{-1} A_{21}}) \norm{A_{12}} \norm{\tw_0} \sum_{i=0}^{n-1} \sum_{j=i+1}^{n-1} \beta_i \beta_j P_{j+1:n}^{(1)} P_{i+1:j-1}^{(2)} P_{0:i-1}^{(2)} \eqsp,
    \end{align}
    and using \Cref{convolution_inequality}-\ref{convolution_inequality_beta} and \(\beta_j\leq \beta_i\) when \(j < i\) we get
    \begin{align}
        \sum_{i=0}^{n-1} \sum_{j=i+1}^{n-1} \beta_i  \beta_j P_{j+1:n}^{(1)} P_{i+1:j-1}^{(2)} P_{0:i-1}^{(2)} &\leq \sum_{i=0}^{n-1}\beta_i P_{0:i-1}^{(2)}  \sum_{j=i+1}^{n-1} \beta_i P_{j+1:n}^{(1)} P_{i+1:j-1}^{(2)} \leq  \ConstC_{\beta}^P\sum_{i=0}^{n-1} \beta_i \gamma_i^{(b-a)/a} P_{i+1:n}^{(1)} P_{0:i-1}^{(2)}\\ & \lesssim \ConstC_{\beta}^P\sum_{i=1}^{n-1} \beta_0  P_{i+1:n}^{(1)} P_{1:i-1}^{(2)} \lesssim P_{0:n}^{(1)}\eqsp.
    \end{align}
    Collecting these results yields the bound \(\expep{T_1}{p} \lesssim P_{0:n}^{(1)}\eqsp.\)
    To write a bound for $T_2$ we change the order of summation
    \begin{align}
        T_2 = \sum_{t=0}^{n-2} \underbrace{\bigl(\sum_{i=t+1}^{n-1} \sum_{j=i+1}^n \beta_j G_{j+1:n}^{(1)} A_{12} \beta_i G_{i+1:j-1}^{(2)} (L_{i+1} + A_{22}^{-1} A_{21}) A_{12} \Gamma_{t+1:i-1}^{(2)}\bigr)}_{U_t} \xi_{t+1} \eqsp,
    \end{align}
    and combine Burkholder's inequality \cite[Theorem 8.1]{osekowski} with $\PE^{1/p}[\norm{\xi_{t+1}}^p] \lesssim p^3 \gamma_t$:
    \begin{align}
       \expep{T_2}{p} \lesssim p^4  \biggl(\sum_{t=0}^{k-2} \norm{U_t}^2\gamma_{t}^2\biggr)^{1/2} \eqsp.
    \end{align}
    Now we use  \Cref{convolution_inequality}-\ref{convolution_inequality_beta} and get
    \begin{align}
        \norm{U_t} &\lesssim \sum_{i=t+1}^{n-1} \sum_{j=i+1}^n \beta_j \beta_i P_{j+1:n}^{(1)} P_{i+1:j-1}^{(2)} P_{t+1:i-1}^{(2)} \lesssim  \sum_{i=t+1}^{n-1} \beta_i \gamma_{i}^{(b-a)/a} P_{i+1:n}^{(1)} P_{t+1:i-1}^{(2)} \lesssim P_{t+1:n}^{(1)} \gamma_{t}^{2(b-a)/a} \eqsp.
    \end{align}
    Note that due to \Cref{lem:summ_alpha_k}-\ref{lem:summ_alpha_k_p_item}
    \begin{align}
        \sum_{t=0}^{n-2} \gamma_t^{2+4(b-a)/a} P_{t+1:n}^{(1)} \lesssim  \sum_{t=0}^{n-2} \beta_t^{(4b-2a)/b} P_{t+1:n}^{(1)} \lesssim \beta_n^{(3b-2a)/b} \eqsp.
    \end{align}
    The latter inequality yields the required bound.
\end{proof}

\begin{proof}[Proof of \Cref{lem:s_2_bound}]
    First, recall that 
    \begin{equation}
        S_n^{(2)} = \sum_{j=0}^{n}\beta_jG_{j+1:n}^{(1)}A_{12}\sum_{i=0}^{j-1}\beta_iG_{i+1:j-1}^{(2)}D_iV_{i+1}\eqsp.
    \end{equation}
    We change the order of summation and get
    \begin{align}
        S_{n}^{(2)} = \sum_{i=0}^{n-1} \beta_i \bigl(\sum_{j=i+1}^n \beta_j  G_{j+1:n}^{(1)} A_{12} G_{i+1:j-1}^{(2)}\bigr) D_i V_{i+1} \eqsp.
    \end{align}
    Burkholder's inequality \cite[Theorem 8.1]{osekowski} immediately implies that
    \begin{align}
    \expep{S_{n}^{(2)}}{p} \leq p \biggl(\sum_{i=0}^{n-1} \beta_i^2 c_\infty^2 \bigl \|\sum_{j=i+1}^n \beta_j  G_{j+1:n}^{(1)} A_{12}G_{i+1:j-1}^{(2)}\bigr\|^2  \PE^{2/p}[\|V_{i+1}\|^{p}] \biggr)^{1/2} \eqsp.
    \end{align}
    Now we use \Cref{convolution_inequality}-\ref{convolution_inequality_beta} and get
    \begin{align}
        \big\|\sum_{j=i+1}^n \beta_j  G_{j+1:n}^{(1)}A_{12} G_{i+1:j-1}^{(2)}\big\| \lesssim \sum_{j=i+1}^n \beta_j  P_{j+1:n}^{(1)} P_{i+1:j-1}^{(2)} \lesssim \gamma_i^{(b-a)/a} P_{i+1:n}^{(1)} \eqsp,
    \end{align}
    and the desirable result follows from \Cref{lem:summ_alpha_k}-\ref{lem:summ_alpha_k_p_item}:
    \begin{align}
    \expep{S_{n}^{(2)}}{p} \lesssim p \biggl(p^6\sum_{i=0}^{n-1} \beta_i^2 \gamma_i^{2(b-a)/a} P_{i+1:n}^{(1)}  \biggr)^{1/2}\lesssim p^4\biggl(\sum_{i=0}^{n-1}\beta_i^{2 + 2(b-a)/b} P_{i+1:n}^{(1)}\biggr)^{1/2} \lesssim p^{4}\beta_n^{\frac{3b-2a}{2b}}\eqsp.
    \end{align}
\end{proof}

\begin{proof}[Proof of \Cref{lem:s_3_bound}]
Recall that 
\begin{equation}
    S_n^{(3)} = \sum_{j=0}^{n}\beta_jG_{j+1:n}^{(1)}A_{12}\sum_{i=0}^{j-1}\gamma_iG_{i+1:j-1}^{(2)}W_{i+1} - \sum_{j=0}^{n}\beta_jG_{j+1:n}^{(1)}A_{12}A_{22}^{-1}W_{j+1}\eqsp.
\end{equation}
Changing the order of summation, rewrite  \(S_n^{(3)}\) as follows
\begin{align}\label{eq:sk_3_rewrite}
         S_n^{(3)}
         &=\sum_{i=0}^{n}\beta_i G_{i+1:n}^{(1)}\bigl(\frac{\gamma_i}{\beta_i}\sum_{j=i+1}^{n}\beta_j\bigl(G_{i+1:j}^{(1)}\bigr)^{-1}A_{12}G_{i+1:j-1}^{(2)} - A_{12}A_{22}^{-1}\bigr)W_{i+1}\eqsp.
\end{align}
We can rewrite the term inside the brackets in \eqref{eq:sk_3_rewrite} as
\begin{align}
S_n^{(3)} = \sum_{i=0}^n \beta_i G_{i+1:n}^{(1)} \biggl\{ \underbrace{\sum_{j=i+1}^{n}\gamma_j\big(\frac{\gamma_i\beta_j}{\beta_i\gamma_j}\big(G_{i+1:j}^{(1)}\big)^{-1} - \Id\big)A_{12}G_{i+1:j-1}^{(2)}}_{Z_i^{(1)}} &+ \underbrace{A_{12}\big(\sum_{j=i+1}^{n}\gamma_j G_{i+1:j-1}^{(2)} - \int_{0}^{\sum_{j=i+1}^{n}\gamma_j}\exp(-A_{22}t)\;dt\big)}_{Z_i^{(2)}} \\ \underbrace{- A_{12}A_{22}^{-1}\exp\big(-\sum_{j=i+1}^{n}\gamma_j A_{22}\big)}_{Z_i^{(3)}} \biggr\}W_{i+1}\eqsp.
\end{align}
Consider the real valued positive sequence \(\{\epsilon_k\}\) defined by the equations:
\begin{align}
    \frac{\beta_{k+1}}{\gamma_{k+1}} = \frac{\beta_k}{\gamma_k}\big(1 - \epsilon_k\gamma_k\big) \eqsp.
\end{align}
 As shown in \cite{konda:tsitsiklis:2004}, the following estimates hold:
 \begin{align}
 \label{eq:conda_bound}
     \|Z_i^{(1)}\| \lesssim \beta_i/\gamma_i + \epsilon_i, \quad \|Z_i^{(2)}\| \lesssim \gamma_i\eqsp.
 \end{align}
  \cite{godunov1997modern} implies that:
 \begin{equation}
     \|Z_i^{(3)}\| \leq \sqrt{\pw}\exp\big(-\frac{1}{2\|Q_{22}\|}\sum_{j=i+1}^{n}\gamma_j\big)\eqsp.
 \end{equation}
Using \Cref{assum:stepsize}, easy to see that:
\begin{align}
    \epsilon_k = \gamma_k^{-1} - \beta_{k+1}\beta_k^{-1}\gamma_{k+1}^{-1} \leq \gamma_{k+1}^{-1}\beta_k^{-1}(\beta_k - \beta_{k+1}) \lesssim \beta_k^{-a/b} \beta_k^{-1} \beta_k^{2} = \beta_k^{1-a/b} \eqsp.
\end{align}
Thus, $\norm{Z_i^{(1)}} \lesssim \beta_i^{1-a/b}$. Now we use $\PE^{1/p}[\norm{W_{i+1}}^p] \lesssim p^3$ together with Burkholder’s inequality \cite[Theorem 8.1]{osekowski} and get:
\begin{align}\label{eq:sk3_bh}
    \PE^{2/p}[\|S_n^{(3)}\|^{p}] 
    &\lesssim p^8 \sum_{i=0}^{n}\beta_i^2P_{i+1:n}^{(1)}\big(\|Z_i^{(1)}\|^2 + \|Z_i^{(2)}\|^2 + \|Z_i^{(3)}\|^2\big) \\
    &\lesssim p^8\biggl(\sum_{i=0}^{n}\beta_i^4/\gamma_i^2P_{i+1:n}^{(1)} + \sum_{i=0}^{n}\beta_i^2 \beta_i^{2-\frac{2a}{b}} P_{i+1:n}^{(1)} + \sum_{i=0}^{n}\beta_i^2\gamma_i^2P_{i+1:n}^{(1)} +  \sum_{i=0}^{n}\beta_i^2\exp\big(-\frac{1}{2}{\|Q_{22}\|^{-1}}\sum_{j=i+1}^{n}\gamma_j\big)P_{i+1:n}^{(1)} \biggr)\eqsp.
\end{align}
Rewriting \(\gamma_i\) in terms of \(\beta_i\) and applying \Cref{lem:summ_alpha_k}-\ref{lem:summ_alpha_k_p_item} together with \Cref{assum:stepsize} we get:
\begin{align}
    \PE^{2/p}[\|S_n^{(3)}\|^{p}]& \lesssim p^8 \bigl(\beta_n^{\frac{3b-2a}{b}} + \beta_n^{\frac{3b-2a}{b}} + \beta_n^{\frac{b+2a}{b}} + \beta_n^{\frac{2b-a}{b}}\bigr) \lesssim p^8 \beta_n^{\frac{2b-a}{b}} \eqsp.
\end{align}
\end{proof}

\section{Markov noise}
\label{appendix:markov_case}
\subsection{High-order moment bounds}
We preface this section with a brief reminder of notation used in the Markov chains literature. For a Markov kernel $\MKQ$ on $(\Xset,\Xsigma)$, and a measurable function $f: \Xset \to \rset$, we set
\[
\MKQ f(x) = \int_{\Xset} f(y) \MKQ(x,\rmd y)\eqsp.
\]
Define also total variation distance $\tvdist(\mu, \nu)$ for probability measures $\mu, \nu$:
\begin{align}
    \tvdist(\mu, \nu) = \sup_{\|f\|_{\infty} \leq 1} |\mu(f) - \nu(f)| \eqsp.
\end{align}
\Cref{assum:UGE} ensures that $\MKQ$ is uniformly geometrically ergodic and, moreover, for all $k$ it holds that
\begin{equation}
\label{eq:tau_mix_contraction_appendix}
\dobrush(\MKQ^{k}) := \sup_{x,x' \in \Xset} \tvdist(\MKQ^{k}(x,\cdot),\MKQ^{k}(x',\cdot)) \leq (1/4)^{\lceil k / \taumix \rceil}\eqsp,
\end{equation}
where $\taumix \in \mathbb{N}$ is the mixing time that controls the rate of convergence to the stationary distribution.

We proceed with the proof based on the Poison decomposition, following \cite{kaledin2020finite}. Note that under \Cref{assum:UGE} the Poisson equation, associated with $\MKQ$, that 
\begin{equation}
\label{eq:pois_solution_equation}
\pois{f}{}(x) - \MKQ \pois{f}{}(x) = f(x) - \pi(f)\eqsp, \quad x \in \Xset\eqsp,
\end{equation}
has a unique solution for any bounded measurable $f$, which is given by the formula 
\[
\pois{f}{}(x) = \sum_{k=0}^{\infty}\{\MKQ^{k}f(x) - \pi(f)\}\eqsp.
\]
Moreover, using \Cref{assum:UGE} and the inequality \eqref{eq:tau_mix_contraction_appendix}, one can show that $\pois{f}{}$ is also bounded with
\begin{align}
\label{eq:markov:norm_bound_tmix}
\| \pois{f}{}\|_{\infty} \leq \sum_{k=0}^{+\infty} \sup_{x \in \Xset} \|\MKQ^{k}f(x) - \pi(f)\|_{\infty} \leq 2\|f\|_{\infty} \sum_{k=0}^{+\infty} (1/4)^{\lfloor k/\taumix \rfloor} \leq (8/3) \taumix \| f\|_{\infty}\eqsp.
\end{align}
Throughout this chapter, we use a shorthand notation
\begin{equation}
\label{eq:pois_solution_notation}
\pois{f}{k} := \pois{f}{}(X_k) \eqsp.
\end{equation}
We use the above notations for the solution to Poisson equation with different vector- and matrix-valued functions in the equation \eqref{eq:pois_solution_equation}. To proceed with the proof, we follow the idea of \cite{kaledin2020finite}, where the authors have obtained similar results for the $2$nd moment bounds. The main idea is to decompose the TTSA updates $\theta_k$ and $w_k$ into a sum of two coupled TTSA recursions. Namely, $\theta_k = \theta_k^{(0)} + \theta_k^{(1)}$ and $w_k = w_k^{(0)} + w_k^{(1)}$, where 
\begin{align}
\label{def:markov:extended_recursion}
\begin{cases}
\theta_{k+1}^{(0)} &= \theta_k^{(0)} + \beta_k (b_1 -A_{11} \theta_k^{(0)} - A_{12} w_k^{(0)} + V_{k+1}^{(0)}) \eqsp, \theta_0^{(0)} = \theta_0 \eqsp, \\
w_{k+1}^{(0)} &= w_k^{(0)} + \gamma_k (b_2 -A_{21} \theta_k^{(0)} - A_{22} w_k^{(0)} + W_{k+1}^{(0)}) \eqsp, w_0^{(0)} = w_0 \eqsp, \\
\end{cases}
\end{align}
and 
\begin{align}
\begin{cases}
\label{def:markov:extended_recursion_remainder}
\theta_{k+1}^{(1)} &= \theta_k^{(1)} - \beta_k (A_{11} \theta_k^{(1)} + A_{12} w_k^{(1)} - V_{k+1}^{(1)}) \eqsp, \theta_0^{(1)} = 0 \eqsp, \\
w_{k+1}^{(1)} &= w_k^{(1)} - \gamma_k (A_{21} \theta_k^{(1)} + A_{22} w_k^{(1)} - W_{k+1}^{(1)}) \eqsp, w_0^{(1)} = 0\eqsp.
\end{cases}
\end{align}
In the above recursions the noise variables $V_{k}^{(0)}, V_{k}^{(1)}, W_{k}^{(0)}, W_{k}^{(1)}$ are defined as follows:
\begin{align}
\label{def:markov:vw_01}
V_{k+1}^{(0)} &= \{\pois{\funnoisew_V}{k+1} - \MKQ \pois{\funnoisew_V}{k}\} - \{ \poisA{11}{k+1} - \MKQ \poisA{11}{k}\} (\theta_k - \thetas) - \{ \poisA{12}{k+1} - \MKQ \poisA{12}{k}\} (w_k - w^\star) \eqsp, \\
W_{k+1}^{(0)} &= \{\pois{\funnoisew_W}{k+1} - \MKQ \pois{\funnoisew_W}{k}\} - \{ \poisA{21}{k+1} - \MKQ \poisA{21}{k}\} (\theta_k - \thetas) - \{ \poisA{22}{k+1} - \MKQ \poisA{22}{k}\} (w_k - w^\star) \eqsp, \\
V_{k+1}^{(1)} &= \{\MKQ \pois{\funnoisew_V}{k} - \MKQ \pois{\funnoisew_V}{k+1}\} + \{\MKQ \poisA{11}{k+1} - \MKQ \poisA{11}{k}\}(\theta_k - \thetas) + \{\MKQ \poisA{12}{k+1} - \MKQ \poisA{12}{k}\}(w_k - w^\star) \eqsp, \\
W_{k+1}^{(1)} &= \{\MKQ \pois{\funnoisew_W}{k} - \MKQ \pois{\funnoisew_W}{k+1}\} + \{\MKQ \poisA{21}{k+1} - \MKQ \poisA{21}{k}\}(\theta_k - \thetas) + \{\MKQ \poisA{22}{k+1} - \MKQ \poisA{22}{k}\}(w_k - w^\star) \eqsp.
\end{align}
It is easy to see that $\CPE{V_{k+1}^{(0)}}{\F_k} = 0$ and $\CPE{W_{k+1}^{(0)}}{\F_k} = 0$ $\PP$-a.s. Similar to \eqref{eq:tilde_params}, we do a change of variables and define 
\begin{align}
\label{eq:markov:recurrence}
\begin{cases}
\ttheta_k^{(0)} &= \theta_k^{(0)} - \thetas \eqsp , \\ 
\tw_k^{(0)} &= w_k^{(0)} - w^\star + D_{k-1} \ttheta_k^{(0)} \eqsp, \\
\end{cases}
\qquad 
\begin{cases}
\ttheta_k^{(1)} &= \theta_k^{(1)} \eqsp , \\
\tw_k^{(1)} &= w_k^{(1)} + D_{k-1} \ttheta_k^{(1)} \eqsp.
\end{cases}
\end{align}
It is easy to notice that $\ttheta_k = \ttheta_k^{(0)} + \ttheta_k^{(1)}$ and $\tw_k = \tw_k^{(0)} + \tw_k^{(1)}$. Introduce the following notation
\begin{align}
\xi_{k+1}^{(i)} = \gamma_k W_{k+1}^{(i)} + \beta_k D_k V_{k+1}^{(i)} \eqsp, \eqsp i \in \{0, 1\} \eqsp.
\end{align}
Now we prove the lemma, which is a direct counterpart to \Cref{assum:bound-conditional-moments-p}, previously obtained under a martingale noise assumption.
\begin{lemma}
\label{lem:markov:v_and_w_bounds}
Let $p \geq 2 $. Assume \Cref{assum:hurwitz}, \Cref{assum:aij_bound},  \Cref{assum:UGE}, \Cref{assum:markov:stepsize}($p$). Then, for $i \in \{0,1\}$ and any $k \in \nset \cup \{0\}$ it holds that
\begin{align}
\PE^{1/p}[\norm{V_{k+1}^{(i)}}^p] &\lesssim 1 + M_{k,p}^{\ttheta} + M_{k,p}^{\tw} \eqsp, \\ 
\PE^{1/p}[\norm{W_{k+1}^{(i)}}^p] &\lesssim 1 + M_{k,p}^{\ttheta} + M_{k,p}^{\tw} \eqsp, \\
\PE^{1/p}[\norm{\xi_{k+1}^{(0)}}^p] &\lesssim \gamma_k (1 + M_{k,p}^{\ttheta} + M_{k,p}^{\tw}) \eqsp.
\end{align}
\end{lemma}
\begin{proof}
\Cref{assum:aij_bound} implies that
\[
\norm{\funnoisew_V} \leq \bConst{b} + \bConst{A} \norm{\thetas} + \bConst{A} \norm{\wstar} \text{ and } \norm{\funnoisew_W} \leq \bConst{b} + \bConst{A} \norm{\thetas} + \bConst{A} \norm{\wstar} \eqsp.
\]
It remains to note that, due to construction of $\tw_k$ in \eqref{eq:tilde_params}, it holds that $\PE^{1/p}[\norm{w_k - \wstar}^p] \leq M_{k,p}^{\tw} + c_\infty M_{k,p}^{\ttheta}$. Then it remains to gather similar terms and apply \eqref{eq:markov:norm_bound_tmix}.
\end{proof}

To prove \Cref{prop:markov:moment_bounds} we first state the counterparts to \Cref{prop:tw_bond} and \Cref{prop:mth_closed_bound}:
\begin{proposition}
\label{lem:markov:tw0_bound}
Let $p \geq 2 $. Assume \Cref{assum:hurwitz}, \Cref{assum:aij_bound},  \Cref{assum:UGE}, \Cref{assum:markov:stepsize}($p$). Then it holds that
\begin{align}
\label{eq:markov:mw_first_bound}
(M_{k+1, p}^{\tw})^2 \lesssim  P_{0:k}^{(2)} + p^2 \gamma_k + p^2 \sum_{j=0}^k \gamma_j^2 P_{j+1:k}^{(2)} (M_{j, p}^{\ttheta})^2 \eqsp.
\end{align}
\end{proposition}

\begin{proposition}
\label{lem:markov:ttheta0_bound}
Let $p \geq 2$. Assume \Cref{assum:hurwitz}, \Cref{assum:aij_bound},  \Cref{assum:UGE}, \Cref{assum:markov:stepsize}($p$). Then it holds that
\begin{align}
(M_{k+1, p}^{\ttheta})^2 \lesssim P_{0:k}^{(1)} + p^4  \beta_k + p^4 \sum_{j=0}^k \beta_j^2 P_{j+1:k}^{(1)} (M_{j,p}^{\ttheta})^2 \eqsp.
\end{align}
\end{proposition}

\begin{proof}[Proof of \Cref{prop:markov:moment_bounds}.]

First, we proceed with the bound for $M_{k,p}^{\ttheta}$. Using \Cref{lem:markov:ttheta0_bound} we get
\begin{align}
(M_{k+1, p}^{\ttheta})^2 \lesssim  P_{0:k}^{(1)} + p^4 \beta_k + p^4 \sum_{j=0}^k \beta_j^2 P_{j+1:k}^{(1)} (M_{j,p}^{\ttheta})^2 \eqsp.
\end{align}
Hence, there exists a constant $C_{\ttheta}$ such that
\begin{align}
(M_{k+1, p}^{\ttheta})^2 \leq C_{\ttheta} P_{0:k}^{(1)} + p^4 C_{\ttheta} \beta_k + p^4 C_{\ttheta} \sum_{j=0}^k \beta_j^2 P_{j+1:k}^{(1)} (M_{j,p}^{\ttheta})^2 \eqsp.
\end{align}
Denote the right hand side of the latter inequality by $U_{k+1}$, $U_0 = C_{\ttheta}$. Thus, since $(M_{j,p}^{\ttheta})^2 \leq U_j$ for all $j \geq 0$, we get
\begin{align}
U_{k+1} &\leq (1 - \frac{a_\Delta \beta_k}{2})U_k + p^4 C_{\ttheta} \beta_k - p^4 C_{\ttheta} \beta_{k-1} (1 - \frac{a_\Delta \beta_k}{2}) + p^4 \beta_k^2 C_{\ttheta} U_{k} \overset{(a)}{\leq} (1 - \frac{a_\Delta \beta_k}{4})U_k + p^{4} a_\Delta C_{\ttheta} \beta_k^2 \eqsp,
\end{align}
where in $(a)$ we used \Cref{assum:markov:stepsize}. Hence, enrolling the latter recursion and applying \Cref{lem:summ_alpha_k}, it is easy to get
\begin{align}
\label{eq:markov:mtheta_final_bound}
(M_{k+1,p}^{\ttheta})^2 \leq U_{k+1} \lesssim \prod_{j=0}^k (1 - \frac{a_\Delta \beta_j}{4}) + p^4 \beta_k \eqsp,
\end{align}
and, 
$$
M_{k+1,p}^{\ttheta} \lesssim  \prod_{j=0}^k (1 - \frac{a_\Delta \beta_j}{8}) + p^2  \sqrt{\beta_k} \eqsp,
$$
Now we substitute \eqref{eq:markov:mtheta_final_bound} to \eqref{eq:markov:mw_first_bound}, apply \Cref{lem:summ_alpha_k} and get
\begin{align}
(M_{k+1,p}^{\tw})^2 &\lesssim P_{0:k}^{(2)} + p^2 \gamma_k + p^2  \sum_{j=0}^k \gamma_j^2 P_{j+1:k}^{(2)} \prod_{t=0}^{j-1} (1 - \frac{a_\Delta}{4} \beta_t) +  p^6 \sum_{j=0}^k \gamma_j^2 \beta_{j-1} P_{j+1:k}^{(1)} \eqsp, \\
&\lesssim P_{0:k}^{(2)} + p^2 \gamma_k + p^2 \beta_k + p^6  \beta_k \eqsp,
\end{align}
and the proof follows.
\end{proof}
Thus, the following bound holds for the initial fast scale:
\begin{lemma}
\label{lem:markov:w_hat_moment}
Let $p \geq 2$. Assume \Cref{assum:hurwitz}, \Cref{assum:aij_bound},  \Cref{assum:UGE}, \Cref{assum:markov:stepsize}($p$). Then it holds for all $k \in \nset$ that
\begin{align}
\PE^{1/p}[\norm{\what_{k+1}}^p] \lesssim \prod_{j=0}^k (1 - \frac{a_\Delta}{8} \beta_j) + p^3 \gamma_k \eqsp, \text{ where } \what_{k+1} = w_{k+1} - \wstar \eqsp.
\end{align}
\end{lemma}
\begin{proof}
Note that \Cref{assum:markov:stepsize} guarantees that $\frac{a_\Delta}{8} \beta_j \leq \frac{a_{22}}{8} \gamma_j$. Thus, the proof follows from \Cref{prop:markov:moment_bounds}. 
\end{proof}
For completeness we also provide the following technical lemma, which is proved in \cite{kaledin2020finite}. 
\begin{lemma}[Lemma 11 in \cite{kaledin2020finite}]
\label{lem:representation_lemma:kaledin}
Let $(a_j)_{j \geq 0}$ be a sequence of $d_\theta$-dimensional vectors and $(b_j)_{j \geq 0}$ be a sequence of $d_w$-dimensional vectors. Then it holds that
\begin{align}
\sum_{j=0}^{k} \beta_j \Gamma_{j+1:k}^{(1)} (a_j -a_{j+1}) &= \beta_0 \Gamma_{1:k}^{(1)} a_0 -\beta_k a_{k+1} + \sum_{j=1}^k (\beta_j^2 B_{11}^j \Gamma_{j+1:k}^{(1)} + (\beta_j - \beta_{j-1}) \Gamma_{j:k}^{(1)})  a_j \eqsp, \\
\sum_{j=0}^{k} \gamma_j \Gamma_{j+1:k}^{(2)} (b_j -b_{j+1}) &= \gamma_0 \Gamma_{1:k}^{(2)} b_0 -\gamma_k b_{k+1} + \sum_{j=1}^k (\gamma_j^2 B_{22}^j \Gamma_{j+1:k}^{(2)} + (\gamma_j - \gamma_{j-1}) \Gamma_{j:k}^{(2)}) b_j \eqsp,
\end{align}
\end{lemma}

\begin{proof}[Proof of \Cref{lem:markov:tw0_bound}] ~\ 
First, since $\tw_{k+1} = \tw_{k+1}^{(0)} + \tw_{k+1}^{(1)}$, we get
\begin{align}
\PE^{2/p}[\norm{\tw_{k+1}}^p] \leq 2 \PE^{2/p}[\norm{\tw_{k+1}^{(0)}}^p] + 2 \PE^{2/p}[\norm{\tw_{k+1}^{(1)}}^p] \eqsp.
\end{align}
From now on, we provide bounds for $\PE^{2/p}[\norm{\tw_{k+1}^{(0)}}^p]$ and $\PE^{2/p}[\norm{\tw_{k+1}^{(1)}}^p]$ separately.
\paragraph{(I) Bound on $\PE^{2/p}[\norm{\tw_{k+1}^{(0)}}^p]$.}
\, \\
\noindent First, we derive a bound for $\PE^{1/p}[\norm{\tilde{w}_{k+1}^{(0)}}^p]$. Since $V_{k+1}^{(0)}$ and $W_{k+1}^{(0)}$ are martingale-difference sequences w.r.t. $\F_k$, we obtain, following the lines of \Cref{prop:tw_bond}, that 
\begin{equation}
\label{eq:markov:mth_reccurence_bound}
\PE^{2/p}[\|\tw_{k+1}^{(0)}\|^p] \lesssim  \bigl(P_{0:k}^{(2)}\bigr)^2 + p^2 \sum_{j=0}^{k}\gamma_j^2\bigl(P_{j+1:k}^{(2)}\bigr)^2\biggl(1 + \bigl(M_{j,p}^{\ttheta}\bigr)^2 + \bigl(M_{j,p}^{\tw}\bigr)^2\biggr)\eqsp.
\end{equation}

\paragraph{(II) Bound on $\PE^{2/p}[\norm{\tw_{k+1}^{(1)}}^p]$.}
\, \\
Let us introduce the notation 
\begin{align}
E_j = \frac{\beta_j}{\gamma_j} D_j \eqsp.
\end{align}
Thus, a counterpart to \eqref{eq:w_rec} with initial condition $\tw_0^{(1)} = 0$, we get 
\begin{align}
\label{eq:markov:w1_rec}
\Tilde{w}_{k+1}^{(1)} = - \sum_{j=0}^{k}\Gamma_{j+1:k}^{(2)} \gamma_j (W_{j+1}^{(1)} + E_j V_{j+1}^{(1)}) \eqsp.
\end{align}
From now on for any vector or matrix sequence $\{S_i\}_{i \geq 0}$ we let $S_i = 0$ if $i < 0$. From now on, using algebraic manipulations and recursion \eqref{def:markov:extended_recursion}, we get, following \cite[Derivation of Eq. (63), p. 39]{kaledin2020finite}, that for any $j \geq 0$:
\begin{align}
\label{eq:markov:tw:w_plus_v_representation}
W_{j+1}^{(1)} + E_j V_{j+1}^{(1)} = \Psi_{j+1} - \Psi_{j} &+ \Phi^{(1)}_j (\ttheta_{j+1} - \ttheta_{j}) + \Phi^{(2)}_j (\ttheta_{j} - \ttheta_{j-1}) 
\\ &+ \Xi^{(1)}_j (\tw_{j+1} - \tw_{j}) + \Xi^{(2)}_j (\tw_{j} - \tw_{j-1}) 
\\ &+ \Upsilon^{(1)}_{j+1} \ttheta_{j+1} - \Upsilon^{(1)}_j \ttheta_j + \Upsilon^{(2)}_{j} \ttheta_{j} - \Upsilon^{(2)}_{j-1} \ttheta_{j-1} \\
&+ \Lambda^{(1)}_{j+1} \tw_{j+1} - \Lambda^{(1)}_j \tw_j + \Lambda^{(2)}_{j} \tw_{j} - \Lambda^{(2)}_{j-1} \tw_{j-1} \\
&+ \Pi^\theta_j \ttheta_j + \Pi^w_j \tw_j + (E_{j+1} - E_j) (\MKQ \poiseps{V}{j+1})
\eqsp,
\end{align}
where we have defined 
\begin{align}
\Psi_j &= -\MKQ \pois{\funnoisew_W}{j} - E_j (\MKQ \pois{\funnoisew_V}{j}) \eqsp, \\
\Phi_j^{(1)} &= -\MKQ \poisA{21}{j+1} \eqsp, \quad 
\Phi_j^{(2)} = -E_{j-1} (\MKQ \poisA{11}{j}) + (\MKQ \poisA{22}{j}) D_{j-2} + E_{j-1} (\MKQ \poisA{12}{j}) D_{j-2} \eqsp, \\
\Xi^{(1)}_j &= -\MKQ \poisA{22}{j+1} \eqsp, \quad 
\Xi^{(2)}_j = -E_{j-1} (\MKQ \poisA{12}{j}) \eqsp, \\
\Upsilon^{(1)}_j &= \MKQ \poisA{21}{j} \eqsp, \quad 
\Upsilon^{(2)}_j = E_j (\MKQ \poisA{11}{j+1}) - (\MKQ \poisA{22}{j+1}) D_{j-1} - E_j (\MKQ \poisA{12}{j+1}) D_{j-1} \eqsp, \\
\Lambda^{(1)}_j &= \MKQ \poisA{22}{j} \eqsp, \quad 
\Lambda^{(2)}_j = E_j \MKQ \poisA{12}{j+1} \eqsp, \\
\Pi^{\theta}_j &= -(E_{j} - E_{j-1}) (\MKQ \poisA{11}{j}) + (\MKQ \poisA{22}{j})(D_{j-1} - D_{j-2}) + (E_{j} - E_{j-1}) (\MKQ \poisA{12}{j}) D_{j-2} + E_{j} (\MKQ \poisA{12}{j}) (D_{j-1} - D_{j-2})\eqsp, \\
\Pi^{w}_j &= -(E_j - E_{j-1}) (\MKQ \poisA{12}{j}) \eqsp.
\end{align}
\Cref{assum:aij_bound}, \Cref{assum:UGE}, and \Cref{L_converges} imply that for all $j \geq 0$ and $i \in \{0,1\}$, it holds that
\begin{align}
\norm{\Psi_j} \vee \norm{\Phi^{(i)}_j} \vee \norm{\Xi^{(i)}_j} \vee \norm{\Upsilon^{(i)}_j} \vee \norm{\Lambda^{(i)}_j} \lesssim 1  \eqsp.
\end{align}
Now we derive a bound for $\Pi_j^\theta$ and $\Pi_j^w$. First, note that
\begin{equation}
\label{eq:D_difference_bound}
\norm{D_{t+1}-D_{t}} = \norm{L_{t+1} - L_t}  \lesssim \gamma_{t+1}\eqsp,
\end{equation}
Now, using \Cref{L_converges} and assumption \Cref{assum:markov:stepsize},
\begin{align}
\norm{E_{t+1} - E_t} = \norm{\frac{\beta_{t+1}}{\gamma_{t+1}} D_{t+1} - \frac{\beta_t}{\gamma_t} D_t} \lesssim \frac{\beta_0}{\gamma_0} \gamma_{t+1} + \frac{\beta_t - \beta_{t+1}}{\gamma_t} \norm{D_t} \lesssim \gamma_{t+1} \eqsp.
\end{align}
Finally, we get that
\begin{equation}
\label{eq:markov:bound_w:pi_bound}
\norm{\Pi_j^\theta} \vee \norm{\Pi^w_j} \lesssim \gamma_j \eqsp,
\end{equation}
Introduce
\begin{align}
v_{j+1} = \Psi_{j+1} + \Upsilon_{j+1}^{(1)} \ttheta_{j+1} + \Upsilon_{j}^{(2)} \ttheta_{j} + \Lambda_{j+1}^{(1)} \tw_{j+1} + \Lambda_{j}^{(2)} \tw_j \eqsp.
\end{align}
Expanding now the recurrence \eqref{eq:markov:w1_rec} together with representation \eqref{eq:markov:tw:w_plus_v_representation}, we obtain that 
\begin{align}
\tw_{k+1}^{(1)} =  - \underbrace{\sum_{j=0}^k \gamma_j \Gamma_{j+1:k}^{(2)} (v_{j+1} - v_j)}_{T_1} &- \underbrace{\gamma_k\Phi_k^{(1)}(\ttheta_{k+1} - \ttheta_k) -\sum_{j=1}^k \{\gamma_j \Gamma_{j+1:k}^{(2)} \Phi_{j}^{(2)} + \gamma_{j-1} \Gamma_{j:k}^{(2)} \Phi_{j-1}^{(1)}\} (\ttheta_{j} - \ttheta_{j-1})}_{T_2} \\
&-\underbrace{\gamma_k\Xi_k^{(1)}(\tw_{k+1} - \tw_k) - \sum_{j=1}^k \{\gamma_j \Gamma_{j+1:k}^{(2)} \Xi_{j}^{(2)} + \gamma_{j-1} \Gamma_{j:k}^{(2)} \Xi_{j-1}^{(1)}\} (\tw_{j} - \tw_{j-1})}_{T_3} \\
&- \underbrace{\sum_{j=0}^k \gamma_j \Gamma_{j+1:k}^{(2)} \{ \Pi^\theta_j \ttheta_j + \Pi^w_j \tw_j + (E_{j+1} - E_j) (\MKQ \poiseps{V}{j+1}) \}}_{T_4} \eqsp.
\end{align}
Now we estimate the terms $T_1$ to $T_4$ separately. To proceed with $T_1$, we use \Cref{lem:representation_lemma:kaledin} and obtain
\begin{align}
T_1 = -\gamma_0 \Gamma_{1:k}^{(2)} v_0 + \gamma_k v_{k+1} - \sum_{j=1}^k (\gamma_j^2 B_{11}^j \Gamma_{j+1:k}^{(2)} + (\gamma_j - \gamma_{j-1}) \Gamma_{j:k}^{(2)}) v_j \eqsp.
\end{align}
Hence,
\begin{align}
\label{eq:markov:moment_w:t1_first}
\PE^{1/p}[\norm{T_1}^p] &\lesssim P_{1:k}^{(2)} + \gamma_k (1 + M_{k+1,p}^{\ttheta} + M_{k+1,p}^{\tw} + M_{k,p}^{\ttheta} + M_{k,p}^{\tw}) + \sum_{j=1}^k \gamma_j^2 P_{j+1:k}^{(2)} (1 + M_{j+1,p}^{\ttheta} + M_{j+1,p}^{\tw} + M_{j,p}^{\ttheta} + M_{j,p}^{\tw}) \eqsp.
\end{align}
Using \Cref{lem:markov:v_and_w_bounds} we get for any $j \geq 0$
\begin{align}
\label{eq:markov:bounding_next_moments}
M_{j+1, p}^{\ttheta} \lesssim 1 + M_{j,p}^{\ttheta} + M_{j,p}^{\tw} \eqsp, \eqsp
M_{j+1, p}^{\tw} \lesssim 1 + M_{j,p}^{\ttheta} + M_{j,p}^{\tw} \eqsp.
\end{align}
Thus, we get combining \eqref{eq:markov:bounding_next_moments} with \eqref{eq:markov:moment_w:t1_first}:
\begin{align}
\PE^{1/p}[\norm{T_1}^p] &\lesssim  P_{0:k}^{(2)} + \gamma_k (1 + M_{k,p}^{\ttheta} + M_{k,p}^{\tw}) + \sum_{j=0}^k \gamma_j^2 P_{j+1:k}^{(2)} (1 + M_{j,p}^{\ttheta} + M_{j,p}^{\tw}) \eqsp.
\end{align}
Then, using \Cref{lem:summ_alpha_k}-\ref{lem:summ_alpha_k_p_item} we get 
\begin{align}
\label{eq:markov:w_moment:t1_bound}
\PE^{2/p}[\norm{T_1}^p] &\lesssim (P_{0:k}^{(2)})^2 + \gamma_k^2 (1 + (M_{k,p}^{\ttheta})^2 + (M_{k,p}^{\tw})^2) + \bigl\{ \sum_{j=0}^k \gamma_j^2 P_{j+1:k}^{(2)} (1 + M_{j,p}^{\ttheta} + M_{j,p}^{\tw}) \bigr\}^{2} \\
&\lesssim (P_{0:k}^{(2)})^2 + \gamma_k^2 (1 + (M_{k,p}^{\ttheta})^2 + (M_{k,p}^{\tw})^2) + \gamma_k \sum_{j=0}^k \gamma_j^2 P_{j+1:k}^{(2)} (1 + (M_{j,p}^{\ttheta})^2 + (M_{j,p}^{\tw})^2) \eqsp.
\end{align}
To derive bounds for $T_2, T_3$ we use the definition of $\ttheta_j$, $\tw_{j}$, \Cref{lem:markov:v_and_w_bounds}, and  \Cref{L_converges} to obtain that 
\begin{align}
\label{eq:markov:delta_theta_moments}
\PE^{1/p}[\norm{\ttheta_j - \ttheta_{j-1}}^p] \lesssim \gamma_{j-1} (1 + M_{j-1,p}^{\ttheta} + M_{j-1,p}^{\tw}) \eqsp, \eqsp \PE^{1/p}[\norm{\tw_j - \tw_{j-1}}^p]
\lesssim \gamma_{j-1} (1 + M_{j-1,p}^{\ttheta} + M_{j-1,p}^{\tw})\eqsp,
\end{align}
Thus, the moment bound for $T_2 + T_3$ writes as follows:
\begin{align}
\label{eq:markov:w_moment:t2_t3_bound}
\PE^{2/p}[\norm{T_2 + T_3}^p] &\lesssim \biggl\{\sum_{j=0}^{k} \gamma_j^{2} P_{j+1:k} (1 + M_{j-1,p}^{\ttheta} + M_{j-1,p}^{\tw}) \biggr\}^2 \lesssim \gamma_k \sum_{j=0}^{k} \gamma_j^2 P_{j+1:k}^{(2)} (1 + (M_{j-1,p}^{\ttheta})^2 + (M_{j-1,p}^{\tw})^2) \eqsp.
\end{align}
Finally, the term $T_4$ can be bounded using \eqref{eq:markov:bound_w:pi_bound}:
\begin{align}
\label{eq:markov:w_moment:t4_bound}
\PE^{1/p}[\norm{T_4}^p] \lesssim \sum_{j=0}^k \gamma_j^2 P_{j+1:k}^{(2)} (1 + M_{j,p}^{\ttheta} + M_{j,p}^{\tw}) \eqsp.
\end{align}
Then \Cref{lem:summ_alpha_k}-\ref{lem:summ_alpha_k_p_item} implies that
\begin{align}
\label{eq:markov:t4_bound}
\PE^{2/p}[\norm{T_4}^p] \lesssim \gamma_k \sum_{j=0}^k \gamma_j^2 P_{j+1:k}^{(2)} (1 + (M_{j,p}^{\ttheta})^2 + (M_{j,p}^{\tw})^2) \eqsp.
\end{align}
Gathering the bounds \eqref{eq:markov:w_moment:t1_bound}, \eqref{eq:markov:w_moment:t2_t3_bound}, \eqref{eq:markov:w_moment:t4_bound} we obtain
\begin{align}
\label{eq:markov:w_moment:w1_final}
\PE^{2/p} [\norm{\tw_{k+1}^{(1)}}^p] \lesssim  (P_{0:k}^{(2)})^2 +  \gamma_k^2 (1 + (M_{k,p}^{\ttheta})^2 + (M_{k,p}^{\tw})^2) + \gamma_k \sum_{j=0}^k \gamma_j^2 P_{j+1:k}^{(2)} (1 + (M_{j,p}^{\ttheta})^2 + (M_{j,p}^{\tw})^2) \eqsp.
\end{align}

\paragraph{(III) Gathering (I) and (II).}
\, \\
\noindent Equations \eqref{eq:markov:mth_reccurence_bound} and \eqref{eq:markov:w_moment:w1_final} from the previous paragraphs imply
\begin{align}
(M_{k+1,p}^{\tw})^2 \lesssim (P_{0:k}^{(2)})^2 + p^2 \sum_{j=0}^k \gamma_j^2 P_{j+1:k}^{(2)} (1 + (M_{j,p}^{\ttheta})^2 + (M_{j,p}^{\tw})^2)) \eqsp.
\end{align}
Thus, there exists a constant $C_{\tw} > 0$ such that
\begin{align}
(M_{k+1,p}^{\tw})^2 \leq C_{\tw}(P_{0:k}^{(2)})^2 + p^2 C_{\tw} \sum_{j=0}^k \gamma_j^2 P_{j+1:k}^{(2)} (1 + (M_{j,p}^{\ttheta})^2 + (M_{j,p}^{\tw})^2)) \eqsp.
\end{align}
Denote the right hand side of the latter inequality by $U_{k+1}$ for $k \geq 0$, $U_0 = C_{\tw}$. Hence, for all $s \geq 0$ it holds that $(M_{s,p}^{\tw})^2 \leq U_{s}$. Thus, we get
\begin{align}
U_{k+1} \leq (1- \frac{a_{22}}{2} \gamma_k)^2 U_k + p^2 C_{\tw}  \gamma_k^2 (1 + U_k + (M_{k,p}^{\ttheta})^2) \eqsp.
\end{align}
The conditions on $k_0$ in \Cref{assum:markov:stepsize} guarantee that
\begin{align}
U_{k+1} \leq (1 - \frac{a_{22}}{2} \gamma_k) U_k + p^2 C_{\tw} \gamma_k^2 (1 + (M_{k,p}^{\ttheta})^2) \eqsp.
\end{align}
Enrolling the latter recursion and applying \Cref{lem:summ_alpha_k}-\ref{lem:summ_alpha_k_p_item} we get
\begin{align}
(M_{k+1,p}^{\tw})^2 \lesssim P_{0:k}^{(2)} + p^2  \gamma_k + p^2 \sum_{j=0}^k \gamma_j^2 P_{j+1:k}^{(2)} (M_{k+1, p}^{\ttheta})^2 \eqsp.
\end{align}
\end{proof}
\begin{proof}[Proof of \Cref{lem:markov:ttheta0_bound}]
Expanding the recursion \eqref{def:markov:extended_recursion} yields, with $\Gamma_{j+1:k}^{(1)}$ defined in \eqref{eq:matr_product_determ}, that 
\begin{align}
\label{eq:markov:theta_rec}
\ttheta_{k+1} &=  \Gamma_{0:k}^{(1)}\ttheta_0 - \sum_{j=0}^{k}\beta_j\Gamma_{j+1:k}^{(1)} A_{12}\tw_j - \sum_{j=0}^{k}\beta_j\Gamma_{j+1:k}^{(1)}V_{j+1} \\
&= \Gamma_{0:k}^{(1)}\ttheta_0 - \sum_{j=0}^{k}\beta_j\Gamma_{j+1:k}^{(1)}A_{12}\tw_j^{(0)} - \sum_{j=0}^{k}\beta_j\Gamma_{j+1:k}^{(1)}V_{j+1}^{(0)} - \sum_{j=0}^{k} \beta_j\Gamma_{j+1:k}^{(1)} A_{12}\tw_j^{(1)} - \sum_{j=0}^{k}\beta_j\Gamma_{j+1:k}^{(1)}V_{j+1}^{(1)}
\eqsp.
\end{align}
Next, we recursively expand $\tw_j^{(0)}$ using the relation \eqref{eq:w_rec}:
\begin{align}
\sum_{j=0}^{k}\beta_j\Gamma_{j+1:k}^{(1)}A_{12}\tw_j^{(0)} &=  \sum_{j=0}^{k}\beta_j\Gamma_{j+1:k}^{(1)} A_{12}\bigl(\Gamma_{0:j-1}^{(2)} \tw_{0} - \sum_{i=0}^{j-1}\Gamma_{i+1:j-1}^{(2)}\xi_{i+1}^{(0)}\bigr) \\
&=  \sum_{j=0}^{k}\beta_j\Gamma_{j+1:k}^{(1)} A_{12}\Gamma_{0:j-1}^{(2)} \tw_0  - \sum_{i=0}^{k-1}\bigl(\sum_{j=i+1}^{k}\beta_{j}\Gamma_{j+1:k}^{(1)}A_{12}\Gamma_{i+1:j-1}^{(2)}\bigr)\xi_{i+1}^{(0)}\eqsp.
\end{align}
Define, for $m \leq n$, the quantity 
\begin{equation}
T_{m:n} = \sum_{\ell=m}^{n}\beta_{\ell}\Gamma_{\ell+1:n}^{(1)}A _{12}\Gamma_{m:\ell-1}^{(2)}\eqsp,
\end{equation}
and note that, with $P_{k:j}^{(1)}$, $P_{k:j}^{(2)}$ defined in \eqref{eq:matr_product_determ}, it holds that 
\begin{equation}
\label{eq:markov:t_bound}
\|T_{m:n}\| \lesssim \sum_{\ell=m}^{n}\beta_{\ell}P_{\ell+1:n}^{(1)}P_{m:\ell-1}^{(2)}\eqsp.
\end{equation}
With the above notations, we can rewrite \eqref{eq:markov:theta_rec} as follows:
\begin{align}
\ttheta_{k+1} = \Gamma_{0:k}^{(1)}\ttheta_0 - T_{0:k}\tw_0 + \sum_{j=0}^{k-1}T_{j+1:k}\xi_{j+1}^{(0)} - \sum_{j=0}^{k}\beta_j\Gamma_{j+1:k}^{(1)}V_{j+1}^{(0)} - \sum_{j=0}^{k} \beta_j\Gamma_{j+1:k}^{(1)} A_{12}\tw_j^{(1)} - \sum_{j=0}^{k}\beta_j\Gamma_{j+1:k}^{(1)}V_{j+1}^{(1)} \eqsp.
\end{align}
Thus,
\begin{align}
(M_{k+1,p}^{\ttheta})^2 \lesssim
\underbrace{\PE^{2/p}[\norm{\Gamma_{0:k}^{(1)}\ttheta_0}^p]}_{\mathcal{R}_1} 
+ \underbrace{\PE^{2/p}[\norm{T_{0:k}\tw_0}^p]}_{\mathcal{R}_2}
&+ \underbrace{\PE^{2/p}[\norm{\sum_{j=0}^{k-1}T_{j+1:k}\xi_{j+1}^{(0)}}^p]}_{\mathcal{R}_3}
+ \underbrace{\PE^{2/p}[\norm{\sum_{j=0}^{k}\beta_j\Gamma_{j+1:k}^{(1)}V_{j+1}^{(0)}}^p]}_{\mathcal{R}_4} \\
&+ \underbrace{\PE^{2/p}[\norm{\sum_{j=0}^{k} \beta_j\Gamma_{j+1:k}^{(1)} A_{12}\tw_j^{(1)}}^p]}_{\mathcal{R}_5} 
+ \underbrace{\PE^{2/p}[\norm{\sum_{j=0}^{k}\beta_j\Gamma_{j+1:k}^{(1)}V_{j+1}^{(1)}}^p]}_{\mathcal{R}_6}
\end{align}
\paragraph{(I) Bounds on $\{\mathcal{R}_i\}_{i=1}^4$.}
\, \\
Easy to see that
\begin{align}
\label{eq:markov:theta_moment:R1}
\mathcal{R}_1 \lesssim (P_{0:k}^{(1)})^2 \eqsp.
\end{align}
To proceed with $\mathcal{R}_2$, we apply \Cref{convolution_inequality}  with \(j+1=0\) and use \(\beta_j\leq \rstep\gamma_j\):
\begin{align}
\label{eq:markov:theta_moment:R2}
\mathcal{R}_2 \lesssim \bigl(\eqsp \sum_{j=0}^{k}\beta_jP_{j+1:k}^{(1)}P_{0:j-1}^{(2)}\bigr)^2  &\lesssim \bigl(\eqsp \sum_{j=0}^{k}\gamma_jP_{j+1:k}^{(1)}P_{0:j-1}^{(2)}\bigr)^2 \lesssim \bigl(P_{0:k}^{(1)}\bigr)^2\eqsp.
\end{align}
Applying \Cref{lem:tilde_bounds} and Burkholder’s inequality, we
obtain that
\begin{align}
\label{eq:markov:theta_moment:R3}
\mathcal{R}_3  \lesssim p^2\mathbb{E}^{2/p}\bigl[\bigl(\sum_{j=0}^{k}\beta_j^2\|\Gamma_{j+1:k}^{(1)}V_{j+1}^{(0)}\|^2\bigr)^{p/2}\bigr]
&\lesssim  p^2\sum_{j=0}^{k}\beta_j^2\norm{\Gamma_{j+1:k}^{(1)}}^2\mathbb{E}^{2/p}\bigl[\|V_{j+1}^{(0)}\|^p\bigr]  \\
&\lesssim p^2 \sum_{j=0}^{k}\beta_j^2\bigl(P_{j+1:k}^{(1)}\bigr)^2\bigl(1 + (M_{j,p}^{\ttheta})^2 + (M_{j,p}^{\tw})^2\bigl)\eqsp.
\end{align}
Applying \Cref{lem:tilde_bounds}, \eqref{eq:markov:t_bound} and Burkholder’s inequality to the $\mathcal{R}_4$, we obtain that
\begin{align}
\label{eq:markov:theta_moment:R4}
\mathcal{R}_4  
&\lesssim p^2\mathbb{E}^{2/p}\big[\big(\sum_{j=0}^{k-1}\|T_{j+1:k}\xi_{j+1}^{(0)}\|^2\big)^{p/2}\big] \leq p^2\sum_{j=0}^{k-1}\|T_{j+1:k}\|^2\bigl(\PE^{2/p}[\|\xi_{j+1}^{(0)}\|^p]\bigr) \\
&\lesssim  p^2 \sum_{j=0}^{k-1}\gamma_j^2\big(\sum_{i=j+1}^{k}\beta_iP_{i+1:k}^{(1)}P_{j+1:i-1}^{(2)}\big)^2\big(1 + \big(M_{j,p}^{\ttheta}\big)^2 + \big(M_{j,p}^{\tw}\big)^2\big) \\
& \lesssim  p^2 \sum_{j=0}^{k-1}\beta_j^2\big(\sum_{i=j+1}^{k}\gamma_j P_{i+1:k}^{(1)}P_{j+1:i-1}^{(2)}\big)^2 \big(1 + \big(M_{j,p}^{\ttheta}\big)^2 + \bigl(M_{j,p}^{\tw}\bigr)^2\big)\\
&\overset{(a)}{\lesssim}  p^2 \sum_{j=0}^{k-1}\beta_j^2\bigl(P_{j+1:k}^{(1)}\bigr)^2 \big(1 + \big(M_{j,p}^{\ttheta}\big)^2 + \big(M_{j,p}^{\tw}\big)^2\big)\eqsp,
\end{align}
where the inequality (a) follows from \Cref{convolution_inequality}. 

\paragraph{(II) Bounds on $\mathcal{R}_5$ and $\mathcal{R}_6$.}
\, \\
\noindent
To proceed with $\mathcal{R}_5$, we combine Minkowski's inequality together with \Cref{lem:summ_alpha_k}-\ref{lem:summ_alpha_k_p_item} and get
\begin{align}
\mathcal{R}_5 \lesssim \sum_{j=0}^k \beta_j P_{j+1:k}^{(1)} \PE^{2/p}[\norm{\tw_j^{(1)}}^p] \eqsp. 
\end{align}
Applying \eqref{eq:markov:w_moment:w1_final} we obtain
\begin{align}
\sum_{j=0}^k \beta_j P_{j+1:k}^{(1)} \PE^{2/p}[\norm{\tw_j^{(1)}}^p] \lesssim \sum_{j=0}^k \beta_j P_{j+1:k}^{(1)} P_{0:j-1}^{(2)} &+ \sum_{j=0}^k \beta_j P_{j+1:k}^{(1)} \gamma_{j-1}^2 (1 + (M_{j,p}^{\tw})^2 + (M_{j,p}^{\ttheta})^2) \\ 
&+ \sum_{t=0}^{k-1} \sum_{j=t+1}^k \beta_j \gamma_{j-1} \gamma_t^2 P_{j+1:k}^{(1)}  P_{t+1:j-1}^{(2)} (1 + (M_{t,p}^{\tw})^2 + (M_{t,p}^{\ttheta})^2) \eqsp.
\end{align}
Now we use \Cref{convolution_inequality} and $\beta_j \leq \rstep \gamma_j$ together with $\gamma_j^2 \leq \frac{\gamma_0^2}{\beta_0} \beta_j$:  
\begin{align}
\sum_{j=0}^k \beta_j P_{j+1:k}^{(1)} \PE^{2/p}[\norm{\tw_j^{(1)}}^p] \lesssim  P_{0:k}^{(1)} &+  \sum_{j=0}^k \beta_j^2 P_{j+1:k}^{(1)} (1 + (M_{j,p}^{\tw})^2 + (M_{j,p}^{\ttheta})^2) \eqsp.
\end{align}
Thus,
\begin{align}
\label{eq:markov:theta_moment:R5}
\mathcal{R}_5 \lesssim P_{0:k}^{(1)} + \sum_{j=0}^k \beta_j^2 P_{j+1:k}^{(1)} (1 + (M_{j,p}^{\tw})^2 + (M_{j,p}^{\ttheta})^2) \eqsp.
\end{align}
Set $\what_k = w_k - \wstar$. Hence
\begin{align}
\what_j - \what_{j+1} = \tw_{j+1} - \tw_j - D_{j-1} \ttheta_j + D_j \ttheta_{j+1} \eqsp,
\end{align}
and 
\begin{align}
V_{j+1}^{(1)} &= (\MKQ \poiseps{V}{j} - (\MKQ \poisA{11}{j}) \ttheta_j - (\MKQ \poisA{12}{j}) \what_j) - (\MKQ \poiseps{V}{j+1} - (\MKQ \poisA{11}{j+1}) \ttheta_{j+1} - (\MKQ \poisA{12}{j+1}) \what_{j+1}) \\
& \quad + (\MKQ \poisA{11}{j+1})(\ttheta_{j} - \ttheta_{j+1}) + (\MKQ \poisA{12}{j+1}) (\Id + D_j) (\tw_{j+1} - \tw_{j}) + (\MKQ \poisA{12}{j+1})(D_j - D_{j-1}) \ttheta_j \eqsp.
\end{align}
Now we derive a couple of auxiliary bounds. First, from the definition of $\what_j$ and $\tw_j$ we get
\begin{align}
\PE^{1/p} [\norm{\what_j}^p] \lesssim M_{j,p}^{\tw} + M_{j,p}^{\ttheta} \eqsp.
\end{align}
Set for simplicity
\begin{align}
u_{j} = \MKQ \poiseps{V}{j} - (\MKQ \poisA{11}{j}) \ttheta_j - (\MKQ \poisA{12}{j}) \what_j \eqsp,
\end{align}
and note that
\begin{align}
\norm{u_j} \lesssim 1 + M_{j,p}^{\ttheta} + M_{j,p}^{\tw} \eqsp.
\end{align}
Thus, using \Cref{lem:representation_lemma:kaledin} and Equation \eqref{eq:D_difference_bound} we obtain
\begin{align}
\PE^{1/p}[\norm{\sum_{j=0}^k \beta_j \Gamma_{j+1:k}^{(1)} V_{j+1}^{(1)}}^p] &\lesssim  P_{0:k}^{(1)} + \beta_{k} (1 + M_{k,p}^{\tw} + M_{k,p}^{\ttheta}) + \sum_{j=1}^k \beta_j^2 P_{j+1:k}^{(1)} (1 + M_{j,p}^{\ttheta} + M_{j,p}^{\tw}) \\
&+ \sum_{j=1}^k \beta_j P_{j+1:k}^{(1)} \PE^{1/p}[\norm{\ttheta_{j+1} - \ttheta_j}^p] + \bConst{A} \sqrt{\kappa_{\Delta}} \sum_{j=1}^k \beta_j P_{j+1:k}^{(1)} \PE^{1/p}[\norm{\tw_{j+1} - \tw_j}^p] \\ &+ \sum_{j=1}^k \beta_j \gamma_j P_{j+1:k}^{(1)} M_{j,p}^{\ttheta} \eqsp.
\end{align}
Next, applying \eqref{eq:markov:bounding_next_moments}, \eqref{eq:markov:delta_theta_moments} we get
\begin{align}
\PE^{1/p}[\norm{\sum_{j=0}^k \beta_j \Gamma_{j+1:k}^{(1)} V_{j+1}^{(1)}}^p] &\lesssim P_{0:k}^{(1)} + \beta_k  (1 + M_{k,p}^{\tw} + M_{k,p}^{\ttheta}) + \sum_{j=0}^k \beta_j \gamma_j P_{j+1:k}^{(1)} (1 + M_{j,p}^{\tw} + M_{j,p}^{\ttheta}) \eqsp.
\end{align}
Hence, \Cref{lem:summ_alpha_k}-\ref{lem:summ_alpha_k_p_item} implies that
\begin{align}
\label{eq:markov:theta_moment:R6}
\mathcal{R}_6 &\lesssim (P_{0:k}^{(1)})^2 + \beta_k^2 (1 + (M_{k,p}^{\tw})^2 + (M_{k,p}^{\ttheta})^2) + \sum_{j=0}^k \beta_j P_{j+1:k}^{(1)} \gamma_j^2 (1 + (M_{j,p}^{\tw})^2 + (M_{j,p}^{\ttheta})^2) \\
&\lesssim  (P_{0:k}^{(1)})^2 + \beta_k^2 (1 + (M_{k,p}^{\tw})^2 + (M_{k,p}^{\ttheta})^2) + \sum_{j=0}^k \beta_j^2 P_{j+1:k}^{(1)} (1 + (M_{j,p}^{\tw})^2 + (M_{j,p}^{\ttheta})^2) \eqsp.
\end{align}

\paragraph{(III) Gathering (I) and (II).}
\, \\
\noindent Gathering the similar terms in \eqref{eq:markov:theta_moment:R1}, \eqref{eq:markov:theta_moment:R2}, \eqref{eq:markov:theta_moment:R3}, \eqref{eq:markov:theta_moment:R4}, \eqref{eq:markov:theta_moment:R5}, \eqref{eq:markov:theta_moment:R6} we obtain 
\begin{align}
\label{eq:markov:moment_theta_first}
(M_{k+1,p}^{\ttheta})^2 \lesssim P_{0:k} + p^2 \sum_{j=0}^{k} \beta_j^2 P_{j+1:k}^{(1)} (1 + (M_{k,p}^{\tw})^2 + (M_{k,p}^{\ttheta})^2) \eqsp.
\end{align}
Applying \Cref{prop:tw_bond} and \Cref{lem:summ_alpha_k}, we obtain
\begin{align}
\sum_{j=0}^k\beta_j^2\bigl(P_{j+1:k}^{(1)}\bigr)^2\bigl(M_{j,p}^{\tw}\bigr)^2 &\lesssim \sum_{j=0}^k\beta_j^2\bigl(P_{j+1:k}^{(1)}\bigr)^2\bigl( P_{0:j-1}^{(2)} + p^2 \gamma_{j-1} + p^2 \sum_{i=0}^{j-1}\gamma_i^2P_{i+1:j-1}^{(2)}\bigl(M_{i,p}^{\ttheta}\bigr)^2\bigr) \\
&\lesssim p^2  \sum_{j=0}^{k}\beta_j^2P_{j+1:k}^{(1)}  + p^2 \sum_{j=0}^{k}\beta_j^2\bigl(P_{j+1:k}^{(1)}\bigr)^2\sum_{i=0}^{j-1}\gamma_i^2P_{i+1:j-1}^{(2)}\bigl(M_{i,p}^{\ttheta}\bigr)^2 \\
&\lesssim  p^2 \beta_{k} + p^2 \sum_{i=0}^{k-1}\sum_{j=i+1}^{k}\beta_j^2\gamma_i^2P_{i+1:j-1}^{(2)}P_{j+1:k}^{(1)} \bigl(M_{i,p}^{\ttheta}\bigr)^2\eqsp.
\end{align}
Using that $\beta_j^2 \leq \beta_i^2$ for $j \geq i+1$ and $\gamma_i^2 \leq \gamma_0 \gamma_i$, we get that 
\begin{align}
\sum_{j=0}^k\beta_j^2\bigl(P_{j+1:k}^{(1)}\bigr)^2\bigl(M_{j,p}^{\tw}\bigr)^2 
&\lesssim  p^2 \beta_{k} + p^2 \sum_{i=0}^{k-1}\beta_i^2\bigl(\sum_{j = i+1}^{k}\gamma_iP_{i+1:j-1}^{(2)}P_{j+1:k}^{(1)}\bigr)\bigl(M_{i,p}^{\ttheta}\bigr)^2 \overset{(a)}{\leq}  p^2 \beta_{k} + p^2 \sum_{i=0}^{k-1}\beta_i^2P_{i+1:k}^{(1)}\bigl(M_{i,p}^{\ttheta}\bigr)^2\eqsp,
\end{align}
where $(a)$ follows from \Cref{convolution_inequality}. The proof follows from substituting the latter ineqaulity into \eqref{eq:markov:moment_theta_first}.
\end{proof}

\subsection{CLT for the Polyak-Ruppert averaged estimator}

From equations \eqref{eq:noise_term_new} and \eqref{eq:pr_eq_3}, we derive the extended version of  \eqref{eq:pr_theta_decomposition}:

\begin{align}
\label{eq:markov:pr_theta_decomposition_appendix}
\sqrt{n}\Delta(\btheta_n - \thetas) =
&\frac{1}{\sqrt{n}} \sum_{k=1}^n (\funnoisew_V^{k+1} - A_{12} A_{22}^{-1} \funnoisew_W^{k+1}) \\ &+ \frac{1}{\sqrt{n}} \sum_{k=1}^n \biggl\{\underbrace{( A_{12} A_{22}^{-1} \funcAwtilde_{21}^{k+1} - \funcAwtilde_{11}^{k+1})}_{\Phi_{k+1}} \ttheta_{k} + \underbrace{( A_{12} A_{22}^{-1} \funcAwtilde_{22}^{k+1} - \funcAwtilde_{12}^{k+1})}_{\Psi_{k+1}} (w_{k}-w^\star) \biggr\} 
\\ &
+ \frac{1}{\sqrt{n}} \sum_{k=1}^{n}\beta_k^{-1}(\ttheta_k - \ttheta_{k+1})
- \frac{1}{\sqrt{n}} \sum_{k=1}^{n}A_{12}A_{22}^{-1}\gamma_k^{-1}(w_k - w_{k+1})
\end{align}
Setting
\begin{align}
\psi_{j+1} = \funnoisew_V^{j+1} -  A_{12} A_{22}^{-1} \funnoisew_W^{j+1} \eqsp,
\end{align}
we derive a decomposition of $\sqrt{n} \Delta (\btheta_n - \thetas)$ using the Poisson equation construction \eqref{eq:pois_solution_notation}:
\begin{align}
\label{eq:markov:pr_full_decomposition}
\sqrt{n}\Delta(\btheta_n - \thetas) = \eqsp & \underbrace{\frac{1}{\sqrt{n}} \sum_{k=1}^n \{\pois{\psi}{k+1} - \MKQ \pois{\psi}{k}\}}_{T^{\text{mark}}}
\\ &+ \underbrace{\frac{1}{\sqrt{n}} \{ \MKQ \pois{\psi}{1} -  \MKQ \pois{\psi}{n+1} + (\MKQ \pois{\Phi}{1}) (\Id + D_0) \ttheta_1 - (\MKQ \pois{\Phi}{n+1}) (\Id + D_n) \ttheta_{n+1} + (\MKQ \pois{\Psi}{1}) \tw_1 - (\MKQ \pois{\Psi}{n+1}) \tw_{n+1} \}}_{R_1} \\
&+ \underbrace{\frac{1}{\sqrt{n}} \sum_{k=1}^n \bigl\{ (\pois{\Phi}{k+1} - \MKQ \pois{\Phi}{k}) \ttheta_k + (\pois{\Psi}{k+1} - \MKQ \pois{\Psi}{k}) (w_k - \wstar) \bigr\}}_{R_2} \\
\\ &+ \underbrace{\frac{1}{\sqrt{n}} \sum_{k=1}^n \beta_k^{-1} \{\ttheta_k - \ttheta_{k+1}\} - \frac{1}{\sqrt{n}} \sum_{k=1}^n (\MKQ \pois{\Phi}{k+1}) \{\ttheta_k - \ttheta_{k+1}\}}_{R_3} \\
&- \underbrace{\frac{1}{\sqrt{n}} \sum_{k=1}^n \gamma_k^{-1} A_{12}A_{22}^{-1} \{w_k - w_{k+1}\} - \frac{1}{\sqrt{n}} \sum_{k=1}^n (\MKQ \pois{\Psi}{k+1}) \{w_k - w_{k+1}\}}_{R_4} \eqsp.
\end{align}
Note that \Cref{assum:UGE} implies that
$$
\frac{1}{\sqrt{n}} \sum_{k=1}^n  \psi_{k+1} \xrightarrow{d} \mathcal{N}(0, \cmark{\lineG}{\infty})
$$
for some covariance matrix $\cmark{\lineG}{\infty} \in \rset^{d_\theta \times d_\theta}$. Due to the properties of Poisson equation, we have using the notation of \eqref{eq:markov:pr_full_decomposition}:
\begin{align}
\var[T^{\text{mark}}] = \cmark{\lineG}{\infty} \eqsp.
\end{align}
Using the decomposition \eqref{eq:markov:pr_full_decomposition} we assume that $T^{\text{mark}}$ is a leading term, while 
$$ R_n^{\operatorname{pr, m}} = \sum_{i=1}^4 R_i$$ 
corresponds to a residual one. The proof of \Cref{th:markov:pr_clt} is given below and is based on \Cref{prop:nonlinearapprox} and \Cref{lem:markov:residual_pr_clt}.
\begin{lemma}
\label{lem:markov:residual_pr_clt}
Let $2 \leq p \leq \log n$. Assume \Cref{assum:hurwitz}, \Cref{assum:aij_bound},  \Cref{assum:UGE}, \Cref{assum:markov:stepsize}($\log n$). Then it holds that
\begin{align}
\PE^{1/p} \bigl[\bigl\|R_n^{\operatorname{pr, m}}\bigr\|^p\bigr] \lesssim n^{b-1/2} \prod_{j=0}^{n-1} \bigl(1 - \frac{a_\Delta}{8} \beta_j \bigr) + \frac{\log^3(n)}{n^{(1-b)/2}} + \log^3(n) \frac{(1-a)^{-1} + (1-b)^{-1}}{n^{a-1/2}} + \frac{\log^4(n) (1-a)^{-1}}{n^{a/2}} \eqsp.
\end{align}
\end{lemma}
\begin{proof}[Proof of \Cref{th:markov:pr_clt}.]
Note that for all $k$ it holds that
\begin{align}
\norm{\pois{\psi}{k+1} - \MKQ \pois{\psi}{k}} \leq \frac{16}{3} \taumix \sup_{x \in \Xset} \norm{\psi (x)} < \infty. 
\end{align}
Inroduce
\begin{align}
h(X_k) = \PE[(\pois{\psi}{k+1} - \MKQ \pois{\psi}{k}) (\pois{\psi}{k+1} - \MKQ \pois{\psi}{k})^\top \mid X_k] - \cmark{\lineG}{\infty} \eqsp. 
\end{align}
One can check that $\pi(h) = 0$ using Poisson equation properties, where $\pi$ is given in \Cref{assum:UGE}. Thus, $h$ satisfies the assumptions of \Cref{lem:bounded_differences_norms_markovian}.
Hence, we get applying \Cref{lem:martinglale_limits} with $p = 1$:
\begin{align}
\kolmogorov\bigl(\frac{1}{\sqrt{n}} \sum_{k=1}^n \{\pois{\psi}{k+1} - \MKQ \pois{\psi}{k}\}, \mathcal{N}(0, \cmark{\lineG}{\infty})\bigr) \lesssim \frac{1 + \log n}{n^{1/4}} \eqsp.
\end{align}
Since for all $q \in (0, 1)$ and $a_1, \ldots , a_m > 0$ it holds that $(\sum_{i=1}^m a_i)^{q} \leq \sum_{i=1}^m a_i^{q}$, \Cref{prop:nonlinearapprox}, \Cref{lem:markov:residual_pr_clt} imply that
\begin{align}
\kolmogorov\bigl(\sqrt{n} \Delta (\btheta_n - \thetas), \mathcal{N}(0, \cmark{\lineG}{\infty})\bigr) &\lesssim     \kolmogorov\bigl(\frac{1}{\sqrt{n}} \sum_{k=1}^n \{\pois{\psi}{k+1} - \MKQ \pois{\psi}{k+1}\}, \mathcal{N}(0, \cmark{\lineG}{\infty})\bigr) + c_{d_\theta}^{\frac{p}{p+1}} \bigl( \PE^{1/p}[\norm{R_1 + R_2 + R_3 + R_4}] \bigr)^{\frac{p}{p+1}} \\
&\lesssim \frac{1 + \log n}{n^{1/4}} + c_{d_\theta}^{\frac{p}{p+1}} \bigl\{n^{b-1/2} \prod_{j=0}^{n-1} \bigl(1 - \frac{a_\Delta}{8} \beta_j \bigr) \bigr\}^{\frac{p}{p+1}} + c_{d_\theta}^{\frac{p}{p+1}} \bigl\{\frac{\log^3(n)}{n^{(1-b)/2}}\}^{\frac{p}{p+1}} \\
&+ c_{d_\theta}^{\frac{p}{p+1}} \bigl\{ \log^3(n) \frac{(1-a)^{-1} + (1-b)^{-1}}{n^{a-1/2}}\bigr\}^{\frac{p}{p+1}} + c_{d_\theta}^{\frac{p}{p+1}} \bigl\{\frac{\log^4(n)(1-a)^{-1}}{n^{a/2}}\bigr\}^{\frac{p}{p+1}} \eqsp.
\end{align}
Note that $(n^{\alpha})^{\frac{\log n}{1 + \log n}} \leq n^\alpha \exp(|\alpha|)$ for all $\alpha \in \rset$. Thus, substituting $p := \log n$ into the latter inequality we get
\begin{align}
\kolmogorov\bigl(\sqrt{n} \Delta (\btheta_n - \thetas), \mathcal{N}(0, \cmark{\lineG}{\infty})\bigr) &\lesssim \frac{1 + \log n}{n^{1/4}} + c_{d_\theta} n^{b-1/2} \prod_{j=0}^{n-1} \bigl(1 - \frac{a_\Delta}{16} \beta_j \bigr) + c_{d_\theta} \frac{ \log^3(n)}{n^{(1-b)/2}} \\
&+ c_{d_\theta} \log^3(n) \frac{(1-a)^{-1} + (1-b)^{-1}}{n^{a-1/2}} + c_{d_\theta} \frac{\log^4(n)(1-a)^{-1}}{n^{a/2}} \eqsp,
\end{align}
and the proof follows.
\end{proof}

\begin{proof}[Proof of \Cref{lem:markov:residual_pr_clt}]
First, we use Minkowski's inequality
\begin{align}
\textstyle
\PE^{1/p} [\norm{\sum_{i=1}^4 R_i}^p] \leq \sum_{i=1}^4 \PE^{1/p} [\norm{R_i}^p] \eqsp.
\end{align}
\Cref{prop:markov:moment_bounds} and \Cref{assum:markov:stepsize} directly imply that $\PE^{1/p}[\norm{R_1}^p] \lesssim p^6 n^{-1/2} \leq n^{-1/2} \log^6(n) $. To proceed with $R_2$, we note that $R_2$ is a sum of martingale difference sequence due to the properties of Markov kernel $\MKQ$. Thus, Burkholder's inequality \cite[Theorem 8.1]{osekowski} and \Cref{prop:markov:moment_bounds} imply that
\begin{align}
\PE^{2/p}[\norm{R_2}^p] \leq \frac{p^2}{n} \sum_{k=1}^n \{\PE^{2/p} [\norm{\ttheta_k}] + \PE^{2/p}[\norm{w_k - w^\star}^p]\} &\lesssim \frac{p^2}{n} \sum_{k=1}^n \{ \prod_{j=0}^{k-1} (1 - \frac{a_{22} \gamma_j}{4}) + \prod_{j=0}^{k-1} (1 - \frac{a_{\Delta} \beta_j}{8}) + p^6 \gamma_k \} \\
&\overset{(a)}{\lesssim} \frac{p^2 k_0^b}{n} + \frac{p^8}{n^{a} (1-a)} \overset{(b)}{\lesssim} \frac{\log^{8}(n)}{n^{a} (1-a)} \eqsp,
\end{align}
where in $(a)$ we have additionally used \Cref{lem:summ_alpha_k}-\ref{lem:sum_as_Qell_item} and $(b)$ holds because \Cref{assum:markov:stepsize}($\log n$) implies $k_0^b = \mathcal{O}(\log^4 n)$. Now we derive bounds for $R_3$ and $R_4$. Rewrite $R_3$ as follows:
\begin{align}
R_3 = \frac{1}{\sqrt{n}} \beta_1^{-1} \ttheta_1 - \frac{1}{\sqrt{n}} \beta_n^{-1} \ttheta_{n+1} + \frac{1}{\sqrt{n}} \sum_{k=1}^{n-1} (\beta_{k+1}^{-1} - \beta_k^{-1}) \ttheta_{k+1} - \frac{1}{\sqrt{n}} \sum_{k=1}^{n} (\MKQ \pois{\Phi}{k+1}) \{\ttheta_k - \ttheta_{k+1}\} \eqsp.
\end{align}
Note that since $(1 + x)^b \leq 1 + bx$ for $b \in [0, 1]$, we get $\beta_{k+1}^{-1} - \beta_k^{-1} \leq b (k\beta_k)^{-1}$. Therefore, \Cref{prop:markov:moment_bounds} and \Cref{lem:summ_alpha_k}-\ref{lem:sum_as_Qell_item} imply that
\begin{align}
\sum_{k=1}^{n-1} (\beta_{k+1}^{-1} - \beta_k^{-1}) M_{k+1,p}^{\ttheta} \lesssim \sum_{k=1}^{n-1}(k\beta_k)^{-1}M_{k+1,p}^{\ttheta} \lesssim \sum_{k=1}^{n-1} \prod_{j=0}^{k}\bigl(1 - \beta_j\frac{a_{\Delta}}{8}\bigr) + p^2 \sum_{k=1}^{n-1}(k\beta_k)^{-1}\beta_{k}^{1/2} \lesssim k_0^b + p^2 n^{b/2} \eqsp.
\end{align}
Now, using \Cref{prop:markov:moment_bounds} and $\PE^{1/p}[\norm{\ttheta_{k} - \ttheta_{k+1}}^p] \lesssim p^3 \gamma_j$, we obtain
\begin{align}
\PE^{1/p}[\norm{R_3}^p] &\lesssim \frac{k_0^b}{\sqrt{n}} + \frac{(n+k_0)^{b}}{\sqrt{n}} \prod_{j=0}^{n-1} (1 - \frac{a_\Delta \beta_j}{8}) + \frac{p^2 (n+k_0)^{b/2}}{\sqrt{n}} + \frac{k_0^b + p^2 n^{b/2}}{\sqrt{n}} + \frac{p^3}{(1-b) n^{b-1/2}} \\ 
&\lesssim n^{b-1/2} \prod_{j=0}^{n-1} (1 - \frac{a_\Delta \beta_j}{8}) + \frac{\log^2(n)}{n^{(1-b)/2}} + \frac{\log^3(n)}{(1-b)n^{b-1/2}} + \frac{\log^4(n)}{\sqrt{n}} \eqsp.
\end{align}
To bound $R_4$ it is sufficient to apply $\PE^{1/p}[\norm{w_k - \wstar}^p] \lesssim M_{k,p}^{\ttheta} + M_{k,p}^{\tw}$. Thus, using $\gamma_k^{-1} \lesssim \beta_k^{-1}$ and $a_\Delta \beta_j \leq a_{22} \gamma_j$ to bound the terms with $M_{k,p}^{\ttheta}$ and $M_{k,p}^{\tw}$ separately, one can check that
\begin{align}
\PE^{1/p}[\norm{R_4}^p] &\lesssim n^{b-1/2} \prod_{j=0}^{n-1} \bigl(1 - \frac{a_\Delta}{8} \beta_j \bigr) + \frac{\log^2(n)}{n^{(1-b)/2}} + \frac{\log^3(n)}{(1-b)n^{b-1/2}} + \frac{1}{n^{(1-a)/2}} + \frac{1}{(1-a)n^{a-1/2}} \\
&\lesssim n^{b-1/2} \prod_{j=0}^{n-1} \bigl(1 - \frac{a_\Delta}{8} \beta_j \bigr) + \frac{\log^3(n)}{n^{(1-b)/2}} + \log^3(n) \frac{(1-a)^{-1} + (1-b)^{-1}}{n^{a-1/2}} + \frac{\log^4(n)}{\sqrt{n}} \eqsp.
\end{align}
The proof follows from gathering similar terms.
\end{proof}

\subsection{CLT for the Last iteration estimator}

First, we start from the same decomposition as in the martingale noise setting \eqref{eq:last_iter_decompos}:
\begin{align}\label{eq:last_iter_decompos}    \label{eq:markov:last_iteration_repr_extended}
\ttheta_{n+1} &= -\sum_{j=0}^{n} \beta_j G_{j+1:n}^{(1)} \psi_{j+1} +G_{0:n}^{(1)}\ttheta_0 - \sum_{j=0}^{n}\beta_jG_{j+1:n}^{(1)}A_{12}G_{0:j-1}^{(2)}\tw_0 \\
&+ \sum_{j=0}^{n}\beta_jG_{j+1:n}^{(1)}\;\delta_j^{(1)} + S_{n}^{(1)} + S_{n}^{(2)} + S_{n}^{(3)} \\
&+ \sum_{j=0}^{n}\beta_j G_{j+1:n}^{(1)}\bigl(\overbrace{\bigl\{\funcAwtilde_{11}^{j+1}- A_{12}A_{22}^{-1} \funcAwtilde_{21}^{j+1}\bigr\}}^{\Phi_{j+1}} \ttheta_j + \overbrace{\bigl\{\funcAwtilde_{12}^{j+1} -  A_{12}A_{22}^{-1} \funcAwtilde_{21}^{j+1}\bigr\}}^{\Psi_{j+1}} (w_j-w^\star)   \bigr) \eqsp,
\end{align}
where
\begin{align}
\label{eq:markov:last_iter_decompos_2}
\psi_{j+1} = \funnoisew_{V}^{j+1} &- A_{12} A_{22}^{-1} \funnoisew_{W}^{j+1} \eqsp, \eqsp \delta_j^{(1)} = A_{12} L_j\ttheta_j \eqsp, \eqsp  \delta_j^{(2)} = -(L_{j+1} + A_{22}^{-1}A_{21})A_{12}\tw_j\eqsp, \\
S_n^{(1)} &= -\sum_{j=0}^{n}\beta_jG_{j+1:n}^{(1)}A_{12}\sum_{i=0}^{j-1}\beta_iG_{i+1:j-1}^{(2)}\delta_i^{(2)}\eqsp, \\
S_n^{(2)} &= \sum_{j=0}^{n}\beta_jG_{j+1:n}^{(1)}A_{12}\sum_{i=0}^{j-1}\beta_iG_{i+1:j-1}^{(2)}D_iV_{i+1}\eqsp,\\
S_n^{(3)} &= \sum_{j=0}^{n}\beta_jG_{j+1:n}^{(1)}A_{12}\sum_{i=0}^{j-1}\gamma_iG_{i+1:j-1}^{(2)}W_{i+1} - \sum_{j=0}^{n}\beta_jG_{j+1:n}^{(1)}A_{12}A_{22}^{-1}W_{j+1}\eqsp.
\end{align}
Now we apply the Poisson equation technique and obtain from \eqref{eq:markov:last_iteration_repr_extended}:
\begin{align}    \label{eq:markov:last_iteration_pois_extended}
\ttheta_{n+1} = &-\underbrace{\sum_{j=0}^{n} \beta_j G_{j+1:n}^{(1)} (\pois{\psi}{j+1} - \MKQ \pois{\psi}{j})}_{\cmark{T}{\operatorname{last}}} + \underbrace{\sum_{j=0}^{n} \beta_j G_{j+1:n}^{(1)} (\MKQ \pois{\psi}{j+1} - \MKQ \pois{\psi}{j} + \Phi_{j+1} \ttheta_j + \Psi_{j+1} \what_j)}_{H_n} + G_{0:n}^{(1)}\ttheta_0 \\ &- \sum_{j=0}^{n}\beta_jG_{j+1:n}^{(1)}A_{12}G_{0:j-1}^{(2)}\tw_0 + \sum_{j=0}^{n}\beta_jG_{j+1:n}^{(1)}\;\delta_j^{(1)} + S_{n}^{(1)} + S_{n}^{(2)} + S_{n}^{(3)} \eqsp,
\end{align}
Also define
\begin{align}
R_n^{\operatorname{last, m}} = G_{0:n}^{(1)} \ttheta_0 - - \sum_{j=0}^{n}\beta_jG_{j+1:n}^{(1)}A_{12}G_{0:j-1}^{(2)}\tw_0 + \sum_{j=0}^{n}\beta_jG_{j+1:n}^{(1)}\;\delta_j^{(1)} + S_{n}^{(1)} + S_{n}^{(2)} + S_{n}^{(3)} + H_n \eqsp.
\end{align}
The proof of \Cref{th:markov:last_iter_clt} is based on gaussian approximation of $T_{\operatorname{last}} ^{\operatorname{mark}}$ and the moment bound for $R_n^{\operatorname{last, m}}$ which follows from Lemmas \ref{lem:markov:last_iter_sum_delta1}-\ref{lem:markov:last_iter_hk} that we state below:

\begin{lemma}
\label{lem:markov:last_iter_sum_delta1}
Let $p \geq 2$. Assume \Cref{assum:hurwitz}, \Cref{assum:aij_bound},  \Cref{assum:UGE}, \Cref{assum:markov:stepsize}($p$), \Cref{assum:markov:last_iter}.
Then it holds that
\begin{align}
\PE^{1/p}[\norm{ \sum_{j=0}^{n} \beta_j G_{j+1:n}^{(1)} \delta_j^{(1)}}^p] \lesssim  \frac{p^3}{2b-a-1} P_{0:n}^{(1)} + \frac{p^4}{2b-a-1} \beta_n^{\frac{2b-a/2-1}{b}} \eqsp.
\end{align}
\end{lemma}

\begin{lemma}
\label{lem:markov:last_iter_s1}
Let $p \geq 2$. Assume \Cref{assum:hurwitz}, \Cref{assum:aij_bound},  \Cref{assum:UGE}, \Cref{assum:markov:stepsize}($p$). Then it holds that
\begin{align}
\PE^{1/p}[\norm{S_n^{(1)}}^p] \lesssim P_{0:n}^{(1)} + p^4 \beta_n^{\frac{3b-2a}{2b}} \eqsp.
\end{align}
\end{lemma}

\begin{lemma}
\label{lem:markov:last_iter_s2}
Let $p \geq 2$. Assume \Cref{assum:hurwitz}, \Cref{assum:aij_bound},  \Cref{assum:UGE}, \Cref{assum:markov:stepsize}($p$). Then it holds that
\begin{align}
\PE^{1/p}[\norm{S_n^{(2)}}^p] \lesssim p^3 P_{0:n}^{(1)} + p^4 \beta_n^{\frac{2b-a}{2b}} \eqsp.
\end{align}
\end{lemma}

\begin{lemma}
\label{lem:markov:last_iter_s3}
Let $p \geq 2$. Assume \Cref{assum:hurwitz}, \Cref{assum:aij_bound},  \Cref{assum:UGE}, \Cref{assum:markov:stepsize}($p$). Then it holds that
\begin{align}
\PE^{1/p}[\norm{S_n^{(3)}}^p] \lesssim p^3 P_{0:n}^{(1)} + p^3 \beta_n^{\frac{a}{b}} +  p^4 \beta_k^{\frac{2b-a}{2b}} \eqsp.
\end{align}
\end{lemma}

\begin{lemma}
\label{lem:markov:last_iter_hk}
Let $p \geq 2$. Assume \Cref{assum:hurwitz}, \Cref{assum:aij_bound},  \Cref{assum:UGE}, \Cref{assum:markov:stepsize}($p$). Then it holds that
\begin{align}
\PE^{1/p}[\norm{H_n}^p] \lesssim \frac{p^3}{2b-1} \prod_{j=1}^{n} (1- \frac{a_\Delta}{4} \beta_j) + p^4 \beta_n^{a/b} \eqsp.
\end{align}
\end{lemma}

\begin{proof}[Proof of \Cref{th:markov:last_iter_clt}]
First, we introduce 
\begin{align}
\psi_{j+1} &= \pois{\psi}{j+1} - \MKQ \pois{\psi}{j} \eqsp, \eqsp 
R_n^{\operatorname{last}} &= G_{0:k}^{(1)} \ttheta_0 - \sum_{j=0}^n \beta_j G_{j+1:n}^{(1)} A_{12} G_{0:j-1}^{(2)} \tw_0 + \sum_{j=0}^n \beta_j G_{j+1:n}^{(1)} \delta^{(1)}_j + H_n + S_n^{(1)} + S_n^{(2)} + S_{n}^{(3)} \eqsp.
\end{align}
Set $p = \log n$. Thus, applying \Cref{convolution_inequality}-\ref{convolution_inequality_gamma} we get: 
\begin{align}
\PE^{1/p}[\norm{R_n^{\operatorname{last}}}^p] \lesssim P_{0:n}^{(1)} + P_{0:n}^{(1)} + \PE^{1/p}[\norm{\sum_{j=0}^n \beta_j G_{j+1:n}^{(1)} \delta^{(1)}_j}^p] + \PE^{1/p}[\norm{H_n}^p] + \PE^{1/p}[\norm{S_n^{(1)}}^p] + \PE^{1/p}[\norm{S_n^{(2)}}^p] + \PE^{1/p}[\norm{S_n^{(3)}}^p] \eqsp.
\end{align}
Now we use Lemmas \ref{lem:markov:last_iter_sum_delta1}-\ref{lem:markov:last_iter_hk} together with the inequality $\beta_n^{\frac{2b-a/2-1}{b}} > \beta_n^{\frac{2b-a}{2b}}$ and get
\begin{align}
\PE^{1/p}[\norm{R_n^{\operatorname{last}}}^p] \lesssim \frac{p^3}{2b-a-1} \prod_{j=0}^n (1 - \frac{a_\Delta}{4} \beta_j) + p^3 \beta_n^{a/b} + \frac{p^4}{2b-a-1} \beta_n^{\frac{2b-a/2-1}{b}}
\end{align}
Therefore, \Cref{prop:nonlinearapprox} implies that:
\begin{align}        \label{eq:markov:last_iter_clt:kolmogorov_first}
\kolmogorov(\beta_{n}^{-1/2} \ttheta_{n+1}, \mathcal{N}(0, \lineG_{\infty}^{\operatorname{last}, \text{m}})) \lesssim &\quad \kolmogorov(\beta_{n}^{-1/2} \cmark{T}{\operatorname{last}}, \mathcal{N}(0, \beta_n^{-1}\lineG_{n}^{\operatorname{last}, \text{m}})) \\ 
& + \kolmogorov(\mathcal{N}(0, \beta_n^{-1}\lineG_{n}^{\operatorname{last}, \text{m}}), \mathcal{N}(0, \lineG_{\infty}^{\operatorname{last}, \text{m}})) \\
&+ c_{d_\theta}^{\frac{p}{p+1}} (\PE^{1/p}[\norm{\beta_n^{-1/2}R_n^{\operatorname{last}}}^p])^{\frac{p}{p+1}} \eqsp.
\end{align}
The bound for the second term follows from \cite[Theorem 1.1]{Devroye2018} and \Cref{prop:ricatti_limit}:
\begin{align}
\kolmogorov(\mathcal{N}(0, \beta_n^{-1}\lineG_{n}^{\operatorname{last}, \text{m}}), \mathcal{N}(0, \lineG_{\infty}^{\operatorname{last}, \text{m}})) \lesssim \normop{ \{\lineG_{\infty}^{\operatorname{last}, \text{m}}\}^{-1/2} (\beta_n^{-1}\lineG_{n}^{\operatorname{last}, \text{m}}) \{\lineG_{\infty}^{\operatorname{last}, \text{m}}\}^{-1/2} - \Id}[\text{Fr}] \lesssim \frac{\sqrt{d_\theta}}{n^{b} \lambda_{\min}(\lineG_{\infty}^{\operatorname{last}, \text{m}})} \eqsp.
\end{align}
Next, the bound for the third term in \eqref{eq:markov:last_iter_clt:kolmogorov_first} follows from $(n^{-\alpha})^{\frac{\log n}{1+\log n}} \lesssim n^{-\alpha}$ and $(\sum a_i)^{q} \leq \sum a_i^q$ for $a_i > 0$ and $q \in (0, 1)$:
\begin{align}
\label{eq:mark:last_iter_clt:rnlast:bound}
c_{d_\theta}^{\frac{p}{p+1}} (\PE^{1/p}[\norm{\beta_n^{-1/2}R_n^{\operatorname{last}}}^p])^{\frac{p}{p+1}} \lesssim c_{d_\theta} \frac{\beta_n^{-1/2}\log^3 (n)}{2b-a-1} \prod_{j=0}^n (1 - \frac{a_\Delta}{8} \beta_j) + c_{d_\theta} \beta_n^{\frac{2a-b}{2b}}\log^3(n) + c_{d_\theta} \frac{\log^4(n)}{2b-a-1} \beta_n^{\frac{3b-a-2}{2b}} \eqsp.
\end{align}
Now we derive a bound for $\kolmogorov(\beta_n^{-1/2} \cmark{T}{\operatorname{last}}, \mathcal{N}(0, \beta_n^{-1}\lineG_{n}^{\operatorname{last}, \text{m}}))$. Introduce
\begin{align}
M_i = \beta_n^{-1/2} \beta_i G_{i+1:n}^{(1)} \psi_{i+1} \eqsp.
\end{align}
Hence, we get
\begin{align}
\norm{M_i} \lesssim \beta_n^{-1/2} \beta_i P_{i+1:n}^{(1)} \eqsp.
\end{align}
Note that \Cref{assum:markov:stepsize} implies that
\begin{align}
\frac{\beta_{i+1} P_{i+2:n}^{(1)}}{\beta_i P_{i+1:n}^{(1)}} = \frac{1}{(1-\frac{a_\Delta}{2} \beta_{i+1}) \frac{\beta_i}{\beta_{i+1}}} \geq \frac{1}{(1-\frac{a_\Delta}{2} \beta_{i+1})(1 + \frac{a_\Delta}{16} \beta_{i+1})} \geq \frac{1}{(1-\frac{a_\Delta}{2} \beta_{i+1}) (1+\frac{a_\Delta}{2} \beta_{i+1})} > 1 \eqsp.
\end{align}
Thus, for all $i$ it holds that $\beta_i P_{i+1:n}^{(1)} \leq \beta_n$ and $\norm{M_i} \lesssim \beta_n^{1/2}$. Now we introduce the function
\begin{align}
h(X_i) = \PE^{\F_i}[M_i M_i^\top] \eqsp, \eqsp \F_j = \sigma(X_s:s \leq j) \eqsp.
\end{align}
Note that $\norm{h(X_i)} \leq \beta_n$.  Thus, since
\[ \beta_n^{-1}\lineG_{n}^{\operatorname{last, m}} = \sum_{i=0}^n \PE[h(X_i)] \eqsp, \] \Cref{lem:bounded_differences_norms_markovian} implies that
\begin{align}
\P\biggl[\norm{\sum_{i=0}^n h(X_i) - \beta_n^{-1} \lineG_{n}^{\operatorname{last, m}}} \geq nt\biggr] \leq 4\exp(-\frac{nt^2}{80d\taumix \beta_n^{2}}) \eqsp.
\end{align}
Hence, the assumptions of \Cref{lem:martinglale_limits} hold with $C_1 = 4$ and $C_2 = (80d\taumix \beta_n^2)^{-1}$, which yiels with $\kappa := \beta_n^{1/2}$ and $p := \log n$:
\begin{align}
\kolmogorov(\sum_{i=0}^n M_i, \beta_n^{-1}\lineG_{n}^{\operatorname{last, m}} ) &= \kolmogorov\biggl(\tfrac{1}{\sqrt{n+1}} \sum_{i=0}^n M_i, \mathcal{N}(0, \tfrac{1}{n+1} \beta_n^{-1} \lineG_{n}^{\operatorname{last, m}})\biggr)  \\
&\overset{(a)}{\lesssim}
 (\log (n))^{3/4}  \bigl\{ \beta_n^{1/2} n^{1/4} (\log n)^{1/4} + \beta_n^{1/2} + \frac{1}{n\beta_n^{1/2}} + \frac{\beta_n \sqrt{\log n}}{\sqrt{n}} \bigr\}  \eqsp \\
&\lesssim \frac{\log n}{n^{b/2-1/4}} \eqsp,
\end{align}
where in $(a)$ we have used an elementary inequality $(n^{-\alpha})^{\frac{\log n}{1+\log n}} \lesssim n^{-\alpha}$ together with $\frac{1}{2} \norm{\lineG_{\infty}^{\operatorname{last}, \text{m}}} \leq \norm{\beta_n^{-1} \lineG_{n}^{\operatorname{last}, \text{m}}} \lesssim \norm{\lineG_{\infty}^{\operatorname{last}, \text{m}}}$ which holds due to \Cref{assum:markov:last_iter} and \Cref{prop:ricatti_limit}. Now we combine \eqref{eq:mark:last_iter_clt:rnlast:bound} with the latter inequality and get:
\begin{align}
    \kolmogorov(\beta_{n}^{-1/2} \ttheta_{n+1}, \mathcal{N}(0, \lineG_{\infty}^{\operatorname{last}, \text{m}})) \lesssim  \frac{n^{b/2}\log^3 (n)}{2b-a-1} \prod_{j=0}^n (1 - \frac{a_\Delta}{8} \beta_j) +   \frac{\log^3(n)}{n^{a-b/2}} +  \frac{\log^4(n)}{(2b-a-1)n^{\frac{3b-a-2}{2}}} + \frac{\log n}{n^{b/2-1/4}} \eqsp.
\end{align}
\end{proof}

To prove Lemmas \ref{lem:markov:last_iter_sum_delta1}-\ref{lem:markov:last_iter_hk} we formulate an auxiliary result that controls the moments of $\ttheta^{(1)}_k$, $\tw_k^{(1)}$:
\begin{lemma}
\label{lem:markov:markov_part_moments}
Let $p \geq 2$. Assume \Cref{assum:hurwitz}, \Cref{assum:aij_bound},  \Cref{assum:UGE}, \Cref{assum:markov:stepsize}($p$).
Then it holds for all $k \in \nset$ that
\begin{align}
\PE^{1/p}[\norm{\tw_k^{(1)}}^p] \lesssim P_{0:k}^{(2)} + p^3 \gamma_k \eqsp, \quad \PE^{1/p}[\norm{\ttheta_k^{(1)}}^{p}] \lesssim p^3 P_{0:k}^{(1)} + p^3 \gamma_k \eqsp.
\end{align}
\end{lemma}
\begin{proof}
The bound for $\tw_k^{(1)}$ follows from Equation \ref{eq:markov:w_moment:w1_final} applying \Cref{prop:markov:moment_bounds} and \Cref{lem:summ_alpha_k}-\ref{lem:summ_alpha_k_p_item}. To proceed with $\ttheta_k^{(1)}$, we use the decomposition that follows from \eqref{eq:markov:recurrence}:
\begin{align}
\ttheta_k^{(1)} = -\underbrace{\sum_{i=0}^{k-1} \beta_i \Gamma_{i+1:k-1}^{(1)} A_{12} \tw_{i}^{(1)}}_{Z_1} - \underbrace{\sum_{i=0}^{k-1} \beta_i \Gamma_{i+1:k-1}^{(1)} V_{i+1}^{(1)}}_{Z_2} \eqsp.
\end{align}
The bound for $Z_1$ follows from Minkowski's inequality and the bound for $\tw_k^{(1)}$:
\begin{align}
\PE^{1/p}[\norm{Z_1}^p] \lesssim \sum_{i=0}^{k-1} \beta_i P_{i+1:k-1}^{(1)} P_{0:i-1}^{(2)} + p^3 \sum_{i=0}^{k-1} \beta_i \gamma_i P_{i+1:k-1}^{(1)} \overset{(a)}{\lesssim} P_{0:k}^{(1)} + p^3 \gamma_k \eqsp,
\end{align}
where in $(a)$ we additionally used \Cref{lem:summ_alpha_k}-\ref{lem:summ_alpha_k_p_item} and \Cref{convolution_inequality}-\ref{convolution_inequality_beta}.
\begin{equation}
V_{i+1}^{(1)} = (\MKQ \pois{\funnoisew_V}{i} - (\MKQ \poisA{11}{i}) \ttheta_i - (\MKQ \poisA{12}{i}) \what_i) - (\MKQ \pois{\funnoisew_V}{i+1} - (\MKQ \poisA{11}{i+1}) \ttheta_{i+1} - (\MKQ \poisA{12}{i}) \what_{i+1}) + (\MKQ \poisA{11}{i+1})(\ttheta_i - \ttheta_{i+1}) + (\MKQ \poisA{12}{i+1})(\what_i - \what_{i+1}) \eqsp.
\end{equation}
Introduce the following notation:
\begin{align}
v_{i} &= \MKQ \pois{\funnoisew_V}{i} - (\MKQ \poisA{11}{i}) \ttheta_i - (\MKQ \poisA{12}{i}) \what_i \eqsp.
\end{align}
Thus, we rewrite $Z_2$ using \Cref{lem:representation_lemma:kaledin}:
\begin{align}
Z_2 = \beta_0 \Gamma_{1:k-1}^{(1)} v_0 - \beta_k v_k &+ \sum_{i=1}^{k} (\beta_i^2 B_{11}^i \Gamma_{i+1:k-1}^{(1)} + (\beta_i - \beta_{i-1}) \Gamma_{i:k-1}^{(1)}) v_i \\ &+ \sum_{i=0}^{k-1} \beta_i \Gamma_{i+1:k-1}^{(1)} \{(\MKQ \poisA{11}{i+1})(\ttheta_i - \ttheta_{i+1}) + (\MKQ \poisA{12}{i+1})(\what_i - \what_{i+1})\} \eqsp.
\end{align}
Therefore, we get applying Minkowski's inequality together with \Cref{lem:summ_alpha_k}-\ref{lem:summ_alpha_k_p_item} and $\gamma_i \lesssim \beta_i^{a/b}$:
\begin{align}
\PE^{1/p}[\norm{Z_2}^p] \lesssim p^3 P_{0:k}^{(1)} + p^3 \beta_k + p^3 \beta_k^{a/b} \lesssim p^3 P_{0:k}^{(1)} + p^3 \gamma_k \eqsp.
\end{align}
The proof follows from gathering bounds for $Z_1$ and $Z_2$.
\end{proof}

\begin{proof}[Proof of \Cref{lem:markov:last_iter_sum_delta1}]
First, we introduce:
\begin{align}
\delta_j^{(1, i)} = A_{12} L_j \ttheta_j^{(i)} \eqsp, \eqsp i \in \{0, 1\} \eqsp.
\end{align}
Thus, we rewrite the initial sum
\begin{align}
\sum_{j=0}^{n} \beta_j G_{j+1:n}^{(1)} \delta_j^{(1)} = \underbrace{\sum_{j=0}^{n} \beta_j G_{j+1:n}^{(1)} \delta_j^{(1, 0)}}_{Z_1} + \underbrace{\sum_{j=0}^{n} \beta_j G_{j+1:n}^{(1)} \delta_j^{(1, 1)}}_{Z_2} \eqsp.
\end{align}
$Z_1$ can be bounded repeating the lines of \Cref{lem:delta_bound} due to the recurrence properties \eqref{eq:markov:recurrence}:
\begin{align}
\PE^{1/p}[\norm{Z_1}^p] \lesssim (2b-a-1)^{-1} P_{0:n}^{(1)} + p^4 (2b-a-1)^{-1} \beta_n^{\frac{2b-a/2-1}{b}} \eqsp.
\end{align}
To estimate $\PE^{1/p}[\norm{Z_2}^p]$ we first use \Cref{L_converges}, \Cref{lem:markov:markov_part_moments} and obtain
\begin{align}
\PE^{1/p}[\norm{\delta_j^{(1,1)}}^p] \lesssim \frac{\beta_j}{\gamma_j} \{p^3 P_{0:j}^{(1)} + p^3\gamma_j\} \eqsp.
\end{align}
Thus,
\begin{align}
\PE^{1/p}[\norm{Z_2}^p] \lesssim p^3 P_{0:n}^{(1)} \sum_{j=0}^{k} \frac{\beta_j^2}{\gamma_j} + p^3 \beta_n \lesssim \frac{p^3}{2b-a-1} P_{0:n}^{(1)} + p^3 \beta_n \eqsp.
\end{align}
The proof follows from gathering the bounds for $Z_1$ and $Z_2$ together with applying  \Cref{assum:markov:stepsize}.
\end{proof}

\begin{proof}[Proof of \Cref{lem:markov:last_iter_s1}]
First, we introduce:
\begin{align}
\delta_j^{(2, i)} = -(L_{j+1} + A_{22}^{-1} A_{21}) A_{12} \tw_j^{(i)} \eqsp, \eqsp i \in \{0, 1\} \eqsp.
\end{align}
Thus, we rewrite the initial sum as follows:
\begin{align}
S_n^{(1)} = -\underbrace{\sum_{j=0}^{n}\beta_jG_{j+1:n}^{(1)}A_{12}\sum_{i=0}^{j-1}\beta_iG_{i+1:j-1}^{(2)}\delta_i^{(2, 0)}}_{Z_1} -\underbrace{\sum_{j=0}^{n}\beta_jG_{j+1:n}^{(1)}A_{12}\sum_{i=0}^{j-1}\beta_iG_{i+1:j-1}^{(2)}\delta_i^{(2, 1)}}_{Z_2} \eqsp.
\end{align}
One can obtain the bound on $Z_1$ following the lines of \Cref{lem:s_1_bound} due to the fact that $\tw_j^{(0)}$ is a martingale-difference sequence w.r.t. $\F_j = \sigma(X_s : s \leq j)$:
\begin{align}
\PE^{1/p}[\norm{Z_1}^p] \lesssim P_{0:n}^{(1)} + p^4 \beta_n^{\frac{3b-2a}{2b}} \eqsp.
\end{align}
To derive a bound for $Z_2$ we use Minkowski's inequality and \Cref{lem:markov:markov_part_moments}:
\begin{align}
\PE^{1/p}[\norm{Z_2}^p] &\lesssim \sum_{j=0}^n \sum_{i=0}^{j-1} \beta_j \beta_i P_{j+1:n}^{(1)} P_{i+1:j-1}^{(2)} \{P_{0:i}^{(2)} + p^3 \gamma_i\} =  \sum_{i=0}^{n-1} \beta_i \{P_{0:i}^{(2)} + p^3\gamma_i\}\sum_{j=i+1} ^{n} \beta_j P_{j+1:n}^{(1)} P_{i+1:j-1}^{(2)} \\
&\overset{(a)}{\lesssim} \sum_{i=0}^{n-1} \beta_i \gamma_i^{\frac{b-a}{a}} \{P_{0:i}^{(2)} + p^3\gamma_i\} P_{i+1:n}^{(1)} \overset{(b)}{\lesssim} P_{0:n}^{(1)} + p^3 \beta_n \eqsp,
\end{align}
where in $(a)$ and $(b)$ we have used \Cref{convolution_inequality}-\ref{convolution_inequality_beta} together with \Cref{lem:summ_alpha_k}-\ref{lem:summ_alpha_k_p_item}. The proof follows from gathering similar terms in the bounds for $Z_1$ and $Z_2$.
\end{proof}

\begin{proof}[Proof of \Cref{lem:markov:last_iter_s2}]
First, we decompose $S_n^{(2)}$ as follows:
\begin{align}
S_n^{(2)} = \underbrace{\sum_{j=0}^{n}\beta_jG_{j+1:n}^{(1)}A_{12}\sum_{i=0}^{j-1}\beta_iG_{i+1:j-1}^{(2)}D_iV_{i+1}^{(0)}}_{Z_1} + \underbrace{\sum_{j=0}^{n}\beta_jG_{j+1:n}^{(1)}A_{12}\sum_{i=0}^{j-1}\beta_iG_{i+1:j-1}^{(2)}D_iV_{i+1}^{(1)}}_{Z_2} \eqsp.
\end{align}
Since $V_{i+1}^{(0)}$ is a martingale difference sequence w.r.t. $\F_i$, repeating the lines of \Cref{lem:s_2_bound} one can obtain that
\begin{align}
\label{eq:markov:s2_martingale_part}
\PE^{1/p}[\norm{Z_1}^p] \lesssim  p^4 \beta_n^{\frac{2b-a}{2b}} \eqsp.
\end{align}
To proceed with $Z_2$, recall the decomposition
\begin{equation}
V_{i+1}^{(1)} = (\MKQ \pois{\funnoisew_V}{i} - (\MKQ \poisA{11}{i}) \ttheta_i - (\MKQ \poisA{12}{i}) \what_i) - (\MKQ \pois{\funnoisew_V}{i+1} - (\MKQ \poisA{11}{i+1}) \ttheta_{i+1} - (\MKQ \poisA{12}{i}) \what_{i+1}) + (\MKQ \poisA{11}{i+1})(\ttheta_i - \ttheta_{i+1}) + (\MKQ \poisA{12}{i+1})(\what_i - \what_{i+1}) \eqsp.
\end{equation}
Introduce the following notation:
\begin{align}
v_{i} &= \MKQ \pois{\funnoisew_V}{i} - (\MKQ \poisA{11}{i}) \ttheta_i - (\MKQ \poisA{12}{i}) \what_i \eqsp, Q_i =  \sum_{j=i+1} ^{n}\beta_j G_{j+1:k}^{(1)}A_{12}G_{i+1:j-1}^{(2)}D_i \eqsp.
\end{align}
\Cref{convolution_inequality}-\ref{convolution_inequality_beta} implies that $\norm{Q_i} \lesssim \gamma_i^{(b-a)/a} P_{i+1:k}^{(1)}$. Then we estimate $\norm{Q_i - Q_{i+1}}$ using \Cref{convolution_inequality}-\ref{convolution_inequality_gamma}, \ref{convolution_inequality_beta}:
\begin{align}
\norm{Q_i - Q_{i+1}} &\lesssim \sum_{j=i+1} ^{n}\beta_j P_{j+1:n}^{(1)} P_{i+1:j-1}^{(2)} \norm{D_i - D_{i+1}} + \beta_{i+1} P_{i+2:n}^{(1)} + \gamma_i \sum_{j=i+2}^{n} \beta_j P_{j+1:n}^{(1)} P_{i+2:j-1}^{(2)} \\
&\lesssim \gamma_i \gamma_i^{(b-a)/a} P_{i+1:n}^{(1)} + \beta_i P_{i+1:n}^{(1)} + \gamma_i \gamma_i^{(b-a)/a} P_{i+1:n}^{(1)} \lesssim \beta_i P_{i+1:n}^{(1)}  \eqsp.
\end{align}
Now we swap the order of summation and rewrite $Z_2$ as follows:
\begin{align}
Z_2 = \underbrace{\sum_{i=0}^{n-1} Q_i \beta_i (v_i - v_{i+1})}_{Z_{21}} + \underbrace{\sum_{i=0}^{n-1} Q_i \beta_i \{(\MKQ \poisA{11}{i+1})(\ttheta_i - \ttheta_{i+1}) + (\MKQ \poisA{12}{i+1})(\what_i - \what_{i+1})\}}_{Z_{22}} \eqsp.
\end{align}
Then we further decompose $Z_{21}$:
\begin{align}
Z_{21} = \sum_{i=0}^{n-1} \{ (Q_i \beta_i v_i - Q_{i+1} \beta_{i+1} v_{i+1}) + (Q_{i+1} - Q_i) \beta_i v_{i+1} + Q_i(\beta_{i+1} - \beta_i) v_{i+1} \} \eqsp.
\end{align}
Now note that for all $i$ it holds that $\PE^{1/p}[\norm{v_i}^p] \lesssim p^3$ due to \Cref{prop:markov:moment_bounds}. Thus, we get using \Cref{lem:summ_alpha_k}-\ref{lem:summ_alpha_k_p_item}:
\begin{align}
\PE^{1/p}[\norm{Z_{21}}^p] &\lesssim \PE^{1/p}[\norm{Q_0 \beta_0 v_0}^p] + \PE^{1/p}[\norm{Q_{n} \beta_n v_n}^p] + \sum_{i=0}^{n-1} p^3 \{\beta_i^2 + \beta_i^3\} P_{i+1:n}^{(1)} +  \sum_{i=0}^{n-1} p^3 \beta_i \gamma_i \gamma_i^{(b-a)/a} P_{i+1:n}^{(1)} \\
&\overset{(a)}{\lesssim} p^3 P_{1:n}^{(1)} + p^3 \beta_n \eqsp,
\end{align}
where in $(a)$ we have additionally used $\gamma_i \lesssim \beta_i^{a/b}$ and \Cref{lem:summ_alpha_k}-\ref{lem:summ_alpha_k_p_item}. Since $\PE^{1/p}[\norm{\ttheta_i - \ttheta_{i+1}}^p] \lesssim p^3 \gamma_i$ and $\PE^{1/p}[\norm{\what_i - \what_{i+1}}^p] \lesssim p^3 \gamma_i$, we obtain applying \Cref{lem:summ_alpha_k}-\ref{lem:summ_alpha_k_p_item}:
\begin{align}
\PE^{1/p}[\norm{Z_{22}}^p] \lesssim p^3 \sum_{i=0}^{n-1} \beta_i \gamma_i^{(b-a)/a} \gamma_i P_{i+1:n}^{(1)} \lesssim p^3 \beta_n \eqsp.
\end{align}
Finally, 
\begin{align}
\PE^{1/p}[\norm{S_n^{(2)}}^p] \lesssim \PE^{1/p}[\norm{Z_1}^p] + \PE^{1/p}[\norm{Z_{21}}^p] + \PE^{1/p}[\norm{Z_{22}}^p] \lesssim p^3 P_{0:n}^{(1)} + p^4 \beta_n^{\frac{2b-a}{2b}} \eqsp.
\end{align}
\end{proof}

\begin{proof}[Proof of \Cref{lem:markov:last_iter_s3}]
First, we decompose $S_k^{(3)}$ into two parts:
\begin{align}
S_k^{(3)} = &\underbrace{\sum_{j=0}^{n}\beta_jG_{j+1:n}^{(1)}A_{12}\sum_{i=0}^{j-1}\gamma_iG_{i+1:j-1}^{(2)}W_{i+1}^{(0)} - \sum_{j=0}^{n}\beta_jG_{j+1:n}^{(1)}A_{12}A_{22}^{-1}W_{j+1}^{(0)}}_{Z_1} \\
&+ \underbrace{\sum_{j=0}^{n}\beta_jG_{j+1:n}^{(1)}A_{12}\sum_{i=0}^{j-1}\gamma_iG_{i+1:j-1}^{(2)}W_{i+1}^{(1)} - \sum_{j=0}^{n}\beta_jG_{j+1:n}^{(1)}A_{12}A_{22}^{-1}W_{j+1}^{(1)}}_{Z_2} \eqsp.
\end{align}
Note that $Z_1$ can be bounded following the lines of \Cref{lem:s_3_bound}. Precisely,
\begin{align}   \label{eq:markov:bound_sk3:z1}
\PE^{1/p} [\norm{Z_1}^p] \lesssim p^4 \beta_n^{\frac{2b-a}{2b}} \eqsp.
\end{align}
To proceed with $Z_2$, we derive a decomposition for $W_{i+1}^{(1)}$:
\begin{equation}
\label{eq:markov:W1_decomposition}
W_{i+1}^{(1)} = (\MKQ \pois{\funnoisew_W}{i} - (\MKQ \poisA{21}{i}) \ttheta_i - (\MKQ \poisA{22}{i}) \what_i) - (\MKQ \pois{\funnoisew_W}{i+1} - (\MKQ \poisA{21}{i+1}) \ttheta_{i+1} - (\MKQ \poisA{22}{i}) \what_{i+1}) + (\MKQ \poisA{21}{i+1})(\ttheta_i - \ttheta_{i+1}) + (\MKQ \poisA{22}{i+1})(\what_i - \what_{i+1}) \eqsp.
\end{equation}
Introduce the following notation:
\begin{align}
v_{i} &= \MKQ \pois{\funnoisew_W}{i} - (\MKQ \poisA{11}{i}) \ttheta_i - (\MKQ \poisA{12}{i}) \what_i \eqsp, Q_i =  \sum_{j=i+1} ^{k}\beta_jG_{j+1:k}^{(1)}A_{12}G_{i+1:j-1}^{(2)} \eqsp, \\
Z_{21} &=  \sum_{j=0}^{n}\beta_jG_{j+1:n}^{(1)}A_{12}\sum_{i=0}^{j-1}\gamma_iG_{i+1:j-1}^{(2)}W_{i+1}^{(1)} \eqsp, \eqsp Z_{22} = \sum_{j=0}^{n}\beta_jG_{j+1:n}^{(1)}A_{12}A_{22}^{-1}W_{j+1}^{(1)} \eqsp.
\end{align}
\Cref{convolution_inequality}-\ref{convolution_inequality_beta} implies that $\norm{Q_i} \leq \gamma_i^{(b-a)/a} P_{i+1:k}^{(1)}$ and
\begin{align}
\norm{Q_i - Q_{i+1}} \lesssim \beta_i P_{i+1:n}^{(1)} + \gamma_i \gamma_i^{\frac{b-a}{a}} P_{i+1:n}^{(1)} \lesssim \beta_i P_{i+1:n}^{(1)} \eqsp.
\end{align}
Now we swap the order of summation and rewrite $Z_{21}$ as follows:
\begin{align}
Z_{21} = \underbrace{\sum_{i=0}^{n-1} \gamma_i Q_i (v_i - v_{i+1})}_{Z_{211}} + \underbrace{\sum_{i=0}^{n-1} Q_i \beta_i \{(\MKQ \poisA{21}{i+1})(\ttheta_i - \ttheta_{i+1}) + (\MKQ \poisA{22}{i+1})(\what_i - \what_{i+1})\}}_{Z_{212}} \eqsp.
\end{align}
Then we further decompose $Z_{211}$:
\begin{align}
Z_{211} = \sum_{i=0}^{n-1} \{ (Q_i \gamma_i v_i - Q_{i+1} \gamma_{i+1} v_{i+1}) + (Q_{i+1} - Q_i) \gamma_i v_{i+1} + Q_i(\gamma_{i+1} - \gamma_i) v_{i+1} \} \eqsp.
\end{align}
Therefore, using \Cref{lem:summ_alpha_k}-\ref{lem:summ_alpha_k_p_item} and $\gamma_i \lesssim \beta_i^{a/b}$, easy to see that:
\begin{align}
\label{eq:markov:bound_sk3:z211}
\PE^{1/p}[\norm{Z_{211}}^p] \lesssim p^3 P_{0:n}^{(1)} + p^3 \gamma_n + p^3 \beta_n^{\frac{a+b}{b} - 1} \lesssim  p^3 P_{0:n}^{(1)} + p^3 \beta_n^{a/b} \eqsp.
\end{align}
To derive a bound for $Z_{212}$ we use Minkowski's inequality and get
\begin{align}
\label{eq:markov:bound_sk3:z212}
\PE^{1/p}[\norm{Z_{212}}^p] \lesssim \sum_{i=0}^{n-1} \beta_i \gamma_i^{\frac{b-a}{a}} p^3 \gamma_i P_{i+1:n}^{(1)} = p^3 \sum_{i=0}^{n-1} \beta_i^{1+ \frac{b-a}{b} + \frac{a}{b}} P_{i+1:n}^{(1)} \lesssim p^3 \beta_n \eqsp. 
\end{align}
Thus, $\PE^{1/p}[\norm{Z_{21}}^p] \lesssim p^3 P_{0:n}^{(1)} + p^3 \beta_n^{a/b}$. Substituting the decomposition \eqref{eq:markov:W1_decomposition} into the expression for $Z_{22}$ one can check applying \Cref{lem:representation_lemma:kaledin} that
\begin{align}
\label{eq:markov:bound_sk3:z22}
\PE^{1/p}[\norm{Z_{22}}^p] \lesssim p^3 P_{0:n}^{(1)} + p^3 \beta_n + p^3 \gamma_n \lesssim p^3 P_{0:n}^{(1)} + p^3 \beta_n^{a/b}.
\end{align}
The proof follows from gathering the bounds \eqref{eq:markov:bound_sk3:z1}, \eqref{eq:markov:bound_sk3:z211}, \eqref{eq:markov:bound_sk3:z212}, \eqref{eq:markov:bound_sk3:z22}.
\end{proof}

\begin{proof}[Proof of \Cref{lem:markov:last_iter_hk}]
First, we rewrite $H_n$ using the solutions $\pois{\Phi}{i}$, $\pois{\Psi}{i}$ of the corresponding Poisson equations:
\begin{align}
\Phi_{j+1} \ttheta_j &= (\pois{\Phi}{j+1} - \MKQ \pois{\Phi}{j}) \ttheta_j + \{(\MKQ \pois{\Phi}{j}) \ttheta_j - (\MKQ \pois{\Phi}{j+1} \ttheta_{j+1})\} + (\MKQ \pois{\Phi}{j+1}) (\ttheta_{j+1} - \ttheta_j) \eqsp, \\
\Psi_{j+1} \what_j &= (\MKQ \pois{\Psi}{j+1} - \MKQ \pois{\Psi}{j}) \what_j + \{(\MKQ \pois{\Psi}{j}) \what_j - (\MKQ \pois{\Psi}{j+1} \what_{j+1})\} + (\MKQ \pois{\Psi}{j+1}) (\what_{j+1} - \what_j) \eqsp,
\end{align}
Now we rewrite $H_k$ as follows
\begin{align}
H_n = &\underbrace{\sum_{j=0}^{n} \beta_j G_{j+1:n}^{(1)} \{v_j - v_{j+1}\}}_{H_1} + \underbrace{\sum_{j=0}^n \beta_j G_{j+1:n}^{(1)} \{(\pois{\Phi}{j+1} - \MKQ \pois{\Phi}{j}) \ttheta_j + (\pois{\Psi}{j+1} - \MKQ \pois{\Psi}{j}) \what_j\}}_{H_2} \\
&+ \underbrace{\sum_{j=0}^n \beta_j G_{j+1:n}^{(1)} \{(\MKQ \pois{\Phi}{j+1}) (\ttheta_{j+1} - \ttheta_j) + (\MKQ \pois{\Psi}{j+1}) (\what_{j+1} - \what_j)\}}_{H_3} \eqsp,
\end{align}
where we have set
\begin{align}
v_i = -\MKQ \pois{\funnoisew_{\theta}}{j} + (\MKQ \pois{\Phi}{j}) \ttheta_j + (\MKQ \pois{\Psi}{j}) \what_j \eqsp.
\end{align}
The bound for $H_1$ follows from \Cref{lem:representation_lemma:kaledin} and \Cref{lem:summ_alpha_k}-\ref{lem:summ_alpha_k_p_item}:
\begin{align}
\label{eq:markov:hk_bound:h1}
\PE^{1/p}[\norm{H_1}^p] \lesssim p^3 P_{0:n}^{(1)} + p^3 \beta_n + \sum_{j=0}^n \beta_j^2 P_{j+1:n}^{(1)} p^3 \lesssim p^3 P_{0:n}^{(1)} + p^3 \beta_n \eqsp.
\end{align}
Next, we note that $H_2$ is a sum of martingale difference sequence w.r.t. the filtration $\F_k = \sigma(X_s : s \leq k)$. Thus, we get applying Burkholders inequality \cite[Theorem 8.1]{osekowski}, \Cref{prop:markov:moment_bounds} and \Cref{lem:summ_alpha_k}-\ref{lem:summ_alpha_k_p_item}:
\begin{align}
\label{eq:markov:hk_bound:h2}
\PE^{2/p}[\norm{H_2}^p] \lesssim p^2 \sum_{j=0}^n \beta_j^2 P_{j+1:n}^{(1)} \{P_{0:j}^{(1)} + p^6 \gamma_j\} \lesssim \frac{p^2}{2b-1} P_{0:n}^{(1)} + p^6 \beta_n^{\frac{b+a}{b}} \eqsp.
\end{align}
Finally, we derive a bound for $H_3$ using Minkowski's inequality
\begin{align}
\label{eq:markov:hk_bound:h3}
\PE^{1/p}[\norm{H_3}^p] \lesssim \sum_{j=0}^n p^3 \beta_j \gamma_j P_{j+1:n}^{(1)} \lesssim p^3 \beta_n^{a/b} \eqsp.
\end{align}
The proof follows from gathering the bounds \eqref{eq:markov:hk_bound:h1}, \eqref{eq:markov:hk_bound:h2}, \eqref{eq:markov:hk_bound:h3}.
\end{proof}

\subsection{Matrix concentration inequality}
In this section we state the lemma that derives a McDiarmid-type concentration inequality for matrix-valued functions of an UGE Markov chain. 
\begin{lemma}
\label{lem:bounded_differences_norms_markovian}
Assume \Cref{assum:UGE}. Let $\{g_i\}_{i=1}^n$ be a family of measurable functions from $\Zset$ to  $\rset^{d \times d}$ such that $M = \sup_{Z \in \Zset} \norm{g(Z)} < \infty$
and $\pi(g_i)= 0$ for any $i \in\{1,\ldots,n\}$.
Then, for any initial probability $\xi$ on $(\Zset,\Zsigma)$, $n \in \nset$, $t \geq 0$, it holds
\begin{equation}
\label{eq:prob_for_norms_markov}
\PP_{\xi}\biggl(\normop{\sum\nolimits_{i=1}^{n}g_i(\State_{i})}\geq t\biggr) \leq 4 \exp\biggl\{-\frac{t^2}{80nd\taumix M^2}\biggr\}\eqsp.
\end{equation}
\end{lemma}
\begin{proof} The function $\varphi(\state_1,\dots,\state_n) := \norm{\sum_{i=1}^{n}g_i(\state_{i})}$ on $\Zset^n$ satisfies the bounded differences property:
\begin{align}
|\varphi(z_1, \ldots, z_n) - \varphi(z_1', \ldots, z_n')| \leq \sum_{i=1}^n 2M \1\{z_i \ne z_i'\} \eqsp. 
\end{align}
Hence, since $(1/2)\sup_{z,z' \in \Zset} \norm{\MKQ^{\taumix}(z, \cdot) - \MKQ^{\taumix}(z',\cdot)}[\sf{TV}] \leq 1/4$ by definition of $\taumix$ under \Cref{assum:UGE}, applying \cite[Corollary 2.10]{paulin_concentration_spectral}, we get for $t \geq \PE_{\xi}[\norm{\sum_{i=1}^{n}g_i(\State_{i})}]$,
\begin{align}
\PP_{\xi}\biggl(\norm{\sum\nolimits_{i=1}^{n}g_i(\State_{i})} \geq t\biggr) \leq 2  \exp\left\{-\frac{2(t-\PE_{\xi}[\norm{\sum_{i=1}^{n}g_i(\State_{i})}])^{2}}{9(n \cdot 4M^2) \taumix }\right\} = 2  \exp\left\{-\frac{(t-\PE_{\xi}[\norm{\sum_{i=1}^{n}g_i(\State_{i})}])^{2}}{18nM^2 \taumix }\right\} \eqsp.
\end{align}
It remains to upper bound $\PE_{\xi}[\norm{\sum_{i=1}^{n}g_i(\State_{i})}]$. Note that
\begin{align}
\PE_{\xi}[\norm{\sum\nolimits_{i=1}^{n}g_i(\State_{i})}^{2}] \leq \PE_{\xi}[\norm{\sum\nolimits_{i=1}^{n}g_i(\State_{i})}[\text{Fr}]^{2}] = \sum\nolimits_{i=1}^n  \PE_{\xi}[\norm{g_i(\State_{i})}[\text{Fr}]^{2}] + 2\sum\nolimits_{k=1}^{n-1}\sum\nolimits_{\ell = 1}^{n-k} \trace\biggl(\PE_{\xi}[g_k(\State_{k})^{\top} g_{k+\ell}(\State_{k+\ell})]\biggr) \eqsp,
\end{align}
and, using \Cref{assum:UGE} and $\invariantQ(g_{k+\ell}) = 0$, we obtain
\begin{align}
\trace\biggl(\PE_{\xi}[g_k(\State_{k})^{\top} g_{k+\ell}(\State_{k+\ell})]\biggr) &= \int_{\Zset} \trace\biggl\{ g_k(z)^{\top}\left(\MKQ^{\ell}g_{k+\ell}(z) - \invariantQ(g_{k+\ell})\right) \biggr\} \xi\MKQ^{k}(\rmd z)  \\ &\leq 2 d M^{2} \dobru{\MKQ^{\ell}}\eqsp.
\end{align}
Together with the definition of $\taumix$, this implies
\begin{align}
\textstyle
\sum_{k=1}^{n-1}\sum_{\ell = 1}^{n-k} \trace\biggl( \PE_{\xi}[g_k(\State_{k})^{\top} g_{k+\ell}(\State_{k+\ell})]\biggr)  & \leq 2 ndM^2 \sum_{\ell=1}^{n-1} \dobru{\MKQ^{\ell}}  \leq (8/3) d M^2 \taumix n \eqsp.
\end{align}
Combining the bounds above, we upper bound $\PE_{\xi}[\norm{\sum_{i=1}^{n}g_{i}(\State_{i})}]$ as
\begin{align}
\PE_{\xi}[\norm{\sum\nolimits_{i=1}^{n}g_i(\State_{i})}] &\leq \bigl\{ \PE_{\xi}[\norm{\sum\nolimits_{i=1}^{n}g_i(\State_{i})}^2] \bigr\}^{1/2} \leq
2 \sqrt{dn} M \sqrt{\taumix} =: v_n \eqsp.
\end{align}
Plugging this result in \eqref{eq:prob_for_norms_markov},  we obtain that
\begin{equation}
\label{eq:MacDiarmid_markov_new}
\PP_{\xi}\biggl(\norm{\sum\nolimits_{i=1}^{n}g_i(\State_{i})} \geq t\biggr) \leq
\begin{cases}
1, \quad t < v_n, \\
2\exp\left\{-\frac{(t-v_{n})^{2}}{18 v_{n}^2}\right\}\eqsp, \quad t \geq v_{n}\eqsp.
\end{cases}
\end{equation}
Now it is easy to see that \rhs\ of \eqref{eq:MacDiarmid_markov_new} is upper bounded by $4\exp\{-t^2/(20 v_n^2)\}$ for any $t \geq 0$, and the statement follows.
\end{proof}

\section{Limit of matrix sums}
\label{appendix:ricatti_matrix_sum}
Fix $\funcAw, \lineG \in \rset^{d \times d}$. Introduce the following notation:
\begin{align}
    \lineG_n = \sum_{k=1}^n \beta_k^2 G_{k+1:n} \lineG (G_{k+1:n})^\top \eqsp,
\end{align}
where $\beta_j = c_{0,\beta}/{(n+k_0)^b}$ and $G_{m:k} = \prod_{j=m}^k (\Id - \beta_j \funcAw)$.

\begin{proposition}
    \label{prop:ricatti_limit}
    There exists a matrix $\lineG_\infty$ such that
    \begin{align}
        \lim_{n \to \infty} \{ \beta_n^{-1} \lineG_n \} = \lineG_\infty\eqsp.
    \end{align}
    Moreover, there exists a constant $\ConstC_\Sigma$ that depends on the problem parameters such that
    \begin{align}
        \norm{\beta_n^{-1} \lineG_n - \lineG_\infty} \leq \ConstC_\Sigma n^{-b} \eqsp.
    \end{align}
\end{proposition}

\begin{proof}
First, we define $\lineG^{(0)}$, $\lineG^{(1)}$ as a solution of the system of Ricatti equations 
\begin{align} \label{eq:riccati1}
\lineG^{(0)}  & = \beta_0 \{ \Delta \lineG^{(0)} + \lineG^{(0)} \Delta^\top \} - \beta_0^2 \lineG \eqsp, \\
\lineG^{(0)} - 2 \lineG^{(1)}  & = \beta_0 \{ \Delta \lineG^{(0)} + \lineG^{(0)} \Delta^\top - (\Delta\lineG^{(1)} + \lineG^{(1)} \Delta^\top) \} + \beta_0^2 \{ \Delta \lineG^{(0)} \Delta^\top  -2\lineG \}.
\end{align}

Our goal is to compute $\lineG_n$.
To derive a closed form solution, we observe that $\lineG_n$ can be iteratively computed as
\begin{equation} \label{eq:covariance_iter}
\lineG_{n+1} = (\Id - \beta_{n+1} \Delta) \lineG_n (\Id - \beta_{n+1} \Delta)^\top + \beta_{n+1}^2 \lineG
\end{equation}

Now we consider the diminishing step size rule with $\beta_n = \beta_0 n^{-b}$ for $b \in (\frac{1}{2},1]$ and $\beta_0 > 0$. Note we have ignored $k_0$ in the step size selection as we focus on the asymptotic expression with $n \gg 1$. Arranging terms in \eqref{eq:covariance_iter} yields
\begin{equation} \label{eq:cov_diff}
    \lineG_{n+1} - \lineG_n = -\beta_0 (n+1)^{-b} \big\{ \Delta \lineG_n + \lineG_n \Delta^\top \big\} + \beta_0^2 (n+1)^{-2b} \big\{ \Delta \lineG_n \Delta^\top + \lineG \big\}.
\end{equation}
Set
\begin{equation} \label{eq:cov_model1}
    \lineG_n \equiv n^{-b} \lineG^{(0)} + n^{-2b} \lineG^{(1)} + D_n \eqsp,
\end{equation}
where $D_n$ is a residual term whose order will be determined later.
Note that for any $b > 0$, it holds
\begin{equation} \label{eq:approx_n}
(n+1)^{-b} = n^{-b} -b n^{-1-b} + \frac{b(b+1)}{2} n^{-2-b} + {\cal O}(n^{-3-b}) \eqsp,
\end{equation}
We focus on the case $b = 1$. Applying the above with $b=1$, we observe that
\begin{equation} \label{eq:model_alt1}
\begin{split}
\lineG_{n+1} - \lineG_n & = \{ (n+1)^{-1} - n^{-1} \} \lineG^{(0)} + \{ (n+1)^{-2} - n^{-2} \} \lineG^{(1)} + D_{n+1} - D_n \\
& = \{ - n^{-2} + n^{-3} \} \lineG^{(0)} -2 n^{-3} \lineG^{(1)} + D_{n+1} - D_n + {\cal O}( n^{-4} ) \eqsp.
\end{split}
\end{equation}
 On the other hand, observe that the right hand side of \eqref{eq:cov_diff} can be written as follows
\begin{equation} \notag
\begin{split}
& - \beta_0 (n+1)^{-1} \{ \Delta \lineG_n + \lineG_n \Delta^\top \} = - \beta_0 ( n^{-1} - n^{-2} + {\cal O}(n^{-3}) ) \{ \Delta \lineG_n + \lineG_n \Delta^\top \} \\
& = - \beta_0 (n^{-1}-n^{-2}) ( n^{-1} \{ \Delta \lineG^{(0)} + \lineG^{(0)} \Delta^\top \} + n^{-2} \{ \Delta \lineG^{(1)} + \lineG^{(1)} \Delta^\top \} ) + {\cal O}(n^{-4} + n^{-1} \| D_n \|)
\end{split}
\end{equation}
We also have
\begin{equation} \notag
\begin{split}
& \beta_0^2 (n+1)^{-2} \{ \Delta \lineG_n \Delta^\top + \lineG \} = \beta_0^2 ( n^{-2} - 2 n^{-3} ) \{ \Delta \lineG_n \Delta^\top + \lineG \} + {\cal O}(n^{-4}) \\
& = \beta_0^2 ( n^{-2} - 2 n^{-3} ) \{ n^{-1} \Delta \lineG^{(0)} \Delta^\top + n^{-2} \Delta \lineG^{(1)} \Delta^\top + \lineG \} + {\cal O}( n^{-4} + n^{-2} \| D_n \| )
\end{split}
\end{equation}
Matching terms of the same order with \eqref{eq:model_alt1} shows that
\begin{equation}
\begin{array}{rrl}
\text{($n^{-2}$)}~~ & \lineG^{(0)} \hspace{-.4cm} & = \beta_0 \{ \Delta \lineG^{(0)} + \lineG^{(0)} \Delta^\top \} - \beta_0^2 \lineG \vspace{.1cm}\\
\text{($n^{-3}$)}~~ & \lineG^{(0)} - 2 \lineG^{(1)} \hspace{-.4cm} & = \beta_0 \{ \Delta \lineG^{(0)} + \lineG^{(0)} \Delta^\top - (\Delta\lineG^{(1)} + \lineG^{(1)} \Delta^\top) \} + \beta_0^2 \{ \Delta \lineG^{(0)} \Delta^\top  -2\lineG \}.
\end{array}
\end{equation}
We observe that the remaining terms are all in the order of at most ${\cal O}(n^{-4})$. As such, we also conclude that the residual term in \eqref{eq:cov_model1} is of the order at most $D_n = {\cal O}( n^{-3} )$.
In particular, solving the system of Riccati equations \eqref{eq:riccati1} yield $\lineG^{(0)}, \lineG^{(1)}$, i.e. the asymptotic expression for $\lineG_n$ is
\begin{equation}
    \sum_{k=1}^n \beta_k^2 G_{k+1:n} \lineG (G_{k+1:n})^\top = n^{-1} \lineG^{(0)} + n^{-2} \lineG^{(1)} + {\cal O}(n^{-3})
\end{equation}
Note that as the above analysis assumes the asymptotic case when $n \gg 1$, it actually covers the case when $\beta_n = \beta_0 (n+n_0)^{-1}$.
The similar computations with $b \in (1/2, 1)$ imply that
\begin{align}
    \lineG_n = n^{-b} \lineG^{(0)} + n^{-2b} \lineG^{(1)} + {\cal O}(n^{-1-2b}) \eqsp.
\end{align}
Therefore, setting $\lineG_\infty= \lineG^{(0)} / \beta_0$ we get
\begin{align}
    \norm{\beta_n^{-1} \lineG_n - \lineG_\infty} \leq \frac{\norm{\lineG^{(0)}} + \norm{\lineG^{(1)}}}{\beta_0} \ConstC_b n^{-b} = \ConstC_\Sigma n^{-b} \eqsp,
\end{align}
where $\ConstC_b$ depends only on $b$.
\end{proof}

The next lemma controls the minimal eigenvalue of $\lineG_n$:

\begin{lemma}
    \label{lem:ricatti_sigma_n_lambda_min}
    Under the assumptions of \Cref{prop:ricatti_limit} it holds for all $n^b \geq \frac{2\ConstC_\Sigma}{\lambda_{\min}(\lineG_\infty)}$ that
    \begin{align}
        \lambda_{\min} (\beta_n^{-1} \lineG_n) \geq \frac{\lambda_{\min} (\lineG_\infty)}{2} \eqsp.
    \end{align}
\end{lemma}

\begin{proof}
    First, we use \Cref{prop:ricatti_limit} and obtain
    \begin{align}
        \norm{\beta_n^{-1} \lineG_n - \lineG_\infty^{\operatorname{last}}} \leq \frac{\ConstC_\Sigma}{n^{b}} \eqsp.
    \end{align}
    Hence, Lidskiy's inequality implies that
    \begin{align}
        \lambda_{\min}(\beta_n^{-1} \lineG_n) = \lambda_{\min}(\beta_n^{-1} \lineG_n - \lineG_\infty+ \lineG_\infty) \geq \lambda_{\min}(\lineG_\infty) - \norm{\beta_n^{-1} \lineG_n - \lineG_\infty} \geq \frac{\lambda_{\min}(\lineG_\infty)}{2} \eqsp.
    \end{align}
\end{proof}

\section{Applications}
\label{appendix:examples}
In this section, we verify that the GTD and TDC algorithms satisfy \Cref{assum:hurwitz}. Verification of the remaining assumptions is straightforward and thus omitted. We concentrate on the Markovian setting, as it is more prevalent in practical applications. Recall that the behavior policy  \(\pi\) generates a trajectory \(\{(s_k,a_k,r_k)\}_{k=0}^{\infty}\eqsp\),  where \(a_k \sim \bar{\pi}(\cdot \mid s_k),\eqsp s_{k+1}\sim \MKQ(\cdot|s_k, a_k)\) for all \(k \geq 0\) and the corresponding Markov kernel \(\MKQ_{\pi}\) satisfies \Cref{assum:P_pi_ergodicity_td}.
\paragraph{Generalized Temporal Difference learning.}
The GTD algorithm was first introduced in \cite{Sutton_2008}. Recall its update rule:
\begin{equation}
    \begin{cases}
\theta_{k+1} = \theta_k + \beta_k (\varphi_k - \lambda \varphi_{k+1}) (\varphi_k)^\top w_k \eqsp,\quad  \theta_0 \in \rset^{d}\eqsp, \\
w_{k+1} = w_{k} + \gamma_k (\delta_k \varphi_k - w_k) \eqsp, \quad w_0 = 0\eqsp.
\end{cases}
\end{equation}
The above recursion is a special case of our linear two-timescale SA in \eqref{eq:2ts1}, \eqref{eq:2ts2} with the notations:
\begin{align}
     b_1 &=0\eqsp,\quad A_{11}=0\eqsp, \quad A_{12}= - \PE[(\varphi_k-\lambda \varphi_{k+1})\varphi_k^\top]\eqsp,\\
    b_2&=\PE[\varphi_k r_k]\eqsp,\quad A_{21} = - \PE[\varphi_k(\lambda\varphi_{k+1}-\varphi_k)^\top]\eqsp,\quad A_{22}= \Id_d\eqsp,\\ 
     V_{k+1}&=\left((\varphi_k- \lambda \varphi_{k+1})\varphi_k^\top - \PE[(\varphi_k- \lambda \varphi_{k+1})\varphi_k^\top] \right) w_k,\\
     W_{k+1}&=\varphi_k r_k - \PE[\varphi_k r_k] + \left( (\varphi_k- \lambda \varphi_{k+1})\varphi_k^\top - \PE[(\varphi_k- \lambda \varphi_{k+1})\varphi_k^\top] \right) \theta_k\eqsp,
    \label{eqs:gtdnoises}
\end{align}
where the above expectations are taken with respect to the randomness of the policy $\bar{\pi}$.
The noise boundedness follows from \Cref{assum:phi_norm_bound}, while \Cref{assum:hurwitz} holds since \(A_{22} = \Id\) and \(\Delta = \PE[(\varphi_k-\lambda\varphi_{k+1})\varphi_k^\top]\PE[\varphi_k(\varphi_k - \lambda\varphi_{k+1})^\top]\) is positive definite.

\paragraph{Temporal-difference learning with gradient correction.}
The TDC algorithm was first introduced in \cite{sutton:gtd2:2009}.
Its update rule is:
\begin{equation}
    \begin{cases}
\theta_{k+1} = \theta_k + \beta_k \delta_k\varphi_k - \beta_k\gamma \varphi_{k+1} (\varphi_k^\top w_k) \eqsp, \\
w_{k+1} = w_{k} + \gamma_k (\delta_k - \varphi_{k}^\top w_k)\varphi_k \eqsp.
\end{cases}
\end{equation}
Reformulating these updates as an instance of \eqref{eq:2ts1}--\eqref{eq:2ts2} yields:
\begin{align}
    b_1 &=\PE[\varphi_k r_k], \quad A_{11}=\PE[\varphi_k(\varphi_k - \lambda\varphi_{k+1})^\top)], \quad A_{12}=  \PE[\lambda\varphi_{k+1}\varphi_k^\top],\\
    b_2&=\PE[\varphi_k r_k], \quad A_{21} =\PE[\varphi_k(\varphi_k - \lambda\varphi_{k+1})^\top)], \quad A_{22}= \PE[\varphi_k\varphi_k^{\top}],\\ 
    V_{k+1} &= \{\PE[\varphi_k(\varphi_k - \lambda\varphi_{k+1})^\top)] - \varphi_k(\varphi_k - \lambda\varphi_{k+1})^\top \} \theta_k + \{\PE[\lambda\varphi_{k+1}\varphi_k^\top] - \lambda\varphi_{k+1}\varphi_k^\top \} w_k, \\
     W_{k+1}&= \{\PE[\varphi_k(\varphi_k - \lambda\varphi_{k+1})^\top)] - \varphi_k(\varphi_k - \lambda\varphi_{k+1})^\top\} \theta_k + \{\PE[\varphi_k\varphi_k^{\top}] - \varphi_k\varphi_k^{\top}\} w_k\eqsp.
\end{align}
The relations \(A_{11} = A_{21}\) and \(A_{12} = A_{22} - A_{11}^\top\) imply that \(\Delta\) is positive definite:
\begin{equation}
    \Delta = A_{11} - A_{12} A_{22}^{-1} A_{21} 
           = A_{11} - (A_{22} - A_{11}^\top) A_{22}^{-1} A_{11} 
           = A_{11}^\top A_{22}^{-1} A_{11}\eqsp. 
\end{equation}

\section{Technical lemmas}
\label{appendix:technical}
We begin this section with technical lemmas that allows to upper bound the sums of the form
\[
\sum_{j=1}^{k}\alpha_{j}^{q}\prod_{\ell=j+1}^{k}(1-\alpha_{\ell} b)\eqsp.
\]

\begin{lemma}
\label{lem:summ_alpha_k}
The following statement holds:
\begin{enumerate}[(i)]
    \item \; \label{lem:summ_alpha_k_first} Let $b > 0$ and $(\alpha_k)_{k \geq 0}$ be a non-increasing sequence such that $\alpha_0 \leq 1/b$. Then
    \[
    \sum_{j=1}^{k} \alpha_j \prod_{l=j+1}^{k} (1 - \alpha_l b) = \frac{1}{b} \biggl\{1  - \prod_{l=1}^{k} (1 - \alpha_l b) \biggr\} \eqsp.
    \]
    \item \; \label{lem:summ_alpha_k_p_item} 
    Let $b > 0$ and $\alpha_k = \frac{c_0}{(k+k_0)^\gamma}$, $\gamma \in (0,1)$, such that $c_0 \leq 1/b$ and $k_0^{1-\gamma} \geq \frac{8\gamma}{bc_0}$. Then for any $q \in (1, 4]$ it holds that
    \begin{align}
        \sum_{j=1}^k \alpha_j^q \prod_{\ell=j+1}^k (1 - \alpha_\ell b) \leq \frac{6}{b} \alpha_k^{q-1} \eqsp.
    \end{align}
    Moreover, for any real-valued sequence $(b_j)_{j\geq0}$ it holds that
    \begin{align}
        \label{eq:sum_alpha_k_p:jensen}
        \biggl\{\sum_{j=1}^k b_j \alpha_j^q \prod_{\ell=j+1}^k (1 - \alpha_\ell b)\biggr\}^2 \leq \frac{6}{b} \alpha_k^{q-1} \sum_{j=1}^{k}  b_j^2 \alpha_j^q \prod_{\ell=j+1}^k (1 - \alpha_\ell b) \eqsp.
    \end{align}
    \item \; \label{lem:sum_as_Qell_item} Let $b, c_0, k_0 > 0$ and $\alpha_\ell = c_0 (\ell+k_0)^{-\gamma}$ for $\gamma \in (1/2, 1)$ and $\ell \in \nset$. Assume that $bc_0 < 1$ and $k_0^{1-\gamma} \geq \frac{1}{bc_0}$. Then, for any $\ell, n \in \nset$, $\ell \leq n$, it holds that
    \begin{equation}
    \label{eq:const_L_b_def}
    \sum_{k=\ell}^{n} \alpha_\ell\prod_{j=\ell+1}^{k} (1-b\alpha_j) \leq c_0 + \frac{1}{b(1-\gamma)} \eqsp.
    \end{equation}
\end{enumerate}
\end{lemma}

\begin{proof}
    \Cref{lem:summ_alpha_k}-\ref{lem:summ_alpha_k_first} follows from Lemma~24 in \cite{durmus2021stability}. The first part of \Cref{lem:summ_alpha_k}-\ref{lem:summ_alpha_k_p_item} follows from Lemma~33 in \cite{samsonov2025statistical} and the second one \eqref{eq:sum_alpha_k_p:jensen} is a consequence of Jensen's inequality applied to $f(x) = x^2$. 
\Cref{lem:summ_alpha_k}-\ref{lem:sum_as_Qell_item} is elementary.
\end{proof}

\begin{lemma}
\label{convolution_inequality} 
Assume \Cref{assum:stepsize} or \Cref{assum:markov:stepsize}. Then it holds for all $j, k \in \nset$ that
\begin{enumerate}[(i)]
\item \; 
\label{convolution_inequality_gamma}
\begin{equation}
\sum_{i=j+1}^{k}\gamma_j P_{i+1:k}^{(1)}P_{j+1:i-1}^{(2)} \leq \ConstC_{\gamma}^P P_{j+1:k}^{(1)}\eqsp, \text{ where } \ConstC_{\gamma}^P =  \frac{24}{a_{22}(1-\frac{c_{0,\beta}a_{\Delta}}{2})} \eqsp.
\end{equation}
\item \;
\label{convolution_inequality_beta}
\begin{align}
    \sum_{i=j+1}^{k} \beta_j P_{i+1:k}^{(1)} P_{j+1:i-1}^{(2)} \leq \ConstC_{\beta}^P \gamma_{j}^{(b-a)/a} P_{j+1:k}^{(1)} \eqsp, \text{ where } \ConstC_{\beta}^P = \frac{24c_{0, \beta}}{a_{22}c_{0, \gamma}^{b/a}(1 - \frac{c_{0, \beta} a_{\Delta}}{2})} \eqsp.
\end{align}
\end{enumerate}
\end{lemma}
\begin{proof}
First, we prove \ref{convolution_inequality_gamma}. Since the step size was chosen such that \(\beta_k/\gamma_k\leq \rstep \leq a_{22}/(2a_\Delta)\) we have
\begin{equation}
\frac{\left(1 - \frac{a_{22}}{2}\gamma_{\ell}\right)}{\left(1 - \frac{a_{\Delta}}{2}\beta_{\ell}\right)} \leq \left(1 - \frac{a_{22}}{4}\gamma_{\ell}\right) \text{ and } \frac{P_{j+1:i-1}^{(2)}}{P^{(1)}_{j+1:i-1}} \leq\prod_{\ell = j+1}^{i-1} (1 - \frac{a_{22}}{4}\gamma_{\ell})\eqsp.
\end{equation}
Now proposition follows from \Cref{lem:summ_alpha_k}-\ref{lem:sum_as_Qell_item}:
\begin{align}
\sum_{i=j + 1}^{k} \gamma_j P_{i+1:k}^{(1)}P_{j+1:i-1}^{(2)} 
= (1 - \frac{\beta_ja_{\Delta}}{2})^{-1} P_{j+1:k}^{(1)} \sum_{i=j + 1}^{k} \gamma_j\frac{P_{j+1:i-1}^{(2)}}{P_{j+1:i-1}^{(1)}} &\leq (1 - \frac{\beta_0a_{\Delta}}{2})^{-1}P_{j+1:k}^{(1)}\sum_{i = j+1}^{k}\gamma_j\prod_{\ell=j+1}^{i-1}(1- \frac{a_{22}}{4}\gamma_{\ell})\\
&\leq \frac{24}{a_{22}(1-\frac{c_{0,\beta}a_{\Delta}}{2})}P_{j+1:k}^{(1)} = \ConstC_{\gamma}^P P_{j+1:k}^{(1)}\eqsp.
\end{align}
To proceed with \ref{convolution_inequality_beta}, we note that
$
    \beta_j = \frac{c_{0, \beta}}{c_{0, \gamma}^{b/a}} \gamma_j^{1+(b-a)/a}
$.
Hence, we get using the technique similar to \Cref{convolution_inequality} combined with \Cref{lem:summ_alpha_k}-\ref{lem:sum_as_Qell_item}: \begin{align}
\sum_{i=j + 1}^{k} \beta_j P_{i+1:k}^{(1)}P_{j+1:i-1}^{(2)} 
= (1 - \frac{\beta_ja_{\Delta}}{2})^{-1} P_{j+1:k}^{(1)} \sum_{i=j + 1}^{k} \beta_j\frac{P_{j+1:i-1}^{(2)}}{P_{j+1:i-1}^{(1)}} &\leq \frac{c_{0, \beta}}{c_{0, \gamma}^{b/a}(1 - \frac{\beta_0a_{\Delta}}{2})} P_{j+1:k}^{(1)}\sum_{i = j+1}^{k}\gamma_j^{b/a}\prod_{\ell=j+1}^{i-1}(1- \frac{a_{22}}{4}\gamma_{\ell})\\
&\leq \frac{24c_{0, \beta}}{a_{22}c_{0, \gamma}^{b/a}(1 - \frac{c_{0, \beta} a_{\Delta}}{2})} \gamma_j^{(b-a)/a} \eqsp.
\end{align}
\end{proof}

\begin{lemma}[Lemma 36 in \cite{samsonov2025statistical}]
\label{lem:bound_sum_exponent}
    For any $A >0$, any $1 \leq i \leq n-1$,   and $\gamma\in(1/2, 1)$ it holds
   \begin{equation}
        \sum_{j=i}^{n-1}\exp\biggl\{-A(j^{1-\gamma} - i^{1-\gamma})\biggr\} \leq
        \begin{cases}
            1 + \exp\bigl\{\frac{1}{1-\gamma}\bigr\}\frac{1}{A^{1/(1-\gamma)}(1-\gamma)}\Gamma(\frac{1}{1-\gamma})\eqsp, &\text{ if } Ai^{1-\gamma} \leq \frac{1}{1-\gamma}\eqsp;\\
            1 + \frac{1}{A(1-\gamma)^2}i^\gamma\eqsp,  &\text{ if } Ai^{1-\gamma} >\frac{1}{1-\gamma}\eqsp.
        \end{cases}
    \end{equation}
\end{lemma}

\end{appendices}
\end{document}